\documentclass[12pt, final]{l4dc2023}

\title[Provably Efficient Generalized Lagrangian Policy Optimization for Safe MARL]{Provably Efficient Generalized Lagrangian Policy Optimization for Safe Multi-Agent Reinforcement Learning}
\usepackage{times}

\usepackage{amsmath}
\usepackage{amssymb}
\usepackage{graphicx}
\usepackage{mathtools}
\usepackage{enumerate}
\usepackage{enumitem}
\usepackage{footnote}
\usepackage{float}
\usepackage{xspace}
\usepackage{multirow}
\usepackage{nicefrac}
\usepackage{wrapfig}
\usepackage{framed}
\usepackage{url}
\usepackage{tikz}
\usetikzlibrary{arrows,chains,matrix,positioning,scopes}

\let\hat\widehat
\let\tilde\widetilde

\newtheorem{assumption}{Assumption}

\DeclareMathOperator*{\minimize}{minimize}
\DeclareMathOperator*{\maximize}{maximize}
\DeclareMathOperator*{\subject}{subject~to}

\newcommand{\calL}{{\mathcal{L}}}
\newcommand{\calX}{{\mathcal{X}}}
\newcommand{\calY}{{\mathcal{Y}}}

\newcommand{\calP}{{\mathcal{P}}}

\newcommand{\calF}{{\mathcal{F}}}

\DeclareMathOperator*{\argmin}{argmin}
\DeclareMathOperator*{\argmax}{argmax}
\DeclareMathOperator*{\argminimax}{argminimax}

\newcommand{\field}[1]{\mathbb{#1}}

\newcommand{\Ind}[1]{ \field{I}{\{{#1}\}} }

\newcommand{\norm}[1]{\left\|{#1}\right\|}

\newcommand{\DefinedAs}[0]{\mathrel{\mathop:}=}

\newcommand{\inner}[2]{\left\langle #1,#2 \right\rangle}

\newcommand{\rbr}[1]{\left(#1\right)}
\newcommand{\sbr}[1]{\left[#1\right]}

\newcommand{\abr}[1]{\left|#1\right|}

\DeclareFontFamily{OMX}{MnSymbolE}{}
\DeclareFontShape{OMX}{MnSymbolE}{m}{n}{
    <-6>  MnSymbolE5
   <6-7>  MnSymbolE6
   <7-8>  MnSymbolE7
   <8-9>  MnSymbolE8
   <9-10> MnSymbolE9
  <10-12> MnSymbolE10
  <12->   MnSymbolE12}{}
\DeclareSymbolFont{mnlargesymbols}{OMX}{MnSymbolE}{m}{n}
\SetSymbolFont{mnlargesymbols}{bold}{OMX}{MnSymbolE}{b}{n}
\DeclareMathDelimiter{\llangle}{\mathopen}{mnlargesymbols}{'164}{mnlargesymbols}{'164}
\DeclareMathDelimiter{\rrangle}{\mathclose}{mnlargesymbols}{'171}{mnlargesymbols}{'171}

\usepackage{prettyref}

\newcommand{\savehyperref}[2]{\texorpdfstring{\hyperref[#1]{#2}}{#2}}
\newrefformat{eq}{\savehyperref{#1}{\textup{(\ref*{#1})}}}
\newrefformat{eqn}{\savehyperref{#1}{Equation~\eqref{#1}}}
\newrefformat{lem}{\savehyperref{#1}{Lemma~\ref*{#1}}}
\newrefformat{def}{\savehyperref{#1}{Definition~\ref*{#1}}}
\newrefformat{line}{\savehyperref{#1}{line~\ref*{#1}}}
\newrefformat{thm}{\savehyperref{#1}{Theorem~\ref*{#1}}}
\newrefformat{corr}{\savehyperref{#1}{Corollary~\ref*{#1}}}
\newrefformat{cor}{\savehyperref{#1}{Corollary~\ref*{#1}}}
\newrefformat{sec}{\savehyperref{#1}{Section~\ref*{#1}}}
\newrefformat{app}{\savehyperref{#1}{Appendix~\ref*{#1}}}
\newrefformat{assum}{\savehyperref{#1}{Assumption~\ref*{#1}}}
\newrefformat{ex}{\savehyperref{#1}{Example~\ref*{#1}}}
\newrefformat{fig}{\savehyperref{#1}{Figure~\ref*{#1}}}
\newrefformat{alg}{\savehyperref{#1}{Algorithm~\ref*{#1}}}
\newrefformat{rem}{\savehyperref{#1}{Remark~\ref*{#1}}}
\newrefformat{conj}{\savehyperref{#1}{Conjecture~\ref*{#1}}}
\newrefformat{prop}{\savehyperref{#1}{Proposition~\ref*{#1}}}
\newrefformat{proto}{\savehyperref{#1}{Protocol~\ref*{#1}}}
\newrefformat{prob}{\savehyperref{#1}{Problem~\ref*{#1}}}
\newrefformat{claim}{\savehyperref{#1}{Claim~\ref*{#1}}}
\newrefformat{que}{\savehyperref{#1}{Question~\ref*{#1}}}

\usepackage{algorithm}
\usepackage{algorithmic}

\usepackage[utf8]{inputenc} 
\usepackage[T1]{fontenc}    
\usepackage{hyperref}       
\usepackage{url}            
\usepackage{booktabs}       
\usepackage{amsfonts}       
\usepackage{nicefrac}       
\usepackage{microtype}      
\usepackage{xcolor}         

\hypersetup{
	breaklinks=true,   
	colorlinks=true,   
	pdfusetitle=true,  
}

\author{%
 \Name{Dongsheng Ding} \Email{dongshed@seas.upenn.edu}\\
 \addr University of Pennsylvania, Philadelphia, PA 19104, USA
 \AND
 \Name{Xiaohan Wei} \Email{ubimeteor@fb.com}\\
 \addr Meta, Menlo Park, CA 94065 USA
 \AND
 \Name{Zhuoran Yang} \Email{zhuoran.yang@yale.edu}\\
 \addr Yale University,
 New Haven, CT 06511, USA
 \AND
 \Name{Zhaoran Wang} \Email{zhaoranwang@gmail.com}\\
 \addr Northwestern University, 
 Evanston, IL 60208, USA
   \AND
 \Name{Mihailo R. Jovanovi\'c} \Email{mihailo@usc.edu}\\
 \addr University of Southern California, Los Angeles, CA 90089, USA
}

\begin{document}

\maketitle

\begin{abstract}%
	 We examine online safe multi-agent reinforcement learning using constrained Markov games in which agents compete by maximizing their expected total rewards under a constraint on expected total utilities. Our focus is confined to an episodic two-player zero-sum constrained Markov game with independent transition functions that are unknown to agents, adversarial reward functions, and stochastic utility functions. For such a Markov game, we employ an approach based on the occupancy measure to formulate it as an online constrained saddle-point problem with an explicit constraint. We extend the Lagrange multiplier method in constrained optimization to handle the constraint by creating a generalized Lagrangian with minimax decision primal variables and a dual variable. Next, we develop an upper confidence reinforcement learning algorithm to solve this Lagrangian problem while balancing exploration and exploitation. Our algorithm updates the minimax decision primal variables via online mirror descent and the dual variable via projected gradient step and we prove that it enjoys sublinear rate $ O((|X|+|Y|) L \sqrt{T(|A|+|B|)}))$ for both regret and constraint violation after playing $T$ episodes of the game. Here, $L$ is the horizon of each episode, $(|X|,|A|)$ and $(|Y|,|B|)$ are the state/action space sizes of the min-player and the max-player, respectively.
	 To the best of our knowledge, we provide the first provably efficient online safe reinforcement learning algorithm in constrained Markov games.

\end{abstract}


\section{Introduction}

Safe Reinforcement Learning (RL) studies how a single agent learns to maximize its expected total reward subject to safety-concerned constraints by interacting with an unknown environment over time~\citep{garcia2015comprehensive,thomas2015safe,amodei2016concrete}. The constrained Markov decision processes (MDPs) provide a standard class of constraint critical environment models~\citep{altman1999constrained} that are utilized in autonomous robots~\citep{feyzabadi2017robot,fisac2018general}, personalized medicine~\citep{girard2018structural}, online advertising~\citep{boutilier2016budget}, and financial management~\citep{abe2010optimizing}. General constrained MDPs for two or more agents are often formulated as constrained Markov games (MGs) in which agents compete under constraints~\citep{altman2000constrained,altman2005zero,altman2008constrained}, providing an effective model for safe multi-agent RL~\citep{nguyen2014decentralized,shalev2016safe,zhang2019multi}.

Considerable recent progress has been made in single-agent safe RL, especially for solving constrained MDP problems with constraint satisfaction guarantees~\citep{efroni2020exploration,brantley2020constrained,bai2020model,ding2020provably,chen2021primal,singh2020learning,ding2022convergence}. In these references, Lagrangian-based methods have been combined with the optimistic exploration to address exploration-exploitation trade-off under constraints. These constrained MDP learning algorithms are sample-efficient (in achieving both low regret and low constraint violation) and they effectively enhance classical RL methods to attain safety requirements. However, most of these algorithms are limited to the single-agent setting and it is an open question how to balance the exploration-exploitation trade-off under constraints for multiple agents. Another motivation for our work comes from recent advances on the efficient competitive RL algorithms in MGs~\citep{wei2017online,bai2020provable,bai2020near,xie2020learning}.

In this work, we take initial steps towards developing provably efficient safe multi-agent RL algorithms. We examine perhaps the most basic safe multi-agent RL setup that involves a two-player zero-sum constrained MG with independent state transitions~\citep{altman2000constrained,altman2005zero,altman2008constrained,singh2014characterization}. 
This problem represents a generalization of constrained MDPs to the two-player case with coupled constraints. In such a constrained MG, two players follow their own state transitions independently, take actions simultaneously, and observe the reward and utility functions while competing against each other by maximizing/minimizing the reward while both are restrained by the constraint regarding some utility for safety reasons. The decision-coupling that arises from the constraint is often encountered in multi-agent systems~\citep{rosen1965existence,li2014decoupling,kulkarni2011generalized,kulkarni2017games,de2019resource}. More specifically, we aim to design an online RL algorithm for solving episodic two-player zero-sum constrained MGs. Here, two players do not know the transition models and have no access to a generative model, but can play the game for multiple episodes using arbitrary policies. The goal is to find an approximate constrained Nash equilibrium of the game in hindsight, a generalization of Nash equilibrium to characterize violating constraints if any unilateral deviations occur. We utilize a notion of regret to quantify the approximation error of the constrained Nash equilibrium and employ a constraint dissatisfaction (which results from violation of any utility constraints) to evaluate the constraint violation.

\vspace*{1ex}
\noindent\textbf{Contribution}. We develop the first provably efficient algorithm for a constrained Markov game (MG) with $O(\sqrt{T})$ regret and $O(\sqrt{T})$ constraint violation. Specifically, we introduce an episodic constrained MG with unknown independent transition functions and decision-couplings that come from both adversarial reward functions and coupled stochastic constraints on utility functions. We use the occupancy measure approach to formulate such a MG as a constrained saddle-point problem with an explicit constraint. 
We extend the Lagrange method in constrained optimization to deal with the constraint by creating a generalized Lagrangian with minimax decision primal variables and a dual variable.
We develop an upper confidence reinforcement learning algorithm -- an $\underline{\text{U}}$pper $\underline{\text{C}}$onfidence $\underline{\text{B}}$ound 
$\underline{\text{C}}$onstrained $\underline{\text{SA}}$ddle-$\underline{\text{P}}$oint $\underline{\text{O}}$ptimization (UCB-CSAPO) algorithm -- to solve this Lagrangian problem while balancing exploration and exploitation. Our algorithm updates the minimax decision primal variables via optimistic mirror descent and the dual variable via projected gradient step and we prove that it enjoys sublinear rate $ O((|X|+|Y|) L \sqrt{T(|A|+|B|)}))$ for both regret and constraint violation after playing $T$ episodes. Here, $L$ is the horizon of each episode, $(|X|,|A|)$ and $(|Y|,|B|)$ are the state/action space sizes of the min-player and max-player, respectively. 

\vspace*{1ex}
\noindent\textbf{Related Work}. We briefly review the most-related work; see Appendix~\ref{ap.related} for details. Our work is closely related to safe multi-agent RL in constrained MGs. 
The Nash equilibrium for constrained MGs have been studied in~\cite{altman2000constrained,gomez2003saddle,altman2005zero,alvarez2006existence,altman2007constrained,altman2008constrained,altman2009constrained,singh2014characterization} using the notion of \emph{constrained Nash equilibrium} (which generalizes the concept of \emph{generalized Nash equilibrium} in static games~\citep{arrow1954existence} to MGs); see more studies in~\cite{yaji2015necessary,zhang2019discrete,wei2020discrete,wei2021constrained,zhang2021constrained}. These results are not applicable to the RL setting that assumes unknown models. Recently, asymptotic convergence in learning constrained MGs was examined in~\cite{hakami2015learning,jiang2020finding} but sample efficiency and exploration were not fully addressed, except for a concurrent work on learning correlated equilibria~\citep{
	chenfinding}. Our work fills this gap by adding built-in exploration mechanisms under constraints and proving the first non-asymptotic convergence for learning constrained Nash equilibria. 

Our work is also pertinent to a rich RL literature on learning constrained MDPs~\citep{zheng2020constrained,qiu2020upper,kalagarla2020sample,bai2020model,chow2017risk,tessler2018reward,ding2020natural,ding2020provably, ding2022convergence, wachi2020safe,efroni2020exploration,brantley2020constrained,chen2021primal,liu2021learning,ying2021dual,liu2021fast,bai2021achieving,zhao2021primal,li2021faster,chen2022learning}. While these results provide provably efficient algorithms regarding regret and constraint satisfaction in the single-agent setting, they are not applicable to our multi-agent game being played under constraints, because of the \emph{non-convexity} nauture of constrained multi-agent policy optimization and the \emph{non-stationary} environment each agent is facing. An extended line of work on constrained MDPs focuses on cooperative multi-agent learning under constraints and most efforts study the case where multiple agents have independent MDPs with a coupled budget/resource constraint~\citep{meuleau1998solving,boutilier2016budget,wei2018online,de2020risk,gagrani2020weakly}. All these results assume knowing transition models or system dynamics. Only a few studies considered the shared MDP case~\citep{diddigi2019actor,ludecentralized,parnika2021attention,gu2021multi}, but they lack theoretical guarantees and do not handle exploration. In contrast, our work focuses on the MG setting with unknown models and attacks the exploration challenge directly.

\vspace*{-1ex}
\section{Problem Setup}
\label{prelim}

In this section, we introduce zero-sum Markov games (MGs) with constraints, which are categorized as constrained Markov/stochastic games~\citep{altman2000constrained,altman2005zero,altman2008constrained}. 

In an episodic constrained MG there are two players; a \emph{min-player} -- $(X, A, P_1, r, g,T)$, which minimizes the reward, and a \emph{max-player} -- $(Y, B, P_2, r, h,T)$, which maximizes the reward, while adhering to a coupled utility constraint. Here, $T$ is the number of episodes, $X$ and $Y$  are finite state spaces, $A$ and $B$ are finite action spaces, $P_1$ and $P_2$ are transition probability measures where $P_1 (\cdot \,|\, x, a)$ is a distribution over $X$ if the min-player takes action $a$ in state $x$ and $P_2 (\cdot \,|\, y, b)$ is a distribution over $Y$ if the max-player takes action $b$ in state $y$, $r \DefinedAs \{r^t\}_{t\,=\,1}^T$ is a collection of players' reward functions $r^t$: $X\times Y\times A\times B\to [0,1]$, whereas $g \DefinedAs \{g^t\}_{t\,=\,1}^T$ and $h\DefinedAs \{h^t \}_{t\,=\,1}^T$ are collections of players' utility functions $g^t$: $X\times A\to [0,1]$, $h^t$: $Y\times B\to [0,1]$. For two independent transitions, players are coupled via the reward function and a constraint on their utility functions. 

We utilize layered Markov decision processes to model the environment dynamics. For each player, e.g., the min-player, we assume that the state space $X$ has $L+1$ layers and that it satisfies the loop-free property: (i) $X \DefinedAs X_{0} \cup\cdots\cup X_{L}$ and $X_{\ell_1}\cap X_{\ell_2}=\emptyset$ for $\ell_1\neq \ell_2$; (ii) $X_0=\{ x_0 \}$ and $X_L=\{ x_L \}$; (iii) if $P_1( x' \,\vert\, x, a) >0$, then $x' \in X_{\ell+1}$ and $x\in X_\ell$ for some $\ell \in \{0,1,\cdots,L\}$. This assumption is common in loop-free stochastic shortest path problems~\citep{gyorgy2007line,jaksch2010near,neu2010online,rosenberg2019online,jin2020learning}; it is often used to simplify notation/analysis since any episodic MDPs can be reduced to be loop-free.

The min/max players interact with the environment in episode $t$ as follows. At the beginning, the environment determines the reward function $r^t$ and the utility functions $g^t$ and $h^t$. Meanwhile, two players decide their policies $\pi^t$: $X\times A \to [0,1]$ and $\mu^t$: $Y\times B\to [0,1]$, where $\pi^t( \cdot\,\vert\,x)$ and $\mu^t( \cdot\,\vert\,y)$ are probability distributions over their action spaces $A$ and $B$, respectively. Then, given initial states $x_0$ and $y_0$, both players execute their own policies $\pi^t$ or $\mu^t$ for $L$ steps. At step $\ell\in \{0,\ldots,L-1 \}$, each player only observes its own state $x_{\ell}\in X$ or $y_\ell\in Y$, takes action $a_\ell$ or $b_\ell$ following its own policy $\pi^t$ or $\mu^t$, transits to next state $x_{\ell+1}$ or $y_{\ell+1}$ according to its own transition $P_1(\cdot\,\vert\,x_\ell,a_\ell)$ or $P_2(\cdot\,\vert\,y_\ell,b_\ell)$, and observes reward $r^t$ and local utility $g^t$ or $h^t$. Assume there is no dependence between functions $r^t$, $g^t$, and $h^t$ and they are independent of the underlying MDPs.

To define the learning objective, for the min-player in episode $t$ we introduce the occupancy measure $q_1^t$: $X\times A\times X\to [0,1]$ by $q_1^t(x,a,x') \DefinedAs \text{Prob} (x_\ell=x, a_\ell = a, x_{\ell+1} = x')$ for $x\in X_\ell$, describing the marginal probability of visiting $(x,a,x')$ when executing policy $\pi^t$ under the transition $P_1$. Similarly, we introduce the occupancy measure $q_2^t$: $Y\times B\times Y\to [0,1]$ for the max-player. We recall that a function $q$: $X\times A\times X\to [0,1]$ is an occupancy measure associated with policy $\pi$ and transition $P$ if and only if it satisfies two conditions~\citep{altman1999constrained}: (i) $\sum_{x\in X_\ell}\sum_{a\in A}\sum_{x'\in X_{\ell+1}} q(x, a, x') = 1$ for $\ell \in\{0,\ldots,L-1\}$; (ii) $\sum_{x\in X_{\ell-1}}\sum_{a\in A} q(x, a, x') =\sum_{a\in A}\sum_{x''\in X_{\ell+1}} q(x', a, x'')$ for $x'\in X_\ell$ and $\ell \in\{1,\ldots,L-1\}$. We denote by $\Delta(P)$ a set of valid occupancy measures under $P$,
\[
\Delta(P) \;\DefinedAs\; \big\{ q\!:\! X\times A\times X\to [0,1] \;\vert\; q\text{ satisfies (i) and (ii) as shown above}\big\}.
\]
It is worth noting that the occupancy measure set is convex and compact for finite MDPs~\citep{altman1999constrained}.
Using an occupancy measure $q$, we can express associated transition $P$ and policy $\pi$ as
\begin{equation}\label{eq.P_pi}
P(x'\,\vert\,x,a) \;=\;\frac{q(x,a,x')}{\sum_{x''\,\in\,X_{\ell+1}}q(x,a,x'') }
\; \text{ and } \;
\pi(a\,\vert\,x) \;=\; \frac{\sum_{x'\,\in\,X_{\ell+1}} q(x,a,x')}{\sum_{a\,\in\,A}\sum_{x''\,\in\,X_{\ell+1}}q(x,a,x'') }
\end{equation}
where $x\in X_\ell$. Slightly extending the notation $q$, we use it to represent the probability of visiting $(x,a)$, i.e., $q(x, a) = \sum_{x'\,\in\,X_{\ell+1}} q(x,a,x')$ for $x\neq x_L$. These properties imply that the problem of learning a policy equals learning the associated occupancy measure~\citep{zimin2013online}.

In episode $t$, given a min-policy $\pi^t$ and a max-policy $\mu^t$, we introduce the expected total reward,
\begin{equation}\label{eq.reward}
\begin{array}{lcl}
&& \!\!\!\! \!\!\!\! \!\!\!\! \!\! \displaystyle
\mathbb{E}_{P_1,P_2,\pi^t,\mu^t }
\sbr{\,
	\sum_{\ell\,=\,0}^{L-1} r^{t}(x_\ell,y_\ell,a_\ell,b_\ell) 
	\,} 
\;=\;
\sum_{\ell\,=\,0}^{L-1}\sum_{x\,\in\,X_\ell, \,y\,\in\,Y_\ell}\sum_{a\,\in\,A,\,b\,\in\,B} q_1^t(x,a) q_2^t(y,b) r^t(x,y,a,b)
\\[0.2cm]
&& \!\!\!\! \!\!\!\! \!\!\!\! \!\!
\displaystyle
\;\DefinedAs\;
\inner{q_1^t\cdot q_2^t}{r^t}
\end{array}
\end{equation}
where the expectation $\mathbb{E}$ is taken over the random state-action sequence $\{ (x_\ell,y_\ell,a_\ell,b_\ell) \}_{\ell \,=\,0}^{L-1}$; the action $a_\ell$ follows the policy $\pi^t(\cdot\,\vert\,x_\ell)$  in the state $x_\ell$ and the next state $x_{\ell+1}$ follows the transition $P_1(\cdot\,\vert\,x_\ell,a_\ell)$; the action $b_\ell$ follows the policy $\mu^t(\cdot\,\vert\,y_\ell)$  in the state $y_\ell$ and the next state $y_{\ell+1}$ follows the transition $P_2(\cdot\,\vert\,y_\ell,b_\ell)$. 
Similarly, we can define the expected total utilities as
\begin{subequations}\label{eq.utility}
	\begin{equation}\label{eq.utilityx}
	\mathbb{E}_{P_1,\pi^t}
	\sbr{\,\sum_{\ell\,=\,0}^{L-1} g_{x}^{t}(x_\ell,a_\ell)  \,} 
	\;=\;
	\sum_{\ell\,=\,0}^{L-1}\sum_{x\,\in\,X_\ell}\sum_{a\,\in\,A} q_1^t(x,a) g^t(x,a)
	\;\DefinedAs\;
	\inner{q_1^t}{g^t}
	\end{equation}
	\\[-0.5cm]
	\begin{equation}\label{eq.utilityy}
	\mathbb{E}_{P_2,\mu^t} 
	\sbr{\,\sum_{\ell\,=\,0}^{L-1} h^{t}(y_\ell,a_\ell) \,} 
	\;=\;
	\sum_{\ell\,=\,0}^{L-1}\sum_{y\,\in\,Y_\ell}\sum_{b\,\in\,B} q_2^t(y,b) h^t(y,b)
	\;\DefinedAs\;
	\inner{q_2^t}{h^t}.
	\end{equation}
\end{subequations}

In general, reward function $r^t$ and utility functions $g^t$ and $h^t$ all can change arbitrarily, i.e., being adversarial. However, even if we fix the opponent's policy, there is no algorithm for the player to achieve sublinear regret and constraint violation at the same time when the constraints are changing adversarially~\citep{mannor2009online}. Hence, we restrict the utility functions to be stochastic: $g^t(x,a) \DefinedAs g(x,a; \xi^t)$, $h^t(y,b) \DefinedAs h(y,b; \xi^t)$ with $\mathbb{E}\sbr{g^t(x,a)} = g(x,a)$ and  $\mathbb{E}\sbr{h^t(y,b)} = h(y,b)$, for any $x\in X$, $a\in A$ and $y\in Y$, $b\in B$, where $\xi^t$ is an independent random variable. 

\noindent\textbf{Learning Performance}. We now define the underlying constrained optimization problem and the solution concept for learning constrained MGs. Using the notion of occupancy measure, we formulate a constrained minimax problem in which the objective function is a sum of the expected total rewards over $T$ episodes and the constraint is on a sum of two agent' expected total utilities, 
\begin{equation}\label{eq.opt_wc}
\begin{array}{rcl}
\minimize\limits_{q_1\,\in\,\Delta(P_1)} \,\,\maximize\limits_{q_2\,\in\,\Delta(P_2)} &&\!\!\!\! \displaystyle \sum_{t\,=\,0}^{T-1} \inner{q_1\cdot q_2}{r^t} 
\;\;\; \subject \;\;\;
\inner{q_1}{g} \,+\,\inner{q_2}{h}\;\leq\; b
\end{array}
\end{equation}
where we take $b \in (0, {2L}]$ to avoid trivial cases since we note that $\langle{q_1},{g}\rangle$, $\langle{q_2},{h}\rangle\in [0,L]$. The coupled constraint is used to model the limited use of budget/resource for two players; multi-agent problems with a common constraint are often called~\emph{weakly-coupled} or~\emph{non-orthogonal} in the literature on CMDPs~\citep{meuleau1998solving,boutilier2016budget,wei2018online,salemi2018approximate,gagrani2020weakly} and constrained MGs~\citep{altman2008constrained,altman2009constrained,kulkarni2011generalized,singh2014characterization,kulkarni2017games}. We can generalize it to multiple or local side constraints, e.g., $\langle{q_1},{g}\rangle \leq b_1$ or $\langle{q_2},{h}\rangle\leq b_2$. When transitions $P_1$ and $P_2$ are known, the occupancy measure sets $\Delta(P_1)$ and $\Delta(P_2)$ define convex polytopes on $q_1$ and $q_2$. 

Let $(q_1^\star,q_2^\star)$ be a solution to Problem~\eqref{eq.opt_wc} in hindsight. The existence of $(q_1^\star,q_2^\star)$ follows from compactness of the constraint sets~\citep{neumann1928theorie,rosen1965existence}. It is standard to define an intuitive solution -- {constrained Nash equilibrium} -- via two conditions~\citep{altman2000constrained,daskalakis2020complexity}:
\begin{itemize}
	\item [(i)] $\displaystyle \sum_{t\,=\,0}^{T-1}\langle q_1^\star\cdot q_2^\star, r^t \rangle \;\leq\; \sum_{t\,=\,0}^{T-1} \langle q_1\cdot q_2^\star, r^t \rangle$ for any $q_1 \in \Delta(P_1)$ satisfying $\langle{q_1},{g}\rangle + \langle{q_2^\star},{h}\rangle\leq b$; 
	
	\item [(ii)] $\displaystyle \sum_{t\,=\,0}^{T-1}\langle q_1^\star\cdot q_2, r^t \rangle \;\leq\; \sum_{t\,=\,0}^{T-1}\langle q_1^\star\cdot q_2^\star, r^t \rangle$ for any $q_2\in\Delta(P_2)$ satisfying $\inner{q_1^\star}{g}+\inner{q_2}{h}\leq b$. 
\end{itemize}
Any unilateral deviation from the constrained Nash equilibrium will either break the constraint, or if it is not, then there is no benefit for this player.
With this solution concept, we define the regret for any algorithm that plays the game for $T$ episodes by 
\begin{equation}\label{eq.regret}
\text{Regret}(T) \;=\; \sum_{t\,=\,0}^{T-1} \rbr{ \inner{q_1^t \cdot q_2^\star}{r^t} - \inner{q_1^\star \cdot q_2^t}{r^t} }
\end{equation}
which adds two side optimality gaps, $\sum_{t\,=\,0}^{T-1}\inner{q_1^t \cdot q_2^\star}{r^t} - \inner{q_1^\star \cdot q_2^\star}{r^t}$ for the min-player and $\sum_{t\,=\,0}^{T-1}\langle{q_1^\star \cdot q_2^\star},{r^t}\rangle - \langle{q_1^\star \cdot q_2^t},{r^t}\rangle$ for the max-player, and two players take policies $\pi^t$ and $\mu^t$ in episode $t$ and they define occupancy measures $q_1^t$ and $q_2^t$ under the true transitions $P_1$ and $P_2$. This regret works in a notion of weak regret~\citep{brafman2002r,bai2020provable,xie2020learning} instead of the single-agent type regret~\citep{tian2020provably,bai2020near} which is statistically and computationally hard to bound sublinearly.

To measure the constraint satisfaction, we introduce the violation as a non-negative part of accumulated constraint violations $\langle{q_1^t},{g}\rangle+\langle{q_2^t},{h}\rangle-b$ over $T$ episodes,
\begin{equation}\label{eq.violation}
\text{Violation}(T) \;=\; \sbr{\sum_{t\,=\,0}^{T-1} \rbr{\inner{q_1^t}{g^t}+\inner{q_2^t}{h^t}-b}}_+.
\end{equation}

We next assume feasibility that ensures the existence of constrained Nash equilibrium~\citep{altman2000constrained}. Feasibility can be verified by a priori knowledge on feasible policies.

\begin{assumption}[Feasibility]\label{as.feasibility}
	There exists a joint policy $(\bar \pi,\bar\mu)$ associated to the occupancy measure $(\bar q_1,\bar q_2)$ and a constant $\xi>0$ such that $\inner{\bar q_1}{g}+\inner{\bar q_2}{h}+\xi\leq b$.
\end{assumption}

Having defined the learning performance, we will work with the occupancy measure in the online learning setting where the two players do not know the transition functions, only observe reward/utility functions at the end of each episode, repeatedly play the game for a fixed number of episodes to learn the constrained Nash equilibrium in hindsight.

\vspace*{-1ex}
\section{Proposed Algorithm}
\label{algorithm}

We present a variant of upper confidence reinforcement learning in Algorithm~\ref{UCB-MPD} -- an $\underline{\text{U}}$pper $\underline{\text{C}}$onfidence $\underline{\text{B}}$ound $\underline{\text{C}}$onstrained $\underline{\text{SA}}$ddle-$\underline{\text{P}}$oint $\underline{\text{O}}$ptimization (UCB-CSAPO) algorithm -- for learning constrained MGs. Conceptually, the algorithm works as the primal-dual policy optimization~\citep{efroni2020exploration,ding2020provably,chen2021primal} in the Lagrangian-based framework, which makes it a simple policy optimization algorithm. However, our primal update exploits the structure of constrained MGs to maintain two players' occupancy measures. The domain set of occupancy measures builds on the upper confidence bound exploration or optimism~\citep{jaksch2010near} regarding the estimated transition models using past trajectories. The dual update determines the penalty weight by collecting the possible constraint violation already acquired. In each episode, our algorithm has two key stages: (i) The generalized Lagrangian mirror descent step for updating the occupancy measures with optimism; (ii) The estimation of confidence sets on the occupancy measures.

\noindent\textbf{Generalized Lagrangian Mirror Descent Step}. The main idea of this step is to apply the online primal-dual mirror descent -- an algorithmic generalization of online mirror descent to the constrained problems~\citep{wei2020online}  -- to the constrained MG setting~\citep{altman2000constrained,altman2005zero,altman2008constrained,singh2014characterization}. Let us recall that the occupancy measures $q_1^t$ for the min-player and $q_2^t$ for the max-player are defined over the true transitions $P_1$ and $P_2$ in episode $t$. The primal update of our algorithm maintains two occupancy measures $\hat q_1^{\,t}$, $\hat q_2^{\,t}$ to estimate $q_1^t$, $q_2^t$, separately. Although $\hat q_1^{\,t}$, $\hat q_2^{\,t}$ do not necessarily come from the true transitions $P_1$, $P_2$, they propose a min-policy $\pi^t$ for the min-player and a max-policy $\mu^t$ for the max-player according to the occupancy measure's property~\eqref{eq.P_pi}, i.e., for all $(x,a) \in X\times A$ and $(y,b)\in Y\times B$,
\begin{equation}\label{eq.policy}
\pi^t(a\,\vert\,x) \;=\; \frac{\underset{x'}{\sum} \hat{q}_1^{\,t}(x,a,x')}{\underset{a,x''}{\sum} \hat q_1^{\,t}(x,a,x'') }
\; \text{ and }\;
\mu^t(b\,\vert\,y) \;=\; \frac{\underset{y'}{\sum} \hat q_2^{\,t}(y,b,y')}{\underset{b,y''}{\sum} \hat q_2^{\,t}(y,b,y'') }.
\end{equation}

We describe our Lagrangian-based design to update estimates $\hat q_1^{\,t}$ and $\hat q_2^{\,t}$ in an online fashion. 
Assume that the transitions $P_1$ and $P_2$ are known. We consider a one-episode constrained minimax problem based on reward/utility functions: $r^{t-1}$, $g^{t-1}$, $h^{t-1}$, revealed at the end of episode $t-1$,
\[
\begin{array}{rcl}
\minimize\limits_{q_1\,\in\,\Delta(P_1)} \,\,\maximize\limits_{q_2\,\in\,\Delta(P_2)}
&& \!\!\!\! \inner{q_1\cdot q_2}{r^{t-1}}
\;\;\; \subject \;\;\;
\inner{q_1}{g^{t-1}} \,+\,\inner{q_2}{h^{t-1}} \;\leq\; b
\end{array}
\]
where $\Delta(P_1)$ and $\Delta(P_2)$ are sets of valid occupancy measures under $P_1$ and $P_2$, respectively. 

It is standard to use the method of Lagrange multipliers~\citep{bertsekas2014constrained} to handle constraints by adding penalty terms, if any constraint violation appears, into the original objective, and formulate an unconstrained problem. This is found in constrained games with separate side constraints~\citep{pearsall1976lagrange} and multiple MDPs with coupled constraints~\citep{boutilier2016budget,wei2018online}. However, for constrained MGs either player can contribute to constraint violation $\langle{q_1},{g^{t-1}} \rangle+\langle{q_2},{h^{t-1}}\rangle- b$. It is important to specify which player should get such penalty terms~\citep{altman2009constrained,dai2020optimality}. We employ an attitude that the two players are jointly against the constraint while competing for rewards~\citep{altman2009constrained}. As a result, both would sacrifice their rewards to satisfy the constraint if any violation occurs. We approximate the violation for each player as: $\langle{q_1},{g^{t-1}}\rangle +\langle{\hat q_2^{\,t} },{h^{t-1}}\rangle- b$ for the min-player, and $\langle{\hat q_1^{\,t}},{g^{t-1}}\rangle +\langle{ q_2},{h^{t-1}}\rangle- b$ for the max-player. We formulate a generalized Lagrangian-type function,
\[
\begin{array}{rcl}
L^t(q_1,q_2;\lambda) 
&\DefinedAs&
\langle{q_1\cdot q_2},{r^{t-1}}\rangle
\\[0.2cm]
&&
\,+\,
\lambda \big( \langle{q_1},{g^{t-1}}\rangle +\langle{\hat q_2^{\,t} },{h^{t-1}}\rangle- b \big)
\,-\,
\lambda \big( \langle{\hat q_1^{\,t}},{g^{t-1}}\rangle +\langle{ q_2},{h^{t-1}}\rangle- b \big)
\end{array}
\]
where $q_1$ is the first primal variable for the min-player, $q_2$ is the second primal variable for the max-player, and $\lambda\geq 0$ works as the Lagrange multiplier or the dual variable in penalizing the min-player/max-player via the first/second $\lambda$-term.  Once we update $\lambda=\lambda^{t-1}$ from the last episode, we reach a constrained saddle-point problem, 
$\minimize_{q_1\,\in\,\Delta(P_1)}\,\maximize_{q_2\,\in\,\Delta(P_2)}\, L^t(q_1,q_2;\lambda^{t-1})$.

However, it is not feasible to take the domains $\Delta(P_1)$ and $\Delta(P_2)$ since the true transitions $P_1$ and $P_2$ are unknown. Instead, by the optimism in the face of uncertainty, we use their optimistic estimates $\Delta(k_1^t)$ and $\Delta(k_2^t)$ in sense that $q_1^t\in \Delta(k_1^t)$ and $q_2^t\in \Delta(k_2^t)$ hold with high probability in Lemma~\ref{lem.empiricalP}, where $\Delta(k_1^t)$ and $\Delta(k_2^t)$ are given by \eqref{eq.confidence}. Let $\hat q^{\,t} \DefinedAs (\hat q_1^{\,t},\hat q_2^{\,t})$ and $D(p\,\vert\,q) \DefinedAs \sum_{i} p_i \ln\frac{p_i}{q_i} - \sum_{i} (p_i-q_i)$ that is the unnormalized Kullback-Leibler (KL) divergence between two distributions $p$, $q$.
By a linear approximation of $L^t(q_1,q_2;\lambda^{t-1})$ at the previous iterate $(q_1^{t-1},q_2^{t-1})$, we update the primal variable via an online mirror descent step over the domains of $q_1$ and $q_2$,
\begin{equation}\label{eq.primal}
\begin{array}{rcl}
&&  \!\!\!\! \!\!\!\! \!\!\!\! \!\!
\hat q^{\,t}
\,\leftarrow\,
\displaystyle \argmin_{q_1\,\in\,\Delta(k_1^t)} \argmax_{q_2\,\in\,\Delta(k_2^t)} 
\Big(
V\, 
\big\langle{q_1\cdot \hat q_2^{\,t-1}+\hat q_1^{\,t-1}\cdot q_2},{r^{t-1}}\big\rangle 
\\[0.2cm]
&& \;\;\;\;  \;\;\;\;  \;\;\;\;  \;\;\;\;  \;\;\;\;  \;\;\;\;  \;\;\;\;  \;\;\;\;
\,+\,
\lambda^{t-1} \big( \langle{q_1},{g^{t-1}}\rangle 
- \langle{q_2},{h^{t-1}}\rangle\big)
\, +\, 
\displaystyle \eta^{-1} D\big(q\,\vert\,  \tilde{q}^{\,t-1}\big)\Big)
\end{array}
\end{equation}
where $V>0$ provides the tradeoff between the minimax objective and the constraint, $\eta>0$ is the learning rate, $D(\cdot\,\vert\,\cdot)$ is the unnormalized Kullback-Leibler divergence with a slightly abuse in a way that $D(q\,\vert\,q') \DefinedAs D(q_1\,\vert\,q_1')-D(q_2\,\vert\,q_2')$, $\tilde{q}_1^{\,t-1}$ and $\tilde{q}_2^{\,t-1}$ are mixing policies, e.g., 
\begin{equation}\label{eq.mixing}
\tilde{q}_1^{\,t-1}(x,a) \;=\; (1-\theta) \,\hat {q}_1^{\,t-1}(x,a)\,+\,\theta\,\frac{1}{|X_\ell||A|}
\end{equation}
for $(x,a)\in X_\ell\times A$, $\ell\in\{ 0,1,\ldots,L-1\}$, $\theta\in (0,1]$. 
The mixing step ensures the uniform boundedness of KL divergence and also adds extra exploration into policy search~\citep{wei2020online}. Moreover, we offer an efficient implementation of~\eqref{eq.primal} as solving a convex program in Appendix~\ref{ap.implementation}.

\begin{algorithm}[t]
	\caption{ $\underline{\text{U}}$pper $\underline{\text{C}}$onfidence $\underline{\text{B}}$ound 
		$\underline{\text{C}}$onstrained $\underline{\text{SA}}$ddle-$\underline{\text{P}}$oint $\underline{\text{O}}$ptimization (UCB-CSAPO) }
	\label{UCB-MPD}
	\begin{algorithmic}[1]
		\STATE \textbf{Input}: State/action spaces $(X,A)$ and $(Y,B)$, episode $T$, parameters $V$, $\eta$, $\theta$, and $p\in (0,1)$. 
		\STATE \textbf{Initialization}: 	
		The min-player: $\hat q_1^{\,0}(x,a,x') = \frac{1}{|X^\ell| |A||X^{\ell+1}|}$, $\forall(x,a,x')\in X^\ell\times A \times X^{\ell+1}, \ell \in [0, L-1]$; $n_1^1(x,a) = N_1^1(x,a) = 0$, $\forall(x,a)$; $m_1^1(x,a,x') = M_1^1(x,a,x') =\bar P_1^1(x'\,\vert\,x,a) =0$, $\forall(x,a,x')$.
		
		The max-player: $\hat q_2^{\,0}(y,b,y')  = \frac{1}{|Y^\ell| |B||Y^{\ell+1}|}$, $\forall (y,b,y')\in Y^\ell\times B\times Y^{\ell+1}, \ell \in [0, L-1]$; $n_2^1(y,b) = N_2^1(y,b) = 0$, $\forall (y,b)$; $m_2^1(y,b,y') = M_2^1(y,b,y') = \bar P_2^1(y'\,\vert\,y,b) =0$, $\forall (y,b,y')$.
		
		Let $r^0$, $g^0$, $h^0$ be zero functions, $\lambda^{0}$ be zero, and $k_1^1 = k_2^1=1$.
		
		\FOR{episode $t=1,\ldots,T$} 
		\STATE Update the primal variable $\hat q^{\,t}$ via~\eqref{eq.primal} and the dual variable $\lambda^t$ via~\eqref{eq.dual}.
		\STATE Compute the min-policy $\pi^t$ and the max-policy $\mu^t$ via~\eqref{eq.policy}.
		Execute them for $L$ steps and record trajectories $(x^0,a^0, x^1,\cdots, a^{L-1}, x^{L-1})$ and $(y^0,b^0, y^1,\cdots, b^{L-1}, y^{L-1})$, and reward/utility functions $r^t$, $g^t$, and $h^t$. 
		\STATE  Update local visitation counters at visited trajectories,
		\[
		n_1^{k_1^t} (x^\ell,a^\ell) \,\leftarrow\, n_1^{k_1^t} (x^\ell,a^\ell)+1
		\; 
		\text{ and }
		\;
		m_1^{k_1^t} (x^\ell,a^\ell, x^{\ell+1}) \,\leftarrow\, m_1^{k_1^t} (x^\ell,a^\ell,x^{\ell+1})+1
		\]
		\[
		n_2^{k_2^t} (y^\ell,b^\ell) \,\leftarrow\, n_2^{k_2^t} (y^\ell,b^\ell)+1
		\;
		\text{ and }
		\;
		m_2^{k_2^t} (y^\ell,b^\ell, y^{\ell+1}) \,\leftarrow\, m_2^{k_2^t} (y^\ell,b^\ell,y^{\ell+1})+1.
		\]
		
		\IF{ $n_1^{k_1^t}(x,a)\geq N_1^{k_1^t}(x,a)$ or $n_2^{k_2^t}(y,b)\geq N_2^{k_2^t}(y,b)$ for some $(x,a)\in X\times A$ or $(y,b)\in Y\times B$}
		\STATE Increase epoch counter by one, $k_1^{t+1} \leftarrow k_1^t+1$ or $k_2^{t+1} \leftarrow k_2^t+1$, and update global visitation counters,
		\[
		N_1^{k_1^{t+1}} (x,a) \,\leftarrow\, N_1^{k_1^t} (x,a)\,+\,n_1^{k_1^t}(x,a)
		\; \text{ or } \;
		N_2^{k_2^{t+1}} (y,b) \,\leftarrow\, N_2^{k_2^t} (y,b)\,+\,n_2^{k_2^t}(y,b)
		\]
		\[
		\!\!\!\!  \!\!\!\!  \!\!\!\!  \!\!\!\!  \!\!\!\!
		M_1^{k_1^{t+1}}\! (x,a,x') \leftarrow M_1^{k_1^t}\! (x,a,x') + m_1^{k_1^t}(x,a,x')
		\,	\text{ or } \,
		M_2^{k_2^{t+1}}\! (y,b,y') \leftarrow M_2^{k_2^t}\! (y,b,y') + m_2^{k_2^t}(y,b,y').
		\]
		Update the confidence bounds for $\Delta(k_1^t)$ or $\Delta(k_2^t)$ in~\eqref{eq.confidence}, and set $n_1^{k_1^{t+1}}(x,a) = m_1^{k_1^{t+1}}(x,a,x') = 0$ for all $(x,a)$ and $(x,a,x')$ or $n_2^{k_2^{t+1}}(y,b) = m_2^{k_2^{t+1}}(y,b,y') = 0$ for all $(y,b)$ and $(y,b,y')$.
		\ELSE
		\STATE Set either $k_1^{t+1} = k_1^{t}$ or $k_2^{t+1} = k_2^{t}$.
		\ENDIF
		
		\ENDFOR
	\end{algorithmic}
\end{algorithm}

Once we obtain $\hat q^{\,t}$, we next perform the dual update. If we treat two $\lambda$-related regularization terms in $L^t(\hat q_1^{\,t}, \hat q_2^{\,t};\lambda)$ separately, then gradient ascent/descent over either $\lambda$ leads to the same update rule using the constraint violation $\langle{\hat q_1^{\,t}},{g^{t-1}}\rangle +\langle{\hat q_2^{\,t}},{h^{t-1}}\rangle-b$. Hence, the dual update works in the usual way by adding up all past constraint violations,
\begin{equation}\label{eq.dual}
\lambda^{t} \;=\; \max \big( \lambda^{t-1} \,+\, (\langle{\hat q_1^{\,t}},{g^{t-1}}\rangle +\langle{\hat q_2^{\,t}},{h^{t-1}}\rangle-b\, ),\, 0 \big).
\end{equation}
The dual update~\eqref{eq.dual} increases $\lambda^{t-1}$ when $\hat q^{\,t}$ violates the approximate constraint $\langle{q_1},{g^{t-1}}\rangle +\langle{q_2},{h^{t-1}}\rangle\leq b$. It penalizes both players by yielding individual gains to the constraint satisfaction. The dual update finds uses in constrained MDP problems~\citep{efroni2020exploration,ding2020provably}. 

\noindent\textbf{Estimation of Confidence Sets}. To deal with unknown transitions $P_1$ and $P_2$, we employ the upper confidence bound~\citep{jaksch2010near,neu2010online} to estimate occupancy measure sets $\Delta(P_1)$, $\Delta(P_2)$. We exploit players' history trajectories to estimate their true transitions: $P_1$, $P_2$, and describe estimation uncertainty as confidence sets. The estimation proceeds in epochs as follows.

Let the epoch index for the min-player be $k_1\in \{1,2,\ldots \}$ and the epoch index for the max-player be $k_2\in\{1,2,\ldots \}$. We may represent them by $k_1^t$ and $k_2^t$ for showing the dependence on episode $t$. The epoch counters work in the following way. For each player, e.g., the min-player, we denote by $N_{1}^{k_1}(x,a)$ and $M_1^{k_1}(x,a,x')$ the total numbers of visitations to $(x,a)$ and $(x,a,x')$ before epoch $k_1$, respectively; we represent the total numbers of visitations to $(x,a)$ and $(x,a,x')$ in epoch $k_1$ by $n_{1}^{k_1}(x,a)$ and $m_1^{k_1}(x,a,x')$, respectively; If there exists $(x,a)$ such that $n_{1}^{k_1}(x,a)\geq N_{1}^{k_1}(x,a)$, then we set a new epoch by increasing $k_1$ by one. Similarly, we define $N_{2}^{k_2}(y,b)$, $M_2^{k_2}(y,b,y')$,  $n_{2}^{k_2}(y,b)$, and $m_2^{k_2}(y,b,y')$ for the max-player. Using the defined epoch and visitation counters, we empirically estimate the true transitions $P_1$ or $P_2$ in epoch $k_1$ or $k_2$ by 
\[
\bar P_1^{k_1} (x'\,\vert\,x,a) \; = \;  \frac{M_1^{k_1}(x,a,x')}{\max (1,N_{1}^{k_1}(x,a))}
\; \text{ and } \;
\bar P_2^{k_2} (y'\,\vert\,y,b) \; = \; \frac{M_2^{k_2}(y,b,y')}{\max (1,N_{2}^{k_2}(y,b))}
\]
for all $(x,a,x') \in X\times A\times X$ and $(y,b,y')\in Y\times B\times Y$. 

Let the confidence set of epoch $k_1$ for the min-player be $\calP_1^{k_1}$ and the confidence set of epoch $k_2$ for the max-player be $\calP_2^{k_2}$. We take $\calP_1^{k_1}$ and $\calP_2^{k_2}$ as collections of transitions that deviate from the empirical ones at most $\epsilon_1^{k_1}$ and $\epsilon_2^{k_2}$,
\[
\calP_1^{k_1} \;=\; \big\{ \hat P_1\,\big\vert\, \Vert \hat P_1(\cdot\,\vert\,x,a) - \bar P_1^{k_1}(\cdot\,\vert\,x,a) \Vert_1 \leq \epsilon_1^{k_1}, \forall(x,a) \big \}
\]
\vspace*{-0.4cm}
\[
\calP_2^{k_2} \;=\; \big\{ \hat P_2\,\big\vert\, \Vert \hat P_2(\cdot\,\vert\,y,b) - \bar P_2^{k_2}(\cdot\,\vert\,y,b) \Vert_1 \leq \epsilon_2^{k_2}, \forall (y,b) \big\}
\]
where we take
$\epsilon_1^{k_1}(x,a) = \sqrt{ \frac{ 2|X_{\ell(x)+1}| \log(T|A||X|/\delta)}{\max (1, N_1^{k_1}(x,a)) } } $ and $\epsilon_2^{k_2}(y,b) = \sqrt{ \frac{ 2|Y_{\ell(y)+1}| \log(T|B||Y|/\delta)}{\max (1, N_2^{k_2}(y,b)) } } $, 
$\ell(x)$ and $\ell(y)$ are the layers that certain states belong to, and $\delta\in (0,1)$. We recall the occupancy measure sets $\Delta(P_1)$ or $\Delta(P_2)$ that are induced by the true transitions $P_1$ or $P_2$. We generalize this notion to define $\Delta(\calP_1^{k_1^t})$ or $\Delta(\calP_2^{k_2^t})$ as collections of all possible occupancy measures that are induced by the estimated transitions $\hat P_1 \in\calP_1^k$ or $\hat P_2\in\calP_2^k$,
\begin{equation}\label{eq.confidence}
\Delta(k_1^t) \DefinedAs \Delta(\calP_1^{k_1^t})
\;\text{ or }\;
\Delta(k_2^t) \DefinedAs \Delta(\calP_2^{k_2^t});
\;\; \text{see~\eqref{eq.deltaset} in Appendix \ref{ap.implementation} for explicit forms.}
\end{equation}

\begin{lemma}\label{lem.empiricalP}
	Fix $\delta\in (0,1)$. With probability $1-\delta$, $\Delta(P_1)\subset \Delta(\calP_1^{k_1})$ and $\Delta(P_2)\subset \Delta(\calP_2^{k_2})$ for all $k_1,k_2\in\{ 1,2,\ldots \}$.
\end{lemma}
The proof of Lemma~\ref{lem.empiricalP} follows the confidence bound construction; we provide it in Appendix~\ref{ap.empiricalP}. For all epoch $k_1^t$ or $k_2^t$ (episode $t$), the true transitions $P_1$ and $P_2$ are contained in $\calP_1^{k_1^t}$ and $\calP_2^{k_2^t}$, respectively, with high probability. This supports the primal update~\eqref{eq.primal} such that both players are optimistically searching solutions in a large but tractable domain.  

\vspace*{-1ex}
\section{Performance Guarantees}
\label{regret}

In Theorem~\ref{thm.main}, we present our main theoretical result on the regret and the constraint violation for Algorithm~\ref{UCB-MPD}. We recall the total number of games played by the algorithm $T$, the size of state/action spaces of the min-player $|X|$, $|A|$, and the size of state/action spaces of the max-player $|Y|$, $|B|$.

\begin{theorem}[Regret Bound and Constraint Violation]\label{thm.main}
	Let Assumption~\ref{as.feasibility} hold. 
	Fix $p\in\rbr{0,1}$ and $T\geq \max(|X||A|,|B||Y|)$. 
	In Algorithm~\ref{UCB-MPD}, we set
	$V=L\sqrt{T}$, $\eta = 1/(TL)$, and $\theta = 1/T$.
	Then, with probability $1-p$, the regret~\eqref{eq.regret} and the constraint violation~\eqref{eq.violation} satisfy 
	\[
	\begin{array}{rcl}
	\text{\normalfont Regret}(T),\; \text{\normalfont Violation}(T) & \leq & 
	\tilde O \big(\, (|X|+|Y|) \,L \sqrt{T(|A|+|B|)} \,\big)
	\end{array}
	\]
	where $\tilde O(\cdot)$ hides the logarithmic factor $\log\tfrac{1}{p}$.
\end{theorem}

In Theorem~\ref{thm.main}, we prove that UCB-CSAPO enjoys $O(\sqrt{T})$ regret and $O(\sqrt{T})$ constraint violation using appropriate algorithm parameters $\{V,\eta,\theta,p\}$ and Assumption~\ref{as.feasibility}; see Appendix~\ref{proof} for proof. Our bounds have the optimal dependence on the total number of episodes $T$ up to some logarithmic factors. The $\sqrt{ |A| +  |B|} $ dependence matches the existing lower bound for the single-player case~\citep{bai2020provable}. The only suboptimal dependence comes from $|X|$, $|Y|$ that also exists in existing unconstrained loop-free stochastic shortest path problems~\citep{rosenberg2019online}. It is straightforward to remove knowledge of $T$ by using the doubling trick while not altering our bounds up to logarithmic factors~\citep{rakhlin2013online}. 

We see that Assumption~\ref{as.feasibility} does not impose any restrictions on rewards. Hence, UCB-CSAPO is robust against adversarial reward functions.
Moreover, Theorem~\ref{thm.main} carries to other settings, e.g., constrained MGs with side constraints; see Appendix~\ref{ap.CMG_sc}. 

\vspace*{-1ex}
\section{Concluding Remarks}
\label{conclusion}

We have examined an episodic two-player zero-sum constrained Markov game (MG) with independent transition functions. In our setup, transition functions are unknown to agents, reward functions are adversarial, and utility functions are stochastic. We have proposed the first provably efficient algorithm for playing constrained MGs with $O(\sqrt{T})$ regret and constraint violation. Our algorithm provides a principled extension of the upper confidence reinforcement learning to deal with coupled constraints in constrained MGs. We also remark that the developed algorithmic framework can be readily applied to learning other constrained MGs, e.g., the ones that involve a single controller. 

Our work opens up many interesting directions for future work, such as sharper algorithms with sample complexity lower bounds, constrained rational algorithms, and how to perform safe exploration in other models of constrained MGs.

\newpage

\vspace*{-2ex}
\section*{Acknowledgments}

The work of D. Ding and M. R. Jovanovi\'c was supported in part by the National Science Foundation under awards
ECCS-1708906 and 1809833. Part of this work was done while D. Ding was with the University of Southern California. We also thank NeurIPS 2022 reviewers for providing helpful comments. 

\bibliography{dd-bib}

\newpage

\onecolumn

~\\
\centerline{{\fontsize{14}{14}\selectfont \textbf{Supplementary Materials for }}}

\vspace{6pt}
\centerline{{\fontsize{14}{14}\selectfont
		\textbf{``Provably Efficient Generalized Lagrangian }}}

\vspace{6pt}
\centerline{\fontsize{14}{14}\selectfont \textbf{
		Policy Optimization for Safe Multi-Agent Reinforcement Learning''}}
\vspace{10pt}

\section{Related Work}\label{ap.related}

Safety constraints have gained increasing attention in the literature on multi-agent reinforcement learning (RL); see surveys~\citep{busoniu2008comprehensive,bucsoniu2010multi,zhang2019multi,oroojlooyjadid2019review,yang2020overview,schmidt2022introduction}. We first discusss some related work in framework of Markov games (MGs)~\citep{shapley1953stochastic,littman1994markov}.

\noindent\textbf{Constrained MGs}. 
Our work is closely related to safe multi-agent RL in constrained MGs. 
The constrained MGs generalize constrained MDPs~\citep{altman1999constrained} to multiple agents and Markov/stochastic games~\citep{shapley1953stochastic,littman1994markov} to account for constraints. The Nash equilibrium for constrained MGs have been studied in~\cite{altman2000constrained,gomez2003saddle,altman2005zero,alvarez2006existence,altman2007constrained,altman2008constrained,altman2009constrained,singh2014characterization} using the notion of \emph{constrained Nash equilibrium} (which generalizes the concept of \emph{generalized Nash equilibrium} in static games~\citep{arrow1954existence} to MGs) by assuming some particular transition models and constraints on reward/utility functions \emph{a priori}. More general studies include~\cite{yaji2015necessary,zhang2019discrete,wei2020discrete,wei2021constrained,zhang2021constrained}. These results are not applicable to the RL setting where transition models and reward/utility functions are unknown, and only a finite number of samples are available. Recently, asymptotic convergence in learning constrained MGs was examined in~\cite{hakami2015learning,jiang2020finding} but sample efficiency, constraint satisfaction, and exploration were not fully addressed. Our development fills this gap by adding built-in exploration mechanisms under constraints and proving the first non-asymptotic convergence for learning constrained Nash equilibria. We notice that learning general equilibria with non-asymptotic convergence was studied by~\cite{
	chenfinding}, which was concurrent to us since this work was under review in May 2022. 

\noindent\textbf{Constrained MDPs}. Our work is also pertinent to a rich RL literature on learning unknown constrained MDPs~\citep{zheng2020constrained,qiu2020upper,kalagarla2020sample,bai2020model,chow2017risk,tessler2018reward,ding2020natural,ding2020provably,ding2022convergence,ding2022convergenceACC,ding2022policy,wachi2020safe,efroni2020exploration,brantley2020constrained,chen2021primal,liu2021learning,ying2021dual,liu2021fast,bai2021achieving,zhao2021primal,li2021faster,chen2022learning}. While these results provide provably efficient algorithms regarding regret and constraint satisfaction in the single-agent setting, they are not applicable to our multi-agent game being played under constraints, because of the \emph{non-convexity} nauture of constrained multi-agent policy optimization and the \emph{non-stationary} environment each agent is facing. An extended line of work on constrained MDPs focuses on cooperative multi-agent learning under constraints and most efforts study the case where multiple agents have independent MDPs with a coupled budget/resource constraint~\citep{meuleau1998solving,boutilier2016budget,wei2018online,de2020risk,gagrani2020weakly}. All these results assume that transition models or system dynamics are known. Only a few studies considered the shared MDP case~\citep{diddigi2019actor,ludecentralized,parnika2021attention,gu2021multi}, but they either lack of theoretical guarantees or do not handle exploration. In contrast, our work focuses on the MG setting with unknown transition models, and attacks the exploration challenge directly.

\noindent\textbf{Single-agent RL in MDPs \& multi-agent RL in MGs}.
A considerable literature has provided sample-efficient online RL methods in single-agent and multi-agent unconstrained RL settings; see recent summaries in~\cite{foster2021statistical,du2021bilinear,jin2021bellman} for single-agent RL and~\cite{jin2021power,jin2021v,song2021can} for multi-agent RL. However, it is largely open to extend those sample-efficient online RL methods to constrained MGs due to several technical challenges. First, since the Bellman optimality fails even in constrained MDPs~\citep{piunovskiy2000constrained,borkar2005actor} and the optimal constrained policy is often stochastic~\citep{altman1999constrained}, value-based RL methods are not suitable. Second, applying policy-based RL methods often warrants solving constrained policy optimization problems that are not convex~\citep{achiam2017constrained,ding2020natural}, not mentioning multi-agent policy optimization problems. Third, designing a sample-efficient online RL algorithm for constrained MGs has to deal with the fundamental exploitation/exploration tradeoff under constraints~\citep{efroni2020exploration,brantley2020constrained,ding2020provably}. Despite some recent progress in dealing with each technical issue individually, it is crucial to address them together for multi-agent RL in constrained MGs. In this work, we offer the first positive answer by identifying a class of zero-sum constrained MGs, establishing a new policy optimization algorithm with online exploration for learning such games, and proving near-optimal sample efficiency.

\section{Proof Sketch of Theorem~\ref{thm.main}}
\label{proof}

\noindent\textbf{Regret Analysis}. 
We recall that our algorithm maintains the occupancy measures $(\hat q_1^{\,t}, \hat q_2^{\,t})$ for estimating policies $(\pi^t,\mu^t)$ and Problem~\eqref{eq.opt_wc} defines the comparison solution $(q_1^\star,q_2^\star)$ in hindsight. Naturally, we decompose the regret~\eqref{eq.regret} into two side regrets for both players by inserting $\langle q_1^\star\cdot q_2^\star,r^t\rangle $. By the occupancy measures $(q_1^t, q_2^t)$ associated with $(\pi^t,\mu^t)$ under the true transitions $P_1$ and $P_2$, we further decompose two side regrets into two terms by inserting $\langle \hat q_1^{\,t}\cdot  q_2^\star,r^{t}\rangle$ and $\langle q_1^\star\cdot \hat q_2^{\,t},r^{t}\rangle$, individually. Specifically, we have
\[
\begin{array}{rcl}
\text{Regret}(T) 
&=&\displaystyle 
\underbrace{
	\sum_{t\,=\,0}^{T-1} \big\langle{ \hat q_1^{\,t}\cdot  q_2^\star- q_1^\star\cdot \hat q_2^{\,t} },{r^{t}}\big\rangle 
}_{\hat{\text{\normalfont Regret}}(T)}
\,+\,
\underbrace{
	\sum_{t\,=\,0}^{T-1} \inner{( q_1^t-\hat q_1^{\,t}) \cdot q_2^\star}{r^t} 
}_{\text{\normalfont Error}_1}
\,+\,
\underbrace{
	\sum_{t\,=\,0}^{T-1} \inner{q_1^\star \cdot (\hat q_2^{\,t}-q_2^t)}{r^t} 
}_{\text{\normalfont Error}_2}
\end{array}
\]
where ${\hat{\text{\normalfont Regret}}(T)}$ depicts a regret of an online primal-dual mirror descent problem, ${\text{\normalfont Error}_1}$ is the error of using $\hat q_1^{\,t}$ for the min-player, and ${\text{\normalfont Error}_2}$ is the error of using $\hat q_2^{\,t}$ for the max-player.

We begin with a relatively standard lemma on estimation errors of $\hat q_1^{\,t}$, $\hat q_2^{\,t}$; we prove it in Appendix~\ref{ap.qqdifference}. 

\begin{lemma}\label{lem.qqdifference}
	Fix $\delta\in(0,1)$. Then, with probability $1-2\delta$,
	\[
	\begin{array}{rcl}
	\displaystyle\sum_{t\,=\,0}^{T-1} \norm{\hat q_1^{\,t}-q_1^t}_1
	&\leq& O \Big( L|X|\sqrt{T |A| \log \tfrac{T|X||A|}{\delta}}\Big)
	\\[0.2cm]
	\displaystyle\sum_{t\,=\,0}^{T-1} \norm{\hat q_2^{\,t}-q_2^t}_1
	&\leq& O\Big( L|Y|\sqrt{T |B| \log\tfrac{T|Y||B|}{\delta}}\Big).
	\end{array}
	\]
\end{lemma}
\vspace*{-0.2cm}

We note that $r^t\in [0,1]$, $q_2^\star$ is a probability distribution, and 
${\text{\normalfont Error}_1} 
=
\sum_{t\,=\,0}^{T-1} \inner{( q_1^t-\hat q_1^{\,t} ) \cdot q_2^\star}{r^t} 
\leq
\sum_{t\,=\,0}^{T-1} \norm{q_1^t-\hat q_1^{\,t} }_1.$
Application of Lemma~\ref{lem.qqdifference} yields the following bounds on ${\text{\normalfont Error}_1}$ and ${\text{\normalfont Error}_2}$. 
\begin{lemma}\label{lem.error12}
	Fix $\delta\in(0,1)$. Then, with probability $1-2\delta$,
	\[
	{\text{\normalfont Error}_1} \; \leq \; O\Big(L|X|\sqrt{T|A| \log\tfrac{T|X||A|}{\delta}}\Big)
	\;  \text{ and } \;
	{\text{\normalfont Error}_2} \; \leq \; O\Big(L|Y|\sqrt{T|B| \log\tfrac{T|Y||B|}{\delta}}\Big).
	\]
\end{lemma}

We next bound ${\hat{\text{\normalfont Regret}}(T)}$ by establishing an upper bound in Lemma~\ref{lem.gap_pd} first that is crucial to our regret analysis. The proof idea of Lemma~\ref{lem.gap_pd} is similar to the analysis of online constrained convex optimization~\citep{yu2017online,wei2020online}. A distinction is that we analyze the primal update~\eqref{eq.primal} via a new property of KL divergence for the minimax objective; see it in Appendix~\ref{ap.gap_pd}. 
\begin{lemma}\label{lem.gap_pd}
	Fix $\delta\in (0,1)$. Then, with probability $1-\delta$,
	\[
	\begin{array}{rcl}
	\displaystyle
	{\hat{\text{\normalfont Regret}}(T)}
	&\leq& \displaystyle V^{-1}
	\sum_{t\,=\,0}^{T-1}
	\lambda^{t} \big(\langle{ q_1^\star},{g^{t}}\rangle 
	+ \langle{ q_2^\star},{h^{t}}\rangle -b\big)
	\\[0.2cm]
	&&\displaystyle 
	\,+\, (\eta V)^{-1} L (1+\theta T)\big(\log (|X||A|) +\log (|Y||B|)\big) 
	\,+\,(2 V^{-1}L+4\theta +\eta V) LT.
	\end{array}
	\]
\end{lemma}

Lemma~\ref{lem.gap_pd} establishes an upper bound relying on a stochastic process of duals $\{\lambda^t,t\geq 0\}$. To analyze this bound, we establish the boundedness of $\lambda^t$ in Lemma~\ref{lem.lambda} first. Then, we apply a general Azuma-Hoeffding inequality for supermartingales in Lemma~\ref{lem.violation}. We delay their proofs to Appendix~\ref{ap.violation}.

\begin{lemma}\label{lem.lambda}
	Let Assumption~\ref{as.feasibility} hold. 
	Fix $\delta\in (0,1)$. For any integer $t_0>0$, with probability $1-T\delta$,
	\[
	\lambda^t \; \leq\; \displaystyle \Theta + 2 t_0 L + t_0 \frac{64 L^2}{\xi} \log \rbr{\frac{128 L^2}{\xi}} + t_0 \frac{64 L^2}{\xi} \log\frac{1}{\delta}
	\]
	for all $t = 1,\ldots,T$, where $\xi >0$ and
	\[
	\begin{array}{rcl}
	\Theta &\DefinedAs& t_0
	\rbr{\tfrac{1}{2}\xi  +2 L} \,+\,\tfrac{4L^2 + (8\theta +2\eta V
		+2 )VL}{\xi}
	\,+\,\tfrac{2L (\log(|X||A|/\theta)+\log(|Y||B|/\theta)) }{t_0\xi\eta}.
	\end{array}
	\]
\end{lemma}

\begin{lemma}\label{lem.violation}
	Let Assumption~\ref{as.feasibility} hold.
	Fix $\delta\in (0,1)$. 
	For any integer $t_0>0$, with probability $1-2T\delta$,
	\[
	\sum_{t\,=\,0}^{T-1}
	\lambda^{t} \big(\langle{ q_1^\star},{g^{t}}\rangle 
	+ \langle{ q_2^\star},{h^{t}}\rangle -b\big)
	\;\leq\; \sqrt{2T c^2 \log (1/(\delta T))}
	\]
	where
	$c \DefinedAs  2\Theta L + 4 t_0 L^2 + \tfrac{128 t_0 L^3}{\xi} \rbr{\log \rbr{\tfrac{128 L^2}{\xi}} +\log\tfrac{1}{\delta}}$ 
	and
	$\xi >0$. 
\end{lemma}

We now ready to conclude a bound on ${\hat{\text{\normalfont Regret}}(T)}$ by combining Lemma~\ref{lem.violation} and Lemma~\ref{lem.gap_pd}.
\begin{theorem}\label{thm.regret_hat}
	Let Assumption~\ref{as.feasibility} hold. Fix $T\geq \max(|X||A|,|B||Y|)$. Let $V=L\sqrt{T}$, $\eta = 1/(TL)$, $t_0=\sqrt{T}$, and $\theta = 1/T$. Then,
	with probability $1-2T\delta$ it holds that
	\[
	{\hat{\text{\normalfont Regret}}(T)}\;\leq\;
	\tilde O \big({(|X|+|Y|) L \sqrt T}\big).
	\]
\end{theorem}
\begin{proof} 
	Using the given parameters $V$, $\eta$, $t_0$, and $\theta$ for Lemma~\ref{lem.gap_pd}, ${\hat{\text{\normalfont Regret}}(T)}$ is upper bounded by 
	$\frac{1}{L\sqrt{T}}
	\sum_{t\,=\,0}^{T-1}
	\lambda^{t} \big(\langle{ q_1^\star},{g^{t}}\rangle 
	+ \langle{ q_2^\star},{h^{t}}\rangle -b\big)
	+ \tilde O(L\sqrt{T })$
	with probability $1-\delta$.
	We note that $\Theta\leq	\tilde O(L^2\sqrt{T})$ and $T\geq \max(|X||A|,|B||Y|)$.	
	Using parameters in Lemma~\ref{lem.violation}, with probability $1-2T\delta$,
	\[
	\sum_{t\,=\,0}^{T-1}
	\lambda^{t} \big(\langle{ q_1^\star},{g^{t}}\rangle 
	+ \langle{ q_2^\star},{h^{t}}\rangle -b\big)
	\;\leq\; \tilde O(L^3 T).
	\] 
	We complete the proof by noting $L\leq |X|+|Y|$. 
\end{proof}

We conclude the regret bound in Theorem~\ref{thm.main} by combining Lemma~\ref{lem.error12} and Theorem~\ref{thm.regret_hat}, and $\delta=p/(2T)$.

\noindent\textbf{Constraint Violation Analysis}. 
We begin with a decomposition using the auxiliary occupancy measures $(q_1^t, q_2^t)$. By inserting $\langle \hat q_1^{\,t},g^{t}\rangle$ and $\langle \hat q_2^{\,t},h^{t}\rangle$ into $\text{Violation}(T) $, we have
\[
\begin{array}{rcl}
\text{Violation}(T) 
&=&
\underbrace{\sbr{\sum_{t\,=\,0}^{T-1} \rbr{\inner{\hat q_1^{\,t}}{g^t}+\inner{\hat q_2^{\,t}}{h^t}-b}}_+}_{\hat{\text{\normalfont Violation}}(T)}
\,+\,
\underbrace{
	\sum_{t\,=\,0}^{T-1} \inner{q_1^t-\hat q_1^{\,t}}{g^t}
}_{\text{\normalfont Error}_3}
\, + \,
\underbrace{
	\sum_{t\,=\,0}^{T-1}\inner{q_2^t-\hat q_2^{\,t}}{h^t}
}_{\text{\normalfont Error}_4}.
\end{array}
\]

Similar to Lemma~\ref{lem.error12}, we can prove the following bounds on ${\text{\normalfont Error}_3}$ and ${\text{\normalfont Error}_4}$.
\begin{lemma}\label{lem.error34}
	Fix $\delta\in(0,1)$. Then, with probability $1-2\delta$,
	\[
	{\text{\normalfont Error}_3} \; \leq \; O\Big( L|X|\sqrt{T|A| \log\tfrac{T|X||A|}{\delta}}\Big)
	\; \text{ and } \;
	{\text{\normalfont Error}_4} \; \leq \; O\Big(L|Y|\sqrt{T|B| \log\tfrac{T|Y||B|}{\delta}}\Big).
	\]
\end{lemma}

We next bound ${\hat{\text{\normalfont Violation}}(T)}$ by applying the epoch property~\citep{jaksch2010near}; see a proof in Appendix~\ref{ap.violation_hat}.

\begin{theorem}\label{thm.violation_hat}
	Let $V=L\sqrt{T}$, $\eta = 1/(TL)$, $t_0=\sqrt{T}$, and $\theta = 1/T$. Then,
	\[
	\begin{array}{rcl}
	{\hat{\text{\normalfont Violation}}(T)}
	&\leq&\displaystyle
	\lambda^T
	\,+\, 
	\frac{2 }{T-1}
	\sum_{t\,=\,1}^{T} \lambda^{t-1}
	\,+\,
	\tilde O\big(L\sqrt{ T(|X||A|+|Y||B|)}\big).
	\end{array}
	\]
\end{theorem}

To get the violation bound, we apply Lemma~\ref{lem.lambda} to Theorem~\ref{thm.violation_hat}, use Lemma~\ref{lem.error34}, and take $\delta=p/(2T)$.

\section{Efficient Implementation of~\eqref{eq.primal}}\label{ap.implementation}

In this section, we provide an efficient implementation for the primal update~\eqref{eq.primal}. 

Since the minimax objective in the primal update~\eqref{eq.primal} is separable for two players, it is equivalent to update two occupancy measures individually via
\begin{subequations}\label{eq.primal12}
	\begin{equation}\label{eq.primal1}
	\hat q_1^{\,t}
	\;=\;
	\displaystyle \argmin_{q_1\,\in\,\Delta(k_1^t)}  \;
	V\, \big\langle {q_1\cdot \hat q_2^{\,t-1}},{r^{t-1}}\big\rangle 
	\,+\,
	\lambda^{t-1} \langle{q_1},{g^{t-1}}\rangle 
	\, +\, 
	\displaystyle \eta^{-1} D\big(q_1\,\vert\,  \tilde{q}_1^{\,t-1}\big)
	\end{equation}
	\begin{equation}\label{eq.primal2}
	\hat q_2^{\,t}
	\;=\;
	\displaystyle \argmax_{q_2\,\in\,\Delta(k_2^t)} \;
	V\, \big\langle{\hat q_1^{\,t-1}\cdot q_2},{r^{t-1}}\big\rangle 
	\,-\,\lambda^{t-1} \langle{q_2},{h^{t-1}}\rangle 
	\, -\, 
	\displaystyle \eta^{-1} D\big(q_2\,\vert\,  \tilde{q}_2^{\,t-1}\big).
	\end{equation}
\end{subequations}

Note that $\langle {q_1\cdot \hat q_2^{\,t-1}},{r^{t-1}}\rangle  = \langle {q_1 },{ \hat q_2^{\,t-1}\cdot r^{t-1}}\rangle$ and $\langle{\hat q_1^{\,t-1}\cdot q_2},{r^{t-1}}\rangle =\langle{ q_2},{\hat q_1^{\,t-1}\cdot r^{t-1}}\rangle$. Let 
\[
\phi_1^{t-1} \; \DefinedAs\; V\, \hat q_2^{\,t-1}\cdot r^{t-1} \,+\, \lambda^{t-1} g^{t-1}
\;\text{ and }\;
\phi_2^{t-1} \; \DefinedAs\; -V\, \hat q_1^{\,t-1}\cdot r^{t-1} \,+\, \lambda^{t-1} h^{t-1}.
\]
We can express~\eqref{eq.primal12} in a more compact form,
\begin{subequations}\label{eq.primal12c}
	\begin{equation}\label{eq.primal1c}
	\hat q_1^{\,t}
	\;=\;
	\displaystyle \argmin_{q_1\,\in\,\Delta(k_1^t)}  \;
	\eta \, \langle{q_1},{\phi_1^{t-1}}\rangle 
	\, +\, 
	\displaystyle D\big(q_1\,\vert\,  \tilde{q}_1^{\,t-1}\big)
	\end{equation}
	\begin{equation}\label{eq.primal2c}
	\hat q_2^{\,t}
	\;=\;
	\displaystyle \argmin_{q_2\,\in\,\Delta(k_2^t)} \;
	\eta \, \langle{q_2},{\phi_2^{t-1}}\rangle 
	\, +\, 
	\displaystyle D\big(q_2\,\vert\,  \tilde{q}_2^{\,t-1}\big)
	\end{equation}
\end{subequations}
where we flip the $\argmax$ in~\eqref{eq.primal2} to write $\argmin$ in~\eqref{eq.primal2c} and scale both objectives by multiplying $\eta>0$.

Now, we state an efficient implementation for the primal update~\eqref{eq.primal} by solving convex optimization problems.
The proof is based on the method of Lagrange multipliers and the Lagrange duality theory; they also find uses in the literature~\citep{zimin2013online,rosenberg2019online,jin2020learning}.

\begin{lemma}[Efficient Implementation]
	The primal update~\eqref{eq.primal} is equivalent to
	\begin{subequations}
		\begin{equation}\label{eq.complete1}
		\hat q_1^{\,t} (x,a)\;=\;\frac{\tilde q_1^{\,t}(x,a)}{Z_{1,\ell}^t(	\beta_1^t, \mu_1^{+,t}, \mu_1^{-,t}) }\,{\rm e}^{-B_{1,t}^{\beta_1^t ,\mu_1^{+,t}, \mu_1^{-,t}}(x,a,x')}
		\end{equation}
		\begin{equation}\label{eq.complete2}
		\hat q_2^{\,t} (y,b)\;=\;\frac{\tilde q_2^{\,t}(x,a)}{Z_{2,\ell}^t(	\beta_2^t, \mu_2^{+,t}, \mu_2^{-,t}) }\,{\rm e}^{-B_{2,t}^{\beta_2^t ,\mu_2^{+,t}, \mu_2^{-,t}}(y,b,y')}
		\end{equation}
	\end{subequations}
	where $B_{1,t}^{\beta_1,\mu_1^+,\mu_1^-}(x,a,x') $ and $B_{2,t}^{\beta_2,\mu_2^+,\mu_2^-}(y,b,y') $ are given by
	\[
	\begin{array}{rcl}
	B_{1,t}^{\beta_1,\mu_1^+,\mu_1^-}(x,a,x') &\DefinedAs& \beta_1(x') \,-\, \beta_1(x) \,+\, \eta \phi_1^{t-1}
	\\[0.2cm]
	&& \,+\, (1-\epsilon_1^{k_1}(x,a))\mu_1^{+}(x,a,x') \,-\, (1+\epsilon_1^{k_1}(x,a))\mu_1^{-}(x,a,x')
	\\[0.2cm]
	&& \displaystyle \,+\,\sum_{x''\,\in\,X_{\ell+1}}  \bar P_1^{k_1}(x''\,\vert\,x,a) (\mu_1^{-}(x,a,x'')-\mu_1^{+}(x,a,x''))
	\end{array}
	\]
	\[
	\begin{array}{rcl}
	B_{2,t}^{\beta_2,\mu_2^+,\mu_2^-}(y,b,y') &\DefinedAs& \beta_2(y') \,-\, \beta_2(y) \,+\, \eta \phi_2^{t-1}
	\\[0.2cm]
	&& \,+\, (1-\epsilon_2^{k_2}(y,b))\mu_2^{+}(y,b,y') \,-\, (1+\epsilon_2^{k_2}(y,b))\mu_2^{-}(y,b,y')
	\\[0.2cm]
	&& \displaystyle \,+\, \sum_{y''\,\in\,Y_{\ell+1}}  \bar P_2^{k_2}(y''\,\vert\,y,b) (\mu_2^{-}(y,b,y'')-\mu_2^{+}(y,b,y''))
	\end{array}
	\]
	and $Z_{1,\ell}^t (\beta_1,\mu_1^+,\mu_1^-)$ and $Z_{2,\ell}^t (\beta_2,\mu_2^+,\mu_2^-)$ are given by
	\[
	Z_{1,\ell}^t (\beta_1,\mu_1^+,\mu_1^-)\;=\; \sum_{x\in X_\ell}\sum_{a\in A} \sum_{x'\in X_{\ell+1}}\tilde q_1^{\,t}(x,a)\,{\rm e}^{-B_{1,t}^{\beta_1,\mu_1^+,\mu_1^-}(x,a,x')}
	\]
	\[
	Z_{2,\ell}^t (\beta_2,\mu_2^+,\mu_2^-)\;=\; \sum_{y\in Y_\ell}\sum_{b\in B} \sum_{y'\in Y_{\ell+1}}\tilde q_2^{\,t}(y,b)\,{\rm e}^{-B_{2,t}^{\beta_2,\mu_2^+,\mu_2^-}(y,b,y')}
	\]
	and the dual variables $\beta_1^t(x)$, $\mu_1^{+,t}(x,a,x')$, $\mu_1^{-,t}(x,a,x')$ and $\beta_2^t(y)$, $\mu_2^{+,t}(y,b,y')$, $\mu_2^{-,t}(y,b,y')$ are the solutions to
	\[
	\beta_1^t, \mu_1^{+,t}, \mu_1^{-,t}
	\;=\;
	\argmin_{\beta_1,\, \mu_1^+,\,\mu_1^-\,\geq \,0}\; \sum_{\ell\,=\,0}^{L-1} \ln Z_{1,\ell}^t (\beta_1,\mu_1^+,\mu_1^-)
	\]
	\[
	\beta_2^t, \mu_2^{+,t}, \mu_2^{-,t}
	\;=\;
	\argmin_{\beta_2,\, \mu_2^+,\,\mu_2^-\,\geq \,0}\; \sum_{\ell\,=\,0}^{L-1} \ln Z_{2,\ell}^t (\beta_2,\mu_2^+,\mu_2^-).
	\]
	
\end{lemma}
\begin{proof}
	In~\eqref{eq.primal12c}, we have two standard mirror descent problems. Since two problems enjoy the same structure, we only prove an efficient solution to the first problem~\eqref{eq.primal1c}.
	
	By the online mirror descent optimization~\citep{zimin2013online}, Problem~\eqref{eq.primal1c} is equivalent to
	\begin{equation}\label{eq.primal12steps1}
	\bar q_1^{\,t} 
	\;=\;
	\displaystyle \argmin_{q_1}  \;
	\eta \, \langle{q_1},{\phi_1^{t-1}}\rangle 
	\, +\, 
	\displaystyle D\big(q_1\,\vert\,  \tilde{q}_1^{\,t-1}\big)
	\;
	\text{ and }
	\;
	\hat q_1^{\,t}
	\;=\;
	\displaystyle \argmin_{q_1\,\in\,\Delta(k_1^t)}  \;
	\displaystyle D\big(q_1\,\vert\,  \bar{q}_1^{\,t}\big)
	\end{equation}
	where $\bar q_1^{\,t}$ is a solution to an unconstrained problem and $\hat q_1^{\,t}$ simply takes the projection of $\bar q_1^{\,t}$ to the domain $\Delta(k_1^t)$ in the unnormalized Kullback-Leibler divergence. 
	
	It is straightforward to compute a closed-form solution for the unconstrained problem,
	\begin{equation}\label{eq.unconstrained}
	\bar q_1^{\,t} (x,a) \;=\; \tilde q_1^{\,t} (x,a) \,{\rm e}^{- \eta\phi_1^{t-1}(x,a)},\,\text{ for all } (x,a)\in X\times A.
	\end{equation}
	
	To compute the projection of $\bar q_1^{\,t}$, we recall that the domain set $\Delta(k_1^t)$ explicitly takes the following linear constraints on $q_1$: $X\times A\to [0,1]$,
	\begin{equation}\label{eq.deltaset}
	\Delta(k_1^t) \;\DefinedAs\; \{ q_1: X\times A\to [0,1] \,\vert\, q_1 \text{satisfies the following (i), (ii), (iii), (iv)} \}
	\end{equation}
	\begin{itemize}
		\item [(i)] $q_1(x, a) = \sum_{x'\,\in\,X_{\ell+1}} q_1(x,a,x')$ for $(x,a)\in X_\ell\times A$ and $\ell \in\{0,1,\cdots,L-1\}$;
		
		\item [(ii)] $\sum_{x\in X_\ell}\sum_{a\in A}\sum_{x'\in X_{\ell+1}} q_1(x, a, x') = 1$ for $\ell \in\{0,1,\cdots,L-1\}$; 
		
		\item [(iii)] $\sum_{x\in X_{\ell-1}}\sum_{a\in A} q_1(x, a, x') =\sum_{a\in A}\sum_{x''\in X_{\ell+1}} q_1(x', a, x'')$ for $x'\in X_\ell$ and $\ell \in\{1,\cdots,L-1\}$;
		
		\item [(iv)] 
		$q_1(x,a,x')-\bar P_1^{k_1}(x'\,\vert\,x,a) \sum_{x''\,\in\,X_{\ell+1}}q_1(x,a,x'') \leq \epsilon(x,a,x')$, 
		
		$\bar P_1^{k_1}(x'\,\vert\,x,a) \sum_{x''\,\in\,X_{\ell+1}}q_1(x,a,x'')-q_1(x,a,x') \leq \epsilon(x,a,x')$, 
		
		and $\sum_{x'\,\in\,X_{\ell+1}} \epsilon(x,a,x') \leq \epsilon_1^{k_1}(x,a) \sum_{x'\,\in\,X_{\ell+1}}q_1(x,a,x')$ for $(x,a,x')\in X_\ell \times A\times X_{\ell+1}$ and $\ell \in\{0,1,\cdots,L-1\}$.
	\end{itemize}
	
	\noindent where (ii) and (iii) follow the occupancy measure's property and (iv) displays the confidence set condition for $q_1\in \Delta(k_1^t)$,
	\[
	\Bigg\Vert \frac{q_1(x,a,\cdot)}{\sum_{x''\,\in\,X_{\ell+1}}q_1(x,a,x'') }- \bar P_1^{k_1}(\cdot\,\vert\,x,a) \Bigg\Vert_1 \; \leq \; \epsilon_1^{k_1}(x,a), \text{ for all }(x,a) \in X\times A
	\]
	and we also introduce $\epsilon$: $X\times A\times X\to [0,\infty)$ additionally. Therefore, the projection problem is a convex optimization with the linear constraints. By the method of Lagrange multipliers, we have the following Lagrangian $\calL(q_1,\epsilon; \alpha, \lambda,\beta,\mu^+,\mu^-,\mu)$,
	\[
	\begin{array}{rcl}
	&& \;\calL(q_1,\epsilon; \alpha, \lambda,\beta,\mu^+,\mu^-,\mu) 
	\\[0.2cm]
	&& \displaystyle \;=\;  D\big(q_1\,\vert\,  \bar{q}_1^{\,t}\big) \,+\, \sum_{\ell\,=\,0}^{L-1} \sum_{x\,\in\,X_\ell}\sum_{a\,\in\, A}\alpha(x,a)\rbr{q_1(x, a) - \sum_{x'\,\in\,X_{\ell+1}} q_1(x,a,x')}
	\\[0.2cm]
	&& \displaystyle \,+\, \sum_{\ell\,=\,0}^{L-1}\lambda_\ell\rbr{\sum_{x\in X_\ell}\sum_{a\in A}\sum_{x'\in X_{\ell+1}} q_1(x, a, x') - 1}
	\\[0.2cm]
	&& \displaystyle \,+\, \sum_{\ell\,=\,1}^{L-1}\sum_{x'\,\in\,X_\ell}\beta(x') \rbr{\sum_{x\in X_{\ell-1}}\sum_{a\in A} q_1(x, a, x') -\sum_{a\in A}\sum_{x''\in X_{\ell+1}} q_1(x', a, x'')}
	\\[0.2cm]
	&& \!\!\!\! \!\!\!\! \!\!\!\! \!\!\!\! 
	\displaystyle \,+\, \sum_{\ell\,=\,0}^{L-1}\sum_{x\,\in\,X_\ell}\sum_{a\,\in\, A}\sum_{x'\,\in\,X_{\ell+1}} \mu^+(x,a,x')\rbr{q_1(x,a,x')-\bar P_1^{k_1}(x'\,\vert\,x,a) \sum_{x''\,\in\,X_{\ell+1}}q_1(x,a,x'') - \epsilon(x,a,x')}
	\\[0.2cm]
	&& \!\!\!\! \!\!\!\! \!\!\!\! \!\!\!\! 
	\displaystyle \,+\, \sum_{\ell\,=\,0}^{L-1}\sum_{x\,\in\,X_\ell}\sum_{a\,\in\, A}\sum_{x'\,\in\,X_{\ell+1}} \mu^-(x,a,x')\rbr{\bar P_1^{k_1}(x'\,\vert\,x,a) \sum_{x''\,\in\,X_{\ell+1}}q_1(x,a,x'') -q_1(x,a,x')- \epsilon(x,a,x')}
	\\[0.2cm]
	&& \displaystyle \,+\, \sum_{\ell\,=\,0}^{L-1}\sum_{x\,\in\,X_\ell}\sum_{a\,\in\,A}\mu(x,a) \rbr{\sum_{x'\,\in\,X_{\ell+1}} \epsilon(x,a,x') - \epsilon_1^{k_1}(x,a) \sum_{x'\,\in\,X_{\ell+1}}q_1(x,a,x')}
	\end{array}
	\]
	where $\alpha (x,a)$, $\lambda_\ell$, $\beta(x)$, $\mu^+(x,a,x')\geq0$, $\mu^-(x,a,x')\geq0$, and $\mu(x,a,x')\geq0$ for $(x,a,x')\in X_\ell\times A\times X_{\ell+1}$ are Lagrange multipliers associated to the linear constraints. 
	
	By the Lagrange duality theory, the strong duality holds. To find the optimal solution to the projection problem in~\eqref{eq.primal12steps1}, it suffices to check the first-order stationary conditions. We first take the derivative over $\epsilon(x,a,x')$ for $(x,a,x')\in X_\ell \times A\times X_{\ell+1}$,
	\[
	\frac{\partial \calL}{\partial \epsilon(x,a,x') }
	\;=\; 
	-\, \mu^+(x,a,x') \,-\, \mu^-(x,a,x') \,+\, \mu(x,a)
	\]
	which is zero if we take $\mu(x,a)=\mu^+(x,a,x')+\mu^-(x,a,x')$. Using this stationary condition, we simplify the Lagrangian $\calL(q_1,\epsilon; \alpha, \lambda,\beta,\mu^+,\mu^-,\mu)$ by eliminating $\mu$ and $\epsilon$ into,
	\[
	\begin{array}{rcl}
	&& \;
	\calL(q_1; \alpha, \lambda,\beta,\mu^+,\mu^-) 
	\\[0.2cm]
	&& \;=\; \displaystyle D\big(q_1\,\vert\,  \bar{q}_1^{\,t}\big) \,+\, \sum_{\ell\,=\,0}^{L-1} \sum_{x\,\in\,X_\ell}\sum_{a\,\in\, A}\alpha(x,a)\rbr{q_1(x, a) - \sum_{x'\,\in\,X_{\ell+1}} q_1(x,a,x')}
	\\[0.2cm]
	&& \displaystyle \,+\, \sum_{\ell\,=\,0}^{L-1} \lambda_\ell\rbr{\sum_{x\in X_\ell}\sum_{a\in A}\sum_{x'\in X_{\ell+1}} q_1(x, a, x') - 1}
	\\[0.2cm]
	&& \displaystyle \,+\, \sum_{\ell\,=\,1}^{L-1}\sum_{x'\,\in\,X_\ell}\beta(x') \rbr{\sum_{x\in X_{\ell-1}}\sum_{a\in A} q_1(x, a, x') -\sum_{a\in A}\sum_{x''\in X_{\ell+1}} q_1(x', a, x'')}
	\\[0.2cm]
	&& \!\!\!\! \!\!\!\! \!\!\!\! \!\!\!\! 
	\displaystyle \,+\, \sum_{\ell\,=\,0}^{L-1}\sum_{x\,\in\,X_\ell}\sum_{a\,\in\, A}\sum_{x'\,\in\,X_{\ell+1}} \mu^+(x,a,x')\rbr{ (1-\epsilon_1^{k_1}(x,a))q_1(x,a,x')-\bar P_1^{k_1}(x'\,\vert\,x,a) \sum_{x''\,\in\,X_{\ell+1}}q_1(x,a,x'') }
	\\[0.2cm]
	&& \!\!\!\! \!\!\!\! \!\!\!\! \!\!\!\! 
	\displaystyle \,+\, \sum_{\ell\,=\,0}^{L-1}\sum_{x\,\in\,X_\ell}\sum_{a\,\in\, A}\sum_{x'\,\in\,X_{\ell+1}} \mu^-(x,a,x')\rbr{\bar P_1^{k_1}(x'\,\vert\,x,a) \sum_{x''\,\in\,X_{\ell+1}}q_1(x,a,x'') -(1+\epsilon_1^{k_1}(x,a))q_1(x,a,x')}.
	\end{array}
	\]
	
	For the notational simplicity, we take $\beta(x_0)=\beta(x_L) = 0$. We next check the first-order stationary conditions of $\calL(q_1; \alpha, \lambda,\beta,\mu^+,\mu^-)$ and solve them for the stationary point. 
	We first take the derivative over $q_1(x,a,x')$ and $q_1(x,a)$ for $(x,a,x')\in X_\ell \times A\times X_{\ell+1}$, respectively,
	\[
	\begin{array}{rcl}
	\displaystyle \frac{\partial \calL}{\partial q_1(x,a,x') } &=& - \, \alpha(x,a) \,+\, \lambda_\ell \,+\, \beta(x')-\beta(x)
	\\[0.2cm]
	&& \,+\, (1-\epsilon_1^{k_1}(x,a))\mu^+(x,a,x') \,-\, (1+\epsilon_1^{k_1}(x,a))\mu^-(x,a,x')
	\\[0.2cm]
	&& \displaystyle \,+\, \sum_{x''\,\in\,X_{\ell+1}}  \bar P_1^{k_1}(x''\,\vert\,x,a) (\mu^-(x,a,x'')-\mu^+(x,a,x''))
	\end{array}
	\]
	\[
	\frac{\partial \calL}{\partial q_1(x,a) }\;=\; \ln q_1(x,a) \,-\, \ln \bar q_1^{\,t}(x,a) \,+\, \alpha(x,a).
	\] 
	By setting the second derivative above to be zero, we have $\alpha(x,a) = -\ln q_1(x,a)+\ln \bar q_1^{\,t}(x,a)$. Then, substituting it into the first zero-derivative by eliminating $\alpha(x,a)$ yields,
	\[
	\begin{array}{rcl}
	\displaystyle \ln q_1(x,a) &=& \ln \bar q_1^{\,t}(x,a) \,+\,\eta \phi_1^{t-1}(x,a) \,-\, \lambda_\ell-B^{t}(x,a,x')
	\end{array}
	\]
	\[
	\begin{array}{rcl}
	B^{t}(x,a,x') &=& \beta(x')-\beta(x) \,+\, \eta \phi_1^{t-1}(x,a)
	\\[0.2cm]
	&& \,+\, (1-\epsilon_1^{k_1}(x,a))\mu^+(x,a,x') \,-\, (1+\epsilon_1^{k_1}(x,a))\mu^-(x,a,x')
	\\[0.2cm]
	&& \displaystyle \,+\, \sum_{x''\,\in\,X_{\ell+1}}  \bar P_1^{k_1}(x''\,\vert\,x,a) (\mu^-(x,a,x'')-\mu^+(x,a,x'')).
	\end{array}
	\]
	The solution $q_1^\star(x,a)$ leads to an explicit formula for $\hat q_1^{\,t}$,
	\begin{equation}\label{eq.incomplete}
	\hat q_1^{\,t} (x,a) \;=\;q_1^\star(x,a) \;=\; \bar q_1^{\,t}(x,a) \,{\rm e}^{\eta \phi_1^{t-1}(x,a)- \lambda_\ell-B^{t}(x,a,x')}\;=\;\tilde q_1^{\,t}(x,a) \,{\rm e}^{- \lambda_\ell-B^{t}(x,a,x')}
	\end{equation}
	where the last equality is due to~\eqref{eq.unconstrained} and $x\neq x_L$. We note that it is not unique to determine $\alpha(x,a)$ since it takes the form $\alpha^\star(x,a)=-\eta \phi_1^{t-1}(x,a)+\lambda_\ell+B^{t}(x,a,x')$ for some $x'$. It remains to determine the optimal $\beta$, $\mu^+$, and $\mu^-$.
	
	Bofore showing the optimal $\beta$, $\mu^+$, and $\mu^-$, we take another derivative over $\lambda_\ell$ at $q_1=\hat q_1^t$ and set it to be zero,
	\[
	\sum_{x\in X_\ell}\sum_{a\in A}\sum_{x'\in X_{\ell+1}} \hat q_1^{\,t}(x, a, x') \;=\;1
	\]
	or, equivalently,
	\[
	{\rm e}^{\lambda_\ell} \;=\; \sum_{x\in X_\ell}\sum_{a\in A}\sum_{x'\in X_{\ell+1}}  \tilde q_1^{\,t}(x,a)\,{\rm e}^{-B^{t}(x,a,x')}\;\DefinedAs\; Z_\ell^t
	\]
	which shows that $\lambda_\ell^\star = \ln Z_\ell^t$. It also leads to $\alpha^\star(x,a)=-\eta \phi_1^{t-1}(x,a)+\lambda_\ell^\star+B^{t}(x,a,x')$.
	
	We note that 
	\[
	\begin{array}{rcl}
	&&\!\!\!\! \!\!\!\! \!\!
	\calL(q_1; \alpha, \lambda,\beta,\mu^+,\mu^-)  
	\\[0.2cm]
	&=& \displaystyle
	D\big(q_1\,\vert\,  \bar{q}_1^{\,t}\big)+ \sum_{\ell\,=\,0}^{L-1}\sum_{x\,\in\,X_\ell}\sum_{a\,\in\, A}\sum_{x'\,\in\,X_{\ell+1}} \bigg(\frac{\partial \mathcal{L}}{\partial q_1(x,a,x') } + \alpha(x,a) \bigg)q_1(x,a,x')-\sum_{\ell\,=\,0}^{L-1} \lambda_\ell
	\\[0.2cm]
	&=&  \displaystyle D\big(q_1\,\vert\,  \bar{q}_1^{\,t}\big)+
	\sum_{\ell\,=\,0}^{L-1}\sum_{x\,\in\,X_\ell}\sum_{a\,\in\, A}\sum_{x'\,\in\,X_{\ell+1}} \frac{\partial \mathcal{L}}{\partial q_1(x,a,x') }q_1(x,a,x')
	\\[0.2cm]
	&& \displaystyle
	+\sum_{\ell\,=\,0}^{L-1}\sum_{x\,\in\,X_\ell}\sum_{a\,\in\, A} \bigg(\frac{\partial \mathcal{L}}{\partial q_1(x,a) }- \ln q_1(x,a)+\ln \bar q_1^{\,t}(x,a) \bigg) q_1(x,a)
	-\sum_{\ell\,=\,0}^{L-1} \lambda_\ell
	\\[0.2cm]
	&=&  \displaystyle 
	\sum_{\ell\,=\,0}^{L-1}\sum_{x\,\in\,X_\ell}\sum_{a\,\in\, A}\sum_{x'\,\in\,X_{\ell+1}} \frac{\partial \mathcal{L}}{\partial q_1(x,a,x') }q_1(x,a,x')
	\\[0.2cm]
	&& \displaystyle
	+\sum_{\ell\,=\,0}^{L-1}\sum_{x\,\in\,X_\ell}\sum_{a\,\in\, A} \bigg(\bigg(\frac{\partial \mathcal{L}}{\partial q_1(x,a) }-1\bigg)  q_1(x,a)+\ln \bar q_1^{\,t}(x,a) \bigg) 
	-\sum_{\ell\,=\,0}^{L-1} \lambda_\ell.
	\end{array}
	\]
	We now collect all previously determined optimal dual variables and apply the strong duality,
	\[
	\begin{array}{rcl}
	\beta^\star, \mu^{+,\star}, \mu^{-,\star}
	&=&\displaystyle
	\argmax_{\beta,\, \mu^+,\,\mu^-\,\geq \,0} \maximize_{\alpha,\,\lambda}\,\minimize_{q_1}\; \calL(q_1; \alpha, \lambda,\beta,\mu^+,\mu^-)
	\\[0.2cm]
	&=&\displaystyle
	\argmax_{\beta,\, \mu^+,\,\mu^-\,\geq \,0}\; \calL(q_1^\star; \alpha^\star, \lambda^\star,\beta,\mu^+,\mu^-)
	\\[0.2cm]
	&=&\displaystyle
	\argmax_{\beta,\, \mu^+,\,\mu^-\,\geq \,0}\; -L+\sum_{\ell\,=\,0}^{L-1}\sum_{x\,\in\,X_\ell}\sum_{a\,\in\, A} \ln \bar q_1^{\,t}(x,a) -\sum_{\ell\,=\,0}^{L-1} \lambda_\ell^\star
	\\[0.2cm]
	&=&\displaystyle
	\argmin_{\beta,\, \mu^+,\,\mu^-\,\geq \,0}\; \sum_{\ell\,=\,0}^{L-1} \ln Z_\ell^t
	\end{array}
	\]
	where the third equality is due to: $ \frac{\partial \calL}{\partial q_1(x,a,x') } \vert_{q_1^\star(x,a,x')}=0 $ and $\frac{\partial \calL}{\partial q_1(x,a) } \vert_{q_1^\star(x,a)}= 0$, and we ignore all constants that are independent of $\beta$, $\mu^+$, and $\mu^-$ for the last equality; we note that this minimization problem is a convex optimization problem over the nonnegative orthant. Hence, we have proved the update~\eqref{eq.complete1} as an efficient update~\eqref{eq.incomplete}.
	Similarly, we have an efficient update~\eqref{eq.complete2} for the second problem~\eqref{eq.primal2c} and the proof is complete.
	
\end{proof}

\section{Proof of Lemma~\ref{lem.empiricalP}}
\label{ap.empiricalP}

For any $q_1\in \Delta(P_1)$ and $q_2\in \Delta(P_2)$, we estimate
\[
\hat P_1(\cdot\,\vert\,x,a)\;=\;\frac{q_1(x,a, \cdot)}{\sum_{x'\,\in\,X_{\ell+1}}q_1(x,a, x')}
\;\text{ and }\;
\hat P_2(\cdot\,\vert\,y,b)\;=\;\frac{q_2(y,b, \cdot)}{\sum_{y'\,\in\,Y_{\ell+1}}q_2(y,b, y')}.
\]
Consequently,
\[
\begin{array}{rcl}
&& \!\!\!\! \!\!\!\! \!\! 
\displaystyle
\norm{ \frac{q_1(x,a, \cdot)}{\sum_{x'\,\in\,X_{\ell+1}}q_1(x,a, x')} - \bar P_1^{k_1}(\cdot\,\vert\,x,a) }_1 
\\[0.2cm]
&\leq &
\displaystyle
\norm{ \frac{q_1(x,a, \cdot)}{\sum_{x'\,\in\,X_{\ell+1}}q_1(x,a, x')} -\hat P_1(\cdot\,\vert\,x,a) }_1
\,+\,
\norm{ \hat P_1(\cdot\,\vert\,x,a)-\bar P_1^{k_1}(\cdot\,\vert\,x,a) }_1
\\[0.2cm]
&=&\displaystyle 
\norm{ \hat P_1(\cdot\,\vert\,x,a)-\bar P_1^{k_1}(\cdot\,\vert\,x,a) }_1
\end{array}
\]
which implies that $q_1\in \Delta(P_1^{k_1})$. Similarly, we have $q_2\in \Delta(P_2^{k_2})$. Therefore, $\Delta(P_1)\subset \Delta(\calP_1^{k_1})$ and $\Delta(P_2)\in\Delta(\calP_2^{k_2})$. The probability argument follows Lemma~1~\citep{neu2012adversarial} or its original version, Lemma~17~\citep{jaksch2010near}: with probability $1-\delta$ it holds that
\[
\Vert \hat P_1(\cdot\,\vert\,x,a) - \bar P_1^{k_1}(\cdot\,\vert\,x,a) \Vert_1 \leq \epsilon_1^{k_1}
\;\text{ and }\;
\Vert \hat P_2(\cdot\,\vert\,y,b) - \bar P_2^{k_2}(\cdot\,\vert\,y,b) \Vert_1 \leq \epsilon_2^{k_2}
\]
for all $(x,a)\in X\times A$, $(y,b)\in Y\times B$, and all epochs $k_1$ and $k_2$.

\section{Proof of Lemma~\ref{lem.qqdifference}}
\label{ap.qqdifference}

We recall the occupancy measures induced by the empirical transitions $\hat P_1$ and $\hat P_2$,
\[
\hat{q}_1^{\,t} (x,a,x') \;=\; \hat{d}_1^{\,t} (x) \pi^t(a\,\vert\,x) \hat P_1^{k_1}(x'\,\vert\,x,a)
\;\text{ and }\;
\hat q_2^{\,t} (y,b,y') \;=\; \hat d_2^{\,t} (y) \mu^t(b\,\vert\,y) \hat P_2^{k_2}(y'\,\vert\,y,b)
\] 
\[
\hat q_1^{\,t} (x,a) \;=\; \sum_{x'\,\in\,X_{\ell+1}} \hat q_1^{\,t} (x,a,x')
\;\text{ and }\;
\hat q_2^{\,t} (y,b) \;=\;  \sum_{y'\,\in\,Y_{\ell+1}} \hat q_2^{\,t} (y,b,y')
\] 
where $\hat d_1^{\,t} (x)$ and $\hat d_2^{\,t} (x)$ are the stationary state visitation probabilities, and  the occupancy measures induced by the true transitions $ P_1$ and $ P_2$,
\[
q_1^t (x,a,x') \;=\;  d_1^t(x) \pi^t(a\,\vert\,x)  P_1(x'\,\vert\,x,a)
\;\text{ and }\;
q_2^t (y,b,y') \;=\;  d_2^t(y) \mu^t(b\,\vert\,y)  P_2(y'\,\vert\,y,b)
\] 
\[
q_1^t (x,a) \;=\; \sum_{x'\,\in\,X_{\ell+1}}  q_1^t (x,a,x')
\;\text{ and }\;
q_2^t (y,b) \;=\;  \sum_{y'\,\in\,Y_{\ell+1}}  q_2^t (y,b,y')
\] 
where $d_1^t(x)$ and $d_2^t(x)$ are the stationary state visitation probabilities. We denote by $\ell$ the layer that $x$ or $y$ belongs to.

We first present a useful property on how the transition estimation errors affect the mismatch of occupancy measures.

\begin{lemma}\label{lem.qdifference}
	Let $\hat q_1^{\,t}$, $\hat q_2^{\,t}$, $q_1^t$, and $q_2^t$ be generated by Algorithm~\ref{UCB-MPD}. Then,
	\begin{subequations}
		\begin{equation}\label{eq.qdifference1}
		\begin{array}{rcl}
		\displaystyle
		\norm{\hat q_1^{\,t} - q_1^t}_1
		&\leq&\displaystyle\sum_{j\,=\,0}^{L-1} \sum_{\ell\,=\,0}^{j}\sum_{x\,\in\,X_\ell}\sum_{a\,\in\,A} \pi^t(a\,\vert\,x) \hat d_1^{\,t} (x) \norm{ \hat P_1^{k_1}(\cdot\,\vert\,x,a)-  P_1(\cdot\,\vert\,x,a) }_1
		\end{array}
		\end{equation}
		\begin{equation}\label{eq.qdifference2}
		\begin{array}{rcl}
		\displaystyle
		\norm{\hat q_2^{\,t} - q_2^t}_1
		&\leq&\displaystyle\sum_{j\,=\,0}^{L-1} \sum_{\ell\,=\,0}^{j}\sum_{y\,\in\,Y_\ell}\sum_{b\,\in\,B} \mu^t(b\,\vert\,y) \hat d_2^{\,t} (y) \norm{ \hat P_2^{k_2}(\cdot\,\vert\,y,b)-  P_2(\cdot\,\vert\,y,b) }_1.
		\end{array}
		\end{equation}
	\end{subequations}
\end{lemma}

\begin{proof}
	Since two players have the independent transitions, it suffices to just prove one of two players. We next prove~\eqref{eq.qdifference1} for the min-player. 
	By the definitions, we can bound $\norm{\hat q_1^{\,t} - q_1^t}_1$ by
	\begin{equation}\label{eq.qdifference_xa}
	\begin{array}{rcl}
	\norm{\hat q_1^{\,t} - q_1^t}_1	&=&\displaystyle \sum_{\ell\,=\,0}^{L-1}\sum_{x\,\in\,X_\ell}\sum_{a\,\in\,A}\abr{ \sum_{x'\,\in\,X_{\ell+1}}  \hat q_1^{\,t} (x,a,x')- \sum_{x'\,\in\,X_{\ell+1}}  q_1^t (x,a,x')}
	\\[0.2cm]
	&\leq&\displaystyle \sum_{\ell\,=\,0}^{L-1}\sum_{x\,\in\,X_\ell}\sum_{a\,\in\,A}\norm{ \hat q_1^{\,t} (x,a,\cdot)- q_1^t (x,a,\cdot)}_1
	\\[0.2cm]
	&=&\displaystyle \sum_{\ell\,=\,0}^{L-1}\sum_{x\,\in\,X_\ell}\sum_{a\,\in\,A} \pi^t(a\,\vert\,x)  \norm{ \hat d_1^{\,t} (x)\hat P_1^{k_1}(\cdot\,\vert\,x,a)- d_1^t(x) P_1(\cdot\,\vert\,x,a) }_1
	\end{array}
	\end{equation}
	where we apply the triangle inequality to obtain the inequality. We add and subtract $\hat d_1^{\,t} (x)P_1(\cdot\,\vert\,x,a)$ into the norm $ \Vert{ \hat d_1^{\,t} (x)\hat P_1^{k_1}(\cdot\,\vert\,x,a)- d_1^t(x) P_1(\cdot\,\vert\,x,a) }\Vert_1$, and apply the triangle inequality again,
	\[
	\begin{array}{rcl}
	&&  \!\!\!\! \!\!\!\! \!\!
	\displaystyle
	\norm{ \hat d_1^{\,t} (x)\hat P_1^{k_1}(\cdot\,\vert\,x,a)  -  d_1^t(x) P_1(\cdot\,\vert\,x,a) }_1
	\\[0.2cm]
	&\leq&\displaystyle \norm{ \hat d_1^{\,t} (x)\hat P_1^{k_1}(\cdot\,\vert\,x,a)- \hat d_1^{\,t} (x) P_1(\cdot\,\vert\,x,a) }_1
	\,+\, \norm{ \hat d_1^{\,t} (x) P_1(\cdot\,\vert\,x,a)- d_1^t(x) P_1(\cdot\,\vert\,x,a) }_1.
	\end{array}
	\]
	Therefore,
	\begin{equation}\label{eq.qdifference}
	\begin{array}{rcl}
	&& \!\!\!\! \!\!\!\! \!\!
	\displaystyle
	\sum_{\ell\,=\,0}^{L-1}\sum_{x\,\in\,X_\ell}\sum_{a\,\in\,A}\norm{ \hat q_1^{\,t} (x,a,\cdot)- q_1^t (x,a,\cdot)}_1
	\\[0.2cm]
	&\leq&\displaystyle \sum_{\ell\,=\,0}^{L-1}\sum_{x\,\in\,X_\ell}\sum_{a\,\in\,A} \pi^t(a\,\vert\,x) \hat d_1^{\,t} (x) \norm{ \hat P_1^{k_1}(\cdot\,\vert\,x,a)-  P_1(\cdot\,\vert\,x,a) }_1
	\\[0.2cm]
	&&\displaystyle+ \sum_{\ell\,=\,0}^{L-1}\sum_{x\,\in\,X_\ell}\sum_{a\,\in\,A} \pi^t(a\,\vert\,x)  \abr{\hat d_1^{\,t}(x) - d_1^t(x) }\norm{P_1(\cdot\,\vert\,x,a) }_1.
	\end{array}
	\end{equation}
	We can further simplify the upper bound in~\eqref{eq.qdifference}. Using $\norm{P_1(\cdot\,\vert\,x,a) }_1=1$ and $\sum_{a\,\in\,A} \pi^t(a\,\vert\,x) =1$, we have
	\[
	\begin{array}{rcl}
	\displaystyle\sum_{\ell\,=\,0}^{L-1}\sum_{x\,\in\,X_\ell}\sum_{a\,\in\,A} \pi^t(a\,\vert\,x)  \abr{\hat d_1^{\,t} (x) - d_1^t(x) }\norm{P_1(\cdot\,\vert\,x,a) }_1
	&=& \displaystyle \sum_{\ell\,=\,0}^{L-1}\sum_{x\,\in\,X_\ell}\abr{\hat d_1^{\,t}(x) - d_1^t(x) }.
	\end{array}
	\]
	By the definitions, $\hat d_1^{\,t}(x) = d_1^t(x) =1$ for $x\in X_0$, and $\hat d_1^{\,t}(x)=\sum_{x^\circ\,\in\,X_{\ell-1}}\sum_{a\,\in\, A}\hat q_1^{\,t} (x^\circ,a,x)$ and $d_1^t(x)=\sum_{x^\circ\,\in\,X_{\ell-1}}\sum_{a\,\in\, A} q_1^t (x^\circ,a,x)$ for $x\in X_\ell$. Thus, 
	\[
	\begin{array}{rcl}
	&& \!\!\!\! \!\!\!\! \!\!
	\displaystyle \sum_{\ell\,=\,0}^{L-1}\sum_{x\,\in\,X_\ell}\abr{\hat d_1^{\,t} (x) - d_1^t(x) }
	\\[0.2cm]
	&=& \displaystyle \sum_{\ell\,=\,1}^{L-1}\sum_{x\,\in\,X_\ell}\abr{\hat d_1^{\,t} (x) - d_1^t(x) }
	\\[0.2cm]
	&=& \displaystyle \sum_{\ell\,=\,1}^{L-1}\sum_{x\,\in\,X_\ell}\abr{\sum_{x^\circ\,\in\,X_{\ell-1}}\sum_{a\,\in\, A}\hat q_1^{\,t} (x^\circ,a,x)- \sum_{x^\circ\,\in\,X_{\ell-1}}\sum_{a\,\in\, A} q_1^t (x^\circ,a,x)}
	\\[0.2cm]
	&\leq& \displaystyle \sum_{\ell\,=\,1}^{L-1}\sum_{x\,\in\,X_\ell} \sum_{a\,\in\, A} \sum_{x^\circ\,\in\,X_{\ell-1}}\abr{ \hat q_1^{\,t} (x^\circ,a,x) - q_1^t (x^\circ,a,x) } 
	\\[0.2cm]
	&=& \displaystyle \sum_{\ell\,=\,1}^{L-1}\sum_{a\,\in\, A} \sum_{x^\circ\,\in\,X_{\ell-1}}\norm{ \hat q_1^{\,t} (x^\circ,a,\cdot) - q_1^t (x^\circ,a,\cdot) }_1
	\\[0.2cm]
	&=& \displaystyle \sum_{\ell\,=\,0}^{L-2}\sum_{x\,\in\,X_{\ell}}\sum_{a\,\in\, A}  \norm{ \hat q_1^{\,t} (x,a,\cdot) - q_1^t (x,a,\cdot) }_1.
	\end{array}
	\]
	We now return back to~\eqref{eq.qdifference}, 
	\begin{equation}\label{eq.qdifference_simplified}
	\begin{array}{rcl}
	&& \!\!\!\! \!\!\!\! \!\!
	\displaystyle
	\sum_{\ell\,=\,0}^{L-1}\sum_{x\,\in\,X_\ell}\sum_{a\,\in\,A}\norm{ \hat q_1^{\,t} (x,a,\cdot)- q_1^t (x,a,\cdot)}_1
	\\[0.2cm]
	&\leq&\displaystyle \sum_{\ell\,=\,0}^{L-1}\sum_{x\,\in\,X_\ell}\sum_{a\,\in\,A} \pi^t(a\,\vert\,x) \hat d_1^{\,t}(x) \norm{ \hat P_1^{k_1}(\cdot\,\vert\,x,a)-  P_1(\cdot\,\vert\,x,a) }_1
	\\[0.2cm]
	&&\displaystyle+ \sum_{\ell\,=\,0}^{L-2}\sum_{x\,\in\,X_{\ell}}\sum_{a\,\in\, A}  \norm{ \hat q_1^{\,t} (x,a,\cdot) - q_1^t (x,a,\cdot) }_1
	\end{array}
	\end{equation}
	which is a recursive formula for $\sum_{\ell\,=\,0}^{j}\sum_{x\,\in\,X_{\ell}}\sum_{a\,\in\, A}  \norm{ \hat q_1^{\,t} (x,a,\cdot) - q_1^t (x,a,\cdot) }_1$ over $j\in\{ 0,1,\ldots,L-1 \}$. By the recursion,
	\[
	\begin{array}{rcl}
	&&\!\!\!\! \!\!\!\! \!\!
	\displaystyle
	\sum_{\ell\,=\,0}^{L-1}\sum_{x\,\in\,X_\ell}\sum_{a\,\in\,A}\norm{ \hat q_1^{\,t} (x,a,\cdot)- q_1^t (x,a,\cdot)}_1
	\\[0.2cm]
	&\leq&\displaystyle\sum_{j\,=\,0}^{L-1} \sum_{\ell\,=\,0}^{j}\sum_{x\,\in\,X_\ell}\sum_{a\,\in\,A} \pi^t(a\,\vert\,x) \hat d_1^{\,t} (x) \norm{ \hat P_1^{k_1}(\cdot\,\vert\,x,a)-  P_1(\cdot\,\vert\,x,a) }_1.
	\end{array}
	\]
	Finally, we complete the proof by using~\eqref{eq.qdifference_xa}.
\end{proof}

With Lemma~\ref{lem.qdifference} in place, we are ready to prove Lemma~\ref{lem.qqdifference}.

\begin{proof}[Proof of Lemma~\ref{lem.qqdifference}]
	The proof is based on Lemma~\ref{lem.qdifference}. By~\eqref{eq.qdifference1}, 
	\[
	\begin{array}{rcl}
	&& \!\!\!\! \!\!\!\! \!\!
	\displaystyle
	\norm{\hat q_1^{\,t} - q_1^t}_1
	\\[0.2cm]
	&\leq&\displaystyle\sum_{j\,=\,0}^{L-1} \sum_{\ell\,=\,0}^{j}\sum_{x\,\in\,X_\ell}\sum_{a\,\in\,A}  \big(\pi^t(a\,\vert\,x) \hat d_1^{\,t}(x) - \Ind{(x_\ell,a_\ell)=(x,a)} \big)\norm{ \hat P_1^{k_1}(\cdot\,\vert\,x,a)-  P_1(\cdot\,\vert\,x,a) }_1
	\\[0.2cm]
	&&\displaystyle+\sum_{j\,=\,0}^{L-1} \sum_{\ell\,=\,0}^{j}\sum_{x\,\in\,X_\ell}\sum_{a\,\in\,A} \Ind{(x_\ell,a_\ell)=(x,a)}\norm{ \hat P_1^{k_1}(\cdot\,\vert\,x,a)-  P_1(\cdot\,\vert\,x,a) }_1
	\end{array}
	\]
	where $\Ind{\cdot}$ is the indicator function that is $1$ with probability $\pi^t(a\,\vert\,x) \hat d_1^{\,t}(x)$ and $0$ otherwise.
	
	Let $\rho_1^t(x,a) \DefinedAs \Vert{ \hat P_1^{k_1}(\cdot\,\vert\,x,a)-  P_1(\cdot\,\vert\,x,a) }\Vert_1$. Clearly, $\rho_1^t(x,a) \leq 2$. Summing $\norm{\hat q_1^{\,t} - q_1^t}_1$ from $t=0$ to $t=T-1$ leads to,
	\begin{equation}\label{eq.qdifference_T}
	\begin{array}{rcl}
	&&  \!\!\!\! \!\!\!\! \!\!
	\displaystyle\sum_{t\,=\,0}^{T-1} 
	\norm{\hat q_1^{\,t} - q_1^t}_1
	\\[0.2cm]
	&\leq&\displaystyle \sum_{t\,=\,0}^{T-1} \sum_{j\,=\,0}^{L-1} \sum_{\ell\,=\,0}^{j}\sum_{x\,\in\,X_\ell}\sum_{a\,\in\,A}  \big(\pi^t(a\,\vert\,x) \hat d_1^{\,t}(x) - \Ind{(x_\ell,a_\ell)=(x,a)} \big)\rho_1^t(x,a) 
	\\[0.2cm]
	&&\displaystyle+\sum_{t\,=\,0}^{T-1} \sum_{j\,=\,0}^{L-1} \sum_{\ell\,=\,0}^{j}\sum_{x\,\in\,X_\ell}\sum_{a\,\in\,A} \Ind{(x_\ell,a_\ell)=(x,a)}\rho_1^t(x,a) 
	\end{array}
	\end{equation}
	where the layer $\ell$ depends on episode $t$ implicitly. We next apply the martingale concentration and Lemma~\ref{lem.empiricalP} to the right-hand side of~\eqref{eq.qdifference_T}.
	
	Let $\calF_1^{t}$ be an $\sigma$-algebra that is generated by the state-action sequence, reward/utility functions for the min-player up to episode $t$. By the definition of epoch $k_1\DefinedAs k_1^t$, $\rho_1^t(x,a)$ defines over $\calF_{1}^{t-1}$ only and thus,
	\[
	\mathbb{E} \sbr{  	
		\sum_{x\,\in\,X_\ell}\sum_{a\,\in\,A}  \big(\pi^t(a\,\vert\,x) \hat d_1^{\,t}(x) - \Ind{(x_\ell,a_\ell)=(x,a)} \big)\rho_1^t(x,a) \,\Bigg\vert\, \calF_{1}^{t-1}
	}
	\;=\; 0.
	\]
	Meanwhile, 
	it is easy to see that 
	\[
	\begin{array}{rcl}
	&& \displaystyle \abr{  	
		\sum_{x\,\in\,X_\ell}\sum_{a\,\in\,A}  \big(\pi^t(a\,\vert\,x) \hat d_1^{\,t}(x) - \Ind{(x_\ell,a_\ell)=(x,a)} \big)\rho_1^t(x,a) 
	}
	\\[0.2cm]
	&&\displaystyle
	\;\leq\;
	2	\sum_{x\,\in\,X_\ell}\sum_{a\,\in\,A}  \big(\pi^t(a\,\vert\,x) \hat d_1^{\,t}(x) + \Ind{(x_\ell,a_\ell)=(x,a)} \big)
	\end{array}
	\]
	which is bounded by $4$ since the summands are probability distributions. Hence, \\$\sum_{x\,\in\,X_\ell}\sum_{a\,\in\,A}  \big(\pi^t(a\,\vert\,x) \hat d_1^{\,t} (x) - \Ind{(x_\ell,a_\ell)=(x,a)} \big)\rho_1^t(x,a)$ is a martingale difference sequence that adapts to the filtration $\{\calF_{1}^{t}\}_{t\geq0}$. By the Azuma-Hoeffding inequality, with probability $1-\delta/L$ it holds that
	\begin{equation}\label{eq.qmartingale}
	\sum_{t\,=\,0}^{T-1}\sum_{x\,\in\,X_\ell}\sum_{a\,\in\,A}  \big(\pi^t(a\,\vert\,x) \hat d_1^{\,t}(x) - \Ind{(x_\ell,a_\ell)=(x,a)} \big)\rho_1^t(x,a) 
	\;\leq\;
	4 \sqrt{2T \log \frac{L}{\delta}}
	\end{equation}
	where $\delta\in (0,1)$. By the union bound,~\eqref{eq.qmartingale} holds with probability $1-\delta$ for all $\ell\in\{ 0,1,\ldots,L-1 \}$. Thus, with probability $1-\delta$, we have
	\begin{equation}\label{eq.qmartingale_TL}
	\!\!\!\! 
	\sum_{t\,=\,0}^{T-1} \sum_{j\,=\,0}^{L-1} \sum_{\ell\,=\,0}^{j}\sum_{x\,\in\,X_\ell}\sum_{a\,\in\,A}  \big(\pi^t(a\,\vert\,x) \hat d_1^{\,t}(x) - \Ind{(x_\ell,a_\ell)=(x,a)} \big)\rho_1^t(x,a) 
	\;\leq\;
	2L^2 \sqrt{2T \log \frac{L}{\delta}}.
	\end{equation}
	
	For the rest, we apply Lemma~\ref{lem.empiricalP}. By the definition of epoch $k_1\DefinedAs k_1^t$, we have
	$N_1^{k_1^t} (x,a) =\sum_{k\,=\,0}^{k_1^t-1} n_1^{k}(x,a)$.	An application of~Lemma~\ref{lem.series} yields
	\begin{equation}\label{eq.jaksch_lemma19}
	\sum_{k\,=\,1}^{k_1^t} \frac{n_1^{k}(x,a)}{\max(1, \sqrt{ N_1^{k} (x,a)  })} \;\leq\;2\sqrt{ N_1^{k_1^t} (x,a)   }.
	\end{equation}
	We note that 
	$\sum_{x\,\in\,X_\ell}\sum_{a\,\in\,A} \Ind{(x_\ell,a_\ell)=(x,a)}\rho_1^t(x,a) 
	= \Vert{ \hat P_1^{k_1}(\cdot\,\vert\,x_\ell,a_\ell)-  P_1(\cdot\,\vert\,x_\ell,a_\ell) }\Vert_1$. By Lemma~\ref{lem.empiricalP}, with probability $1-\delta$ it holds that
	\[
	\begin{array}{rcl}
	&& \!\!\!\! \!\!\!\! \!\!
	\displaystyle
	\sum_{t\,=\,0}^{T-1} \sum_{j\,=\,0}^{L-1} \sum_{\ell\,=\,0}^{j}\sum_{x\,\in\,X_\ell}\sum_{a\,\in\,A} \Ind{(x_\ell,a_\ell)=(x,a)}\rho_1^t(x,a) 
	\\[0.2cm]
	&=& \displaystyle \sum_{t\,=\,0}^{T-1} \sum_{j\,=\,0}^{L-1} \sum_{\ell\,=\,0}^{j}\Vert{ \hat P_1^{k_1}(\cdot\,\vert\,x_\ell,a_\ell)-  P_1(\cdot\,\vert\,x_\ell,a_\ell) }\Vert_1
	\\[0.2cm]
	&\leq& \displaystyle \sum_{t\,=\,0}^{T-1} \sum_{j\,=\,0}^{L-1} \sum_{\ell\,=\,0}^{j} \sqrt{ \frac{ 2|X_{\ell+1}| \log(T|A||X|/\delta)}{\max (1, N_1^{k_1}(x_\ell,a_\ell)) } }.
	\end{array}
	\]
	By the definition of $N_1^{k_1}\DefinedAs N_1^{k_1^t}$, using~\eqref{eq.jaksch_lemma19} it is convenient to have 
	\[
	\begin{array}{rcl}
	&& \!\!\!\! \!\!\!\! \!\!
	\displaystyle \sum_{t\,=\,0}^{T-1} \sum_{j\,=\,0}^{L-1} \sum_{\ell\,=\,0}^{j} \sqrt{ \frac{ 2|X_{\ell+1}| \log(T|A||X|/\delta)}{\max (1, N_1^{k_1}(x_\ell,a_\ell)) } } 
	\\[0.2cm]
	&\leq&
	\displaystyle \sum_{j\,=\,0}^{L-1} \sum_{\ell\,=\,0}^{j}\sum_{k\,=\,0}^{k_1^T} \sum_{x\,\in\,X_\ell}\sum_{a\,\in\,A} n_1^k(x,a) \sqrt{ \frac{ 2|X_{\ell+1}| \log(T|A||X|/\delta)}{\max (1, N_1^{k}(x,a)) } } 
	\\[0.2cm]
	&\leq& \displaystyle \sum_{j\,=\,0}^{L-1} \sum_{\ell\,=\,0}^{j}\sum_{x\,\in\,X_\ell}\sum_{a\,\in\,A}2 \sqrt{ 2 N_1^{k^T}(x,a) |X_{\ell+1}| \log\frac{T|A||X|}{\delta}}.
	\end{array}
	\]
	Furthermore, we can make the following simplifications. By the Jensen's inequality,
	\[
	\begin{array}{rcl}
	&& \!\!\!\! \!\!\!\! \!\!
	\displaystyle \sum_{x\,\in\,X_\ell}\sum_{a\,\in\,A}2 \sqrt{ 2 N_1^{k^T}(x,a) |X_{\ell+1}| \log\frac{T|A||X|}{\delta} }
	\\[0.2cm]
	&\leq&
	\displaystyle2 \sqrt{ 2\sum_{x\,\in\,X_\ell}\sum_{a\,\in\,A}N_1^{k^T}(x,a) |X_{\ell+1}||X_\ell||A| \log\frac{T|A||X|}{\delta} }.
	\end{array}
	\]
	We also note that $\sum_{x\,\in\,X_\ell}\sum_{a\,\in\,A}N_1^{k^T}(x,a)\leq T$ and $\sqrt{|X_{\ell+1}||X_\ell|}\leq\rbr{|X_{\ell+1}|+|X_\ell|}/2$. Thus,
	\[
	\begin{array}{rcl}
	&& \!\!\!\! \!\!\!\! \!\!
	\displaystyle \sum_{t\,=\,0}^{T-1} \sum_{j\,=\,0}^{L-1} \sum_{\ell\,=\,0}^{j} \sqrt{ \frac{ 2|X_{\ell+1}| \ln(T|A||X|/\delta)}{\max (1, N_1^{k_1}(x_\ell,a_\ell)) } } 
	\\[0.2cm]
	&\leq& \displaystyle \sum_{j\,=\,0}^{L-1} \sum_{\ell\,=\,0}^{j}2 \sqrt{ 2 T |X_{\ell+1}||X_\ell||A| \log\frac{T|A||X|}{\delta}}
	\\[0.2cm]
	&\leq& \displaystyle \sum_{j\,=\,0}^{L-1} \sum_{\ell\,=\,0}^{j}( |X_{\ell+1}|+|X_\ell|) \sqrt{ 2 T|A| \log\frac{T|A||X|}{\delta}}
	\\[0.2cm]
	&\leq& L|X| \sqrt{ 2 T|A|\log\frac{T|A||X|}{\delta}}.
	\end{array}
	\]
	Therefore, with probability $1-\delta$ it holds that
	\begin{equation}\label{eq.qdifferencerho}
	\begin{array}{rcl}
	\displaystyle
	\sum_{t\,=\,0}^{T-1} \sum_{j\,=\,0}^{L-1} \sum_{\ell\,=\,0}^{j}\sum_{x\,\in\,X_\ell}\sum_{a\,\in\,A} \Ind{(x_\ell,a_\ell)=(x,a)}\rho_1^t(x,a) 
	&\leq&L|X| \sqrt{ 2 T|A| \log\frac{T|A||X|}{\delta} }.
	\end{array}
	\end{equation}
	
	Finally, we take a union of~\eqref{eq.qmartingale_TL} and~\eqref{eq.qdifferencerho} and substitute it into~\eqref{eq.qdifference_T} to conclude the proof.
	
\end{proof}

\section{Proof of Lemma~\ref{lem.gap_pd}}
\label{ap.gap_pd}

We first present a basic property of the Kullback-Leibler divergence that generalizes similar properties in the literature~\citep{nemirovski2009robust,tseng2009accelerated,wei2020online} to the convex-concave minimax problems. For this purpose, we set some standard notations. Let $\calX\subset \mathbb{R}^d$ be a convex set with non-empty interior, $\calX^{\text{int}}\neq \emptyset$. Let $\phi$: $\calX\to\mathbb{R}$ be a function that is is continuously differentiable on $\calX^{\text{int}}$. Let $\Delta_x\subset \calX$ be a compact convex set containing the origin. Denote $\Delta_x^o=\Delta\cap\calX^{\text{ int}}$ and let $\Delta_x^o\neq \emptyset$. We define the Kullback-Leibler divergence, $D$: $\Delta_x\times\Delta_x^o\to\mathbb{R}$,
\[
D(x,x') \;\DefinedAs\; \phi(x) - \phi(x') - \langle \nabla\phi(x'), x-x'\rangle.
\]
An interesting case is when $\Delta_x$ becomes a probability simplex.
If $\phi(x)  = \sum_{i\,=\,1}^d (x_i\log x_i -x_i)$, then $D(x,x') =  \sum_{i\,=\,1}^d x_i\log (x_i /x_i') -\sum_{i\,=\,1}^d (x_i- x_i')$ defines the unnormalized Kullback-Leibler divergence~\citep{cover1999elements,boyd2004convex}. This is the setup we will discuss later.

\begin{lemma}\label{lem.pushback}
	Let $f(x,y)$: $\calX\times\calY\to \mathbb{R}$ be a continuous differentiable function that is convex in $x$ and concave in $y$, where $\calX$ and $\calY$ are compact convex sets in $\mathbb{R}^d$. Suppose for some $x' \in\Delta_x^o$ and $y'\in\Delta_y^o$, 
	\[
	(x^\star,y^\star) \;\in\; \argminimax_{x\,\in\,\Delta_x,\,y\,\in\,\Delta_y}\; f(x,y) + \eta^{-1}D(x\,\vert\,x') - \eta^{-1}D(y\,\vert\,y')
	\]
	and $x^\star\in\Delta_x^o$ and $y^\star\in\Delta_y^o$,
	where $\eta>0$. 
	Then, for any $x\in\Delta_x$ and $y\in\Delta_y$,
	\[
	f(x^\star,y) + \eta^{-1} \big(D(x^\star,x') + D(y^\star,y')\big)
	\;\leq\;
	f(x,y^\star) + \eta^{-1} \big( D(x,x') + D(y,y') - D(x,x^\star) - D(y,y^\star) \big).
	\]
\end{lemma}
\begin{proof}
	For the smooth convex-concave function $f$, it is necessary to have the first-order stationary condition on $(x^\star,y^\star)$. 
	There exist $\nabla_x f(x^\star,y^\star)$ and $\nabla_y(x^\star,y^\star)$ such that
	\begin{subequations}
		\begin{equation}\label{eq.1stx}
		\inner{\nabla_xf(x^\star,y^\star) +\eta^{-1} \big(\phi(x^\star)-\phi(x')\big) }{x-x^\star}\;\geq\;0,\;\text{ for any } x\in\Delta_x
		\end{equation}
		\begin{equation}\label{eq.1sty}
		\inner{-\nabla_yf(x^\star,y^\star) +\eta^{-1} \big(\phi(y^\star)-\phi(y')\big) }{y-y^\star}\;\geq\;0,\;\text{ for any } y\in\Delta_y.
		\end{equation}
	\end{subequations}
	
	By the definition of $D(\cdot\,\vert\,\cdot)$,
	\[
	\begin{array}{rcl}
	&& \!\!\!\! \!\!\!\! \!\! 
	\eta^{-1} \big(D(x,x') - D(x,x^\star)\big)
	\\[0.2cm]
	&=&
	\eta^{-1}\big(
	\phi(x^\star) - \phi(x') - \langle \nabla\phi(x'), x-x'\rangle
	+ \langle \nabla\phi(x^\star), x-x^\star\rangle
	\big)
	\\[0.2cm]
	&=&\eta^{-1}\big(
	\phi(x^\star) - \phi(x') 
	- \eta^{-1} \langle \nabla\phi(x'), x^\star-x'\rangle 	\big)-\langle \nabla_xf(x^\star,y^\star) , x-x^\star\rangle
	\\[0.2cm]
	&&
	\,+\, \langle \nabla_xf(x^\star,y^\star) + \eta^{-1} \big(\nabla\phi(x^\star)-\nabla\phi(x')\big), x-x^\star\rangle
	\\[0.2cm]
	&=&\eta^{-1}D(x^\star,x')-\langle \nabla_xf(x^\star,y^\star) , x-x^\star\rangle
	\\[0.2cm]
	&&
	\,+\, \langle \nabla_xf(x^\star,y^\star) + \eta^{-1} \big(\nabla\phi(x^\star)-\nabla\phi(x')\big), x-x^\star\rangle.
	\end{array}
	\]
	Application of~\eqref{eq.1stx} leads to
	\begin{equation}\label{eq.pushbackx}
	\begin{array}{rcl}
	\eta^{-1} \big(D(x,x') - D(x,x^\star)\big)
	&\geq&
	\eta^{-1}D(x^\star,x')
	-\langle \nabla_xf(x^\star,y^\star) , x-x^\star\rangle
	\\[0.2cm]
	&\geq&
	\eta^{-1}D(x^\star,x')
	+  f(x^\star,y^\star)- f(x,y^\star) 
	\end{array}
	\end{equation}
	where the last inequality is due to the convexity fo $f(x,y^\star)$ in $x$:
	$f(x,y^\star)\geq f(x^\star,y^\star)+\langle\nabla_x f(x^\star,y^\star),x-x^\star\rangle$.
	
	Similarly, we work on $\eta^{-1} \big(D(y,y') - D(y,y^\star)\big)$ and~\eqref{eq.1sty}.
	\begin{equation}\label{eq.pushbacky}
	\begin{array}{rcl}
	\eta^{-1} \big(D(y,y') - D(y,y^\star)\big)
	&\geq&
	\eta^{-1}D(y^\star,y')
	+ f(x^\star,y) -  f(x^\star,y^\star).
	\end{array}
	\end{equation}
	
	Finally, we conclude the proof by adding~\eqref{eq.pushbackx} to~\eqref{eq.pushbacky} from both sides.
\end{proof}

Before the proof of Lemma~\ref{lem.gap_pd}, we next show some useful bounds on the unnormalized Kullback-Leibler divergence.

\begin{lemma}\label{lem.D_lb}
	Let $q(x,a,x')$ and $q'(x,a,x')$ be two occupancy measures, and $q(x,a)$ and $q'(x,a)$ be the associated state-action visitation probability distributions. Then,
	\[
	D(q, q') \;\geq\;\frac{1}{2L} \norm{q-q'}_1^2.
	\]
\end{lemma}
\begin{proof}
	We recall  $q(x,a)$ and $q'(x,a)$,
	\[
	q(x,a)\;=\;\sum_{x'\,\in\,X_\ell} q(x,a,x')
	\;\text{ and }\;
	q'(x,a)\;=\;\sum_{x'\,\in\,X_\ell} q'(x,a,x')
	\]
	where $\ell$ is the layer that $x$ belongs to. We note that $q(x,a)$ and $q'(x,a)$ define probability laws for each $\ell\in\{ 0,1,\ldots,L-1 \}$, and  $\sum_{x\,\in\,X_\ell}\sum_{a\,\in\,A}q(x,a) =  \sum_{x\,\in\,X_\ell}\sum_{a\,\in\,A}q'(x,a) = 1$.
	
	By the definition,
	\[
	\begin{array}{rcl}
	D(q, q') 
	& = & 
	\displaystyle
	\sum_{\ell\,=\,0}^{L-1} \sum_{x\,\in\,X_\ell} \sum_{a\,\in\,A} q(x,a) \log\frac{q(x,a)}{q'(x,a)}
	\,-\,
	\sum_{\ell\,=\,0}^{L-1} \sum_{x\,\in\,X_\ell} \sum_{a\,\in\,A}  \big(q(x,a)-q'(x,a)\big)
	\\[0.2cm]
	& = & 
	\displaystyle
	\sum_{\ell\,=\,0}^{L-1} \sum_{x\,\in\,X_\ell} \sum_{a\,\in\,A} q(x,a) \log\frac{q(x,a)}{q'(x,a)}
	\\[0.2cm]
	& \geq & 
	\displaystyle\frac{1}{2}
	\sum_{\ell\,=\,0}^{L-1} \norm{ q(x,a)-q'(x,a)}_1^2
	\\[0.2cm]
	& \geq & 
	\displaystyle\frac{1}{2L}
	\rbr{ \sum_{\ell\,=\,0}^{L-1} \norm{ q(x,a)-q'(x,a)}_1}^2
	\\[0.2cm]
	& = & 
	\displaystyle\frac{1}{2L} \norm{ q-q'}_1^2
	\end{array}
	\]
	where we apply the Pinsker's inequality to $\sum_{x\,\in\,X_\ell} \sum_{a\,\in\,A} q(x,a) \log\frac{q(x,a)}{q'(x,a)}$ in the first inequality.
	
\end{proof}

\begin{lemma}\label{lem.D_difference}
	Let $q(x,a,x')$ and $q'(x,a,x')$ be two occupancy measures, and $q(x,a)$ and $q'(x,a)$ be the associated state-action visitation probability laws. 
	Define $\tilde{q}{\,'}(x,a) = (1-\theta)q'(x,a)+\theta \frac{1}{|X_\ell||A|}$ for $(x,a)\in X_\ell\times A$, $\ell\in\{0,1,\ldots,L-1 \}$, and $\theta\in (0,1]$. Then,
	\[
	D(q,\tilde{q}{\,'}) -  D(q,q') \;\leq\; \theta L\log (|X||A|)
	\;\text{ and }\;
	D(q,\tilde{q}{\,'}) \;\leq\; L\log(|X||A|/\theta).
	\]
\end{lemma}
\begin{proof}
	By the definition,
	\[
	\begin{array}{rcl}
	&&  \!\!\!\! \!\!\!\! \!\! 
	D(q,\tilde{q}{\,'}) -  D(q, q') 
	\\[0.2cm]
	& = & 
	\displaystyle
	\sum_{\ell\,=\,0}^{L-1} \sum_{x\,\in\,X_\ell} \sum_{a\,\in\,A} q(x,a) \rbr{\log\frac{q(x,a)}{\tilde q{\,'}(x,a)} -  \log\frac{q(x,a)}{ q'(x,a)}}
	-
	\sum_{\ell\,=\,0}^{L-1} \sum_{x\,\in\,X_\ell} \sum_{a\,\in\,A}  \big(q'(x,a)-\tilde{q}{\,'}(x,a)\big)
	\\[0.2cm]
	& = & 
	\displaystyle
	\sum_{\ell\,=\,0}^{L-1} \sum_{x\,\in\,X_\ell} \sum_{a\,\in\,A} q(x,a) \rbr{\log{ q'(x,a)}  -  \log{\tilde q{\,'}(x,a)} }
	\\[0.2cm]
	& = & 
	\displaystyle
	\sum_{\ell\,=\,0}^{L-1} \sum_{x\,\in\,X_\ell} \sum_{a\,\in\,A} q(x,a) \rbr{\log{ q'(x,a)}  -  \log\rbr{(1-\theta) q'(x,a) + \theta\frac{1}{|X_\ell||A|}} }.
	\end{array}
	\]
	By the Jensen's inequality,	
	\[
	\begin{array}{rcl}
	&& \!\!\!\! \!\!\!\! \!\! 
	D(q,\tilde{q}{\,'}) -  D(q, q') 
	\\[0.2cm]
	& \leq & 
	\displaystyle
	\sum_{\ell\,=\,0}^{L-1} \sum_{x\,\in\,X_\ell} \sum_{a\,\in\,A} q(x,a) \rbr{\log{ q'(x,a)}  -  (1-\theta)\log q'(x,a)  - \theta \log{ \frac{1}{|X_\ell||A|}} }
	\\[0.2cm]
	& = & 
	\displaystyle
	\sum_{\ell\,=\,0}^{L-1} \sum_{x\,\in\,X_\ell} \sum_{a\,\in\,A} \theta q(x,a) \rbr{\log{ q'(x,a)} +  \log{|X_\ell||A|} }
	\\[0.2cm]
	& \leq  & 
	\displaystyle
	\sum_{\ell\,=\,0}^{L-1} \sum_{x\,\in\,X_\ell}  \sum_{a\,\in\,A} \theta q(x,a)  \log{|X_\ell||A|}
	\\[0.2cm]
	& \leq  & 
	\displaystyle \theta L \log |X||A|
	\end{array}
	\]
	where the second inequality is due to that a negative entropy is non-positive. 
	
	We next prove the second inequality. By the definition,
	\[
	\begin{array}{rcl}
	D(q,\tilde{q}{\,'})  
	& = & 
	\displaystyle
	\sum_{\ell\,=\,0}^{L-1} \sum_{x\,\in\,X_\ell} \sum_{a\,\in\,A} q(x,a) {\log\frac{q(x,a)}{\tilde q{\,'}(x,a)} }
	-\sum_{\ell\,=\,0}^{L-1} \sum_{x\,\in\,X_\ell} \sum_{a\,\in\,A}  \big(q(x,a)-\tilde{q}{\,'}(x,a)\big)
	\\[0.2cm]
	& = & 
	\displaystyle
	\sum_{\ell\,=\,0}^{L-1} \sum_{x\,\in\,X_\ell} \sum_{a\,\in\,A} q(x,a) \rbr{\log{ q(x,a)}  -  \log{\tilde q{\,'}(x,a)} }
	\\[0.2cm]
	& = & 
	\displaystyle
	\sum_{\ell\,=\,0}^{L-1} \sum_{x\,\in\,X_\ell} \sum_{a\,\in\,A} q(x,a) \rbr{\log{ q(x,a)}  -  \log\rbr{(1-\theta) q'(x,a) + \theta\frac{1}{|X_\ell||A|}} }
	\\[0.2cm]
	& \leq & 
	\displaystyle
	\sum_{\ell\,=\,0}^{L-1} \sum_{x\,\in\,X_\ell} \sum_{a\,\in\,A} -q(x,a) {  \log\rbr{(1-\theta) q'(x,a) + \theta\frac{1}{|X_\ell||A|}} }
	\end{array}
	\]
	where the last inequality is due to that a negative entropy is non-positive. We note that $-\log (\cdot)$ is a non-increasing function. We can simplify the upper bound on $D(q,\tilde{q}{\,'})$ above by,
	\[
	\begin{array}{rcl}
	D(q,\tilde{q}{\,'}) 
	& \leq & 
	\displaystyle
	\sum_{\ell\,=\,0}^{L-1} \sum_{x\,\in\,X_\ell} \sum_{a\,\in\,A} -q(x,a) {  \log\rbr{\theta\frac{1}{|X_\ell||A|}} }
	\\[0.2cm]
	& = & 
	\displaystyle
	\sum_{\ell\,=\,0}^{L-1} \sum_{x\,\in\,X_\ell} \sum_{a\,\in\,A} q(x,a)  \log{\frac{|X_\ell||A|}{\theta}} 
	\\[0.2cm]
	& \leq  & 
	\displaystyle
	\sum_{\ell\,=\,0}^{L-1} \sum_{x\,\in\,X_\ell}  \sum_{a\,\in\,A} q(x,a)  \log{\frac{|X||A|}{\theta}}
	\\[0.2cm]
	& =  & 
	\displaystyle  L \log \frac{|X||A|}{\theta}.
	\end{array}
	\]
\end{proof}

We now are ready to prove Lemma~\ref{lem.gap_pd}. 

\begin{proof}[Proof of Lemma~\ref{lem.gap_pd}]
	By Lemma~\ref{lem.empiricalP}, with probability $1-\delta$ it holds that 
	\[
	\Delta(P_1) \;\subset\; \cap_{t \,=\,0}^{T-1} \Delta(k_1^t) 
	\;\text{ and }\; 
	\Delta(P_2) \;\subset\; \cap_{t \,=\,0}^{T-1} \Delta(k_2^t).
	\]
	We note that the solution $(q_1^\star,q_2^\star)$ in hindsight to Problem~\eqref{eq.opt_wc} satisfies $q_1^\star\in\Delta(P_1)$ and $q_2^\star\in \Delta(P_2)$. Hence, $q_1^\star\in  \cap_{t \,=\,0}^{T-1} \Delta(k_1^t)  $ and $q_2^\star\in \Delta(P_2) \cap_{t \,=\,0}^{T-1} \Delta(k_2^t)$ with probability $1-\delta$. For episode $t$, we apply Lemma~\ref{lem.pushback} to the primal update~\eqref{eq.primal} with 
	\[
	f(x,y) \vert_{x\,=\,q_1,\, y\,=\,q_2} 
	\;=\;
	V\, \big\langle{q_1\cdot \hat q_2^{\,t-1}+\hat q_1^{\,t-1}\cdot q_2},{r^{t-1}}\big\rangle 
	\,+\,
	\lambda^{t-1} \langle{q_1},{g^{t-1}}\rangle 
	\,-\,\lambda^{t-1} \langle{q_2},{h^{t-1}}\rangle 
	\]
	and $x^\star = \hat q_1^{\,t}$, $y^\star= \hat q_2^{\,t}$, $x' = \tilde q_1^{\,t-1}$, $y'=\tilde q_2^{\,t-1}$, $x = q_1^\star$, and $y=q_2^\star$. Thus, with  probability $1-\delta$ it holds for any $t$ that
	\[
	\begin{array}{rcl}
	&& \!\!\!\! \!\!\!\! \!\! V\, \big\langle{\hat q_1^{\,t}\cdot \hat q_2^{\,t-1}+\hat q_1^{\,t-1}\cdot  q_2^\star},{r^{t-1}}\big\rangle 
	\,+\,
	\lambda^{t-1} \langle{\hat q_1^{\,t}},{g^{t-1}}\rangle 
	\,-\,\lambda^{t-1} \langle{ q_2^\star},{h^{t-1}}\rangle 
	\\[0.2cm]
	&& \!\!\!\! \!\!\!\! \!\! \,+\, \eta^{-1} \big(D(\hat q_1^{\,t},\tilde q_1^{\,t-1}) + D(\hat q_2^{\,t}, \tilde q_2^{\,t-1})\big)
	\\[0.2cm]
	&\leq  & V\, \big\langle{ q_1^\star\cdot \hat q_2^{\,t-1}+\hat q_1^{\,t-1}\cdot \hat q_2^{\,t}},{r^{t-1}}\big\rangle 
	\,+\,
	\lambda^{t-1} \langle{ q_1^\star},{g^{t-1}}\rangle 
	\,-\,\lambda^{t-1} \langle{ \hat q_2^{\,t}},{h^{t-1}}\rangle 
	\\[0.2cm]
	&& \,+\, \eta^{-1} \big(D( q_1^\star,\tilde q_1^{\,t-1}) \,+\, D( q_2^\star, \tilde q_2^{\,t-1}) \,-\, D( q_1^\star,\hat q_1^{\,t}) \,-\, D( q_2^\star, \hat q_2^{\,t}) \big)
	\end{array}
	\]
	or, equivalently,
	\begin{equation}\label{eq.pushback_V}
	\begin{array}{rcl}
	&& \!\!\!\! \!\!\!\! \!\! V\, \big\langle{\hat q_1^{\,t}\cdot \hat q_2^{\,t-1}- \hat q_1^{\,t-1}\cdot \hat q_2^{\,t}},{r^{t-1}}\big\rangle 
	\,+\,
	\lambda^{t-1} \langle{\hat q_1^{\,t}},{g^{t-1}}\rangle 
	\,+\,\lambda^{t-1} \langle{ \hat q_2^{\,t}},{h^{t-1}}\rangle 
	\\[0.2cm]
	&& \!\!\!\! \!\!\!\! \!\! \,+\, \eta^{-1} \big(D(\hat q_1^{\,t}, \tilde q_1^{\,t-1}) + D(\hat q_2^{\,t}, \tilde q_2^{\,t-1})\big)
	\\[0.2cm]
	&\leq&  V\, \big\langle{ q_1^\star\cdot \hat q_2^{\,t-1}-\hat q_1^{\,t-1}\cdot  q_2^\star },{r^{t-1}}\big\rangle 
	\,+\,
	\lambda^{t-1} \langle{ q_1^\star},{g^{t-1}}\rangle 
	\,+\,\lambda^{t-1} \langle{ q_2^\star},{h^{t-1}}\rangle 
	\\[0.2cm]
	&& \,+\, \eta^{-1} \big(D( q_1^\star,\tilde q_1^{\,t-1}) \,+\, D( q_2^\star, \tilde q_2^{\,t-1}) \,-\, D( q_1^\star,\hat q_1^{\,t}) \,-\, D( q_2^\star, \hat q_2^{\,t}) \big).
	\end{array}
	\end{equation}
	
	Let $\Delta^t\DefinedAs \frac{1}{2} \rbr{(\lambda^{t})^2-(\lambda^{t-1})^2}$ be the drift of the consecutive dual updates. Then,
	\begin{equation}\label{eq.drift_Q}
	\begin{array}{rcl}
	\Delta^t &=& \displaystyle 
	\frac{1}{2}\rbr{ ( \lambda^{t} )^2 - ( \lambda^{t-1} )^2 }
	\\[0.2cm]
	&=& \displaystyle 
	\frac{1}{2}\rbr{ \max\!^2 \Big( \lambda^{t-1} \,+\, \big( \langle{\hat q_1^{\,t}},{g^{t-1}}\rangle+\langle{\hat q_2^{\,t}},{h^{t-1}}\rangle-b \big),\; 0 \Big) - ( \lambda^{t-1} )^2 }
	\\[0.2cm]
	&\leq& \displaystyle 
	\lambda^{t-1}\big( \langle{\hat q_1^{\,t}},{g^{t-1}}\rangle+\langle{\hat q_2^{\,t}},{h^{t-1}}\rangle-b \big) \,+\, \frac{1}{2}\big( \langle{\hat q_1^{\,t}},{g^{t-1}}\rangle+\langle{\hat q_2^{\,t}},{h^{t-1}}\rangle-b \big)^2
	\\[0.2cm]
	&\leq& \displaystyle 
	\lambda^{t-1}\big( \langle{\hat q_1^{\,t}},{g^{t-1}}\rangle+\langle{\hat q_2^{\,t}},{h^{t-1}}\rangle-b \big) \,+\, 2L^2
	\end{array}
	\end{equation}
	where the first inequality is due to $\max^2(x,0)\leq x^2$ and we apply $\langle{\hat q_1^{\,t}},{g^{t-1}}\rangle$, $\langle{\hat q_2^{\,t}},{h^{t-1}}\rangle$, $b\in [0,L]$ in the last inequality.
	Adding~\eqref{eq.drift_Q} to~\eqref{eq.pushback_V} from both sides of the inequalities without changing the inequality direction yields
	\begin{equation}\label{eq.pushback_VQ}
	\begin{array}{rcl}
	&& \!\!\!\! \!\!\!\! \!\! 
	V\, \big\langle{\hat q_1^{\,t}\cdot \hat q_2^{\,t-1}- \hat q_1^{\,t-1}\cdot \hat q_2^{\,t}},{r^{t-1}}\big\rangle 
	\,+\,
	\Delta^{t} 
	\,+\, \eta^{-1} \big(D(\hat q_1^{\,t},\tilde q_1^{\,t-1}) + D(\hat q_2^{\,t}, \tilde q_2^{\,t-1})\big)
	\\[0.2cm]
	&\leq& V\, \big\langle{ q_1^\star\cdot \hat q_2^{\,t-1}-\hat q_1^{\,t-1}\cdot  q_2^\star },{r^{t-1}}\big\rangle 
	\,+\,
	\lambda^{t-1} \big(\langle{ q_1^\star},{g^{t-1}}\rangle 
	+ \langle{ q_2^\star},{h^{t-1}}\rangle -b\big)
	\,+\,2L^2
	\\[0.2cm]
	&& \,+\, \eta^{-1} \big(D( q_1^\star,\tilde q_1^{\,t-1}) + D( q_2^\star, \tilde q_2^{\,t-1}) \,-\, D( q_1^\star,\hat q_1^{\,t}) \,-\, D( q_2^\star, \hat q_2^{\,t}) \big).
	\end{array}
	\end{equation}
	However,
	\[
	\begin{array}{rcl}
	&& \!\!\!\! \!\!\!\! \!\! V\, \big\langle{\hat q_1^{\,t}\cdot \hat q_2^{\,t-1}- \hat q_1^{\,t-1}\cdot \hat q_2^{\,t}},{r^{t-1}}\big\rangle 
	\,+\, \eta^{-1} \big(D(\hat q_1^{\,t},\tilde q_1^{\,t-1}) + D(\hat q_2^{\,t}, \tilde q_2^{\,t-1})\big)
	\\[0.2cm]
	&=& V\, \big\langle{\hat q_1^{\,t}\cdot \hat q_2^{\,t-1}- \tilde q_1^{\,t-1}\cdot \hat q_2^{\,t-1}},{r^{t-1}}\big\rangle 
	\,+\,V\, \big\langle{ \tilde q_1^{\,t-1}\cdot \hat q_2^{\,t-1} - \hat q_1^{\,t-1}\cdot \hat q_2^{\,t-1}},{r^{t-1}}\big\rangle 
	\\[0.2cm]
	&&  
	\,+\,V\, \big\langle{ \hat q_1^{\,t-1}\cdot \hat q_2^{\,t-1} - \hat q_1^{\,t-1}\cdot \tilde q_2^{\,t-1}},{r^{t-1}}\big\rangle 
	\,+\, V\, \big\langle{\hat q_1^{\,t-1}\cdot \tilde q_2^{\,t-1}- \hat q_1^{\,t-1}\cdot \hat q_2^{\,t}},{r^{t-1}}\big\rangle 
	\\[0.2cm]
	&&  \,+\, \eta^{-1} D(\hat q_1^{\,t},\tilde q_1^{\,t-1}) \,+\, \eta^{-1} D(\hat q_2^{\,t}, \tilde q_2^{\,t-1})
	\\[0.2cm]
	&\geq&  -\,V\, \norm{\hat q_2^{\,t-1}\cdot r^{t-1} }_\infty \norm{\hat q_1^{\,t} - \tilde q_1^{\,t-1}}_1
	\,-\,V\norm{ \hat q_2^{\,t-1} \cdot r^{t-1} }_\infty\norm{ \tilde q_1^{\,t-1}- \hat q_1^{\,t-1}}_1
	\\[0.2cm]
	&& 
	\,-\,V\,\norm{ \hat q_1^{\,t-1}\cdot  r^{t-1} }_\infty\norm{ \hat q_2^{\,t-1} - \tilde q_2^{\,t-1}}_1
	\,-\, V\,\norm{ \hat q_1^{\,t-1}\cdot r^{t-1} }_\infty\norm{ \tilde q_2^{\,t-1}-\hat q_2^{\,t} }_1
	\\[0.2cm]
	&& \,+\, (2\eta L)^{-1}\norm{\hat q_1^{\,t}-\tilde q_1^{\,t-1}}_1^2 \,+\, (2\eta L)^{-1} \norm{\hat q_2^{\,t} - \tilde q_2^{\,t-1}}_1
	\\[0.2cm]
	&\geq&  -\,V\, \norm{\hat q_1^{\,t} - \tilde q_1^{\,t-1}}_1
	\,-\, 2\theta V L
	\,+\, (2\eta L)^{-1}\norm{\hat q_1^{\,t}-\tilde q_1^{\,t-1}}_1^2
	\\[0.2cm]
	&& 
	\,-\,2\theta VL
	\,-\, V\,\norm{ \tilde q_2^{\,t-1}-\hat q_2^{\,t} }_1
	\,+\, (2\eta L)^{-1} \norm{\hat q_2^{\,t} - \tilde q_2^{\,t-1}}_1
	\\[0.2cm]
	&\geq& \,-\,4\theta V L\,-\,\eta V^2 L
	\end{array}
	\]
	where we apply the H\"older's inequality and Lemma~\ref{lem.D_lb} in the first inequality, the second inequality is due to that
	\[
	\begin{array}{rcl}
	\norm{ \tilde q_1^{\,t-1}- \hat q_1^{\,t-1}}_1 
	&=& \displaystyle\sum_{\ell\,=\,0}^{L-1}\sum_{x\,\in\,X_\ell}\sum_{a\,\in\,A}\abr{ (1-\theta)\hat {q}_1^{\,t-1}(x,a)+\theta\frac{1}{|X_\ell||A|}-\hat q_1^{\,t-1}(x,a) }
	\\[0.2cm]
	&\leq& \displaystyle \theta\sum_{\ell\,=\,0}^{L-1}\sum_{x\,\in\,X_\ell}\sum_{a\,\in\,A} \hat {q}_1^{\,t-1}(x,a)+\theta\sum_{\ell\,=\,0}^{L-1}\sum_{x\,\in\,X_\ell}\sum_{a\,\in\,A}\frac{1}{|X_\ell||A|}
	\\[0.2cm]
	&=& 2\theta L
	\end{array}
	\]
	and $\Vert{ \tilde q_2^{\,t-1}- \hat q_2^{\,t-1}}\Vert_1 \leq 2\theta L$ that can be proved similarly, and the last inequality is due to 
	$-bx+ax^2\geq - b^2/(4a)$ for any $a$, $b > 0$. Therefore, we take the lower bound above for the left-hand side of~\eqref{eq.pushback_VQ},
	\begin{equation}\label{eq.pushback_simplify}
	\begin{array}{rcl}
	&& \!\!\!\!\!\! 
	\Delta^{t} \,-\,4\theta V L\,-\,\eta V^2 L
	\\[0.2cm]
	&& \!\!\!\!\!\! \leq V\, \big\langle{ q_1^\star\cdot \hat q_2^{\,t-1}-\hat q_1^{\,t-1}\cdot  q_2^\star },{r^{t-1}}\big\rangle 
	\,+\,
	\lambda^{t-1} \big(\langle{ q_1^\star},{g^{t-1}}\rangle 
	+ \langle{ q_2^\star},{h^{t-1}}\rangle -b\big)
	\,+\,2L^2
	\\[0.2cm]
	&& \,+\, \eta^{-1} \big(D( q_1^\star,\tilde q_1^{t-1}) + D( q_2^\star, \tilde q_2^{\,t-1}) -D( q_1^\star,\hat q_1^{\,t}) - D( q_2^\star, \hat q_2^{\,t}) \big).
	\end{array}
	\end{equation}
	By Lemma~\ref{lem.D_difference},
	\[
	\begin{array}{rcl}
	D( q_1^\star,\tilde q_1^{\,t-1}) - D( q_1^\star,\hat q_1^{\,t}) &=& D( q_1^\star,\tilde q_1^{\,t-1}) - D( q_1^\star,\hat q_1^{\,t-1})  + D( q_1^\star,\hat q_1^{\,t-1}) - D( q_1^\star,\hat q_1^{\,t}) 
	\\[0.2cm]
	&\leq & \theta L\log (|X||A|)+ D( q_1^\star,\hat q_1^{\,t-1}) - D( q_1^\star,\hat q_1^{\,t}) 
	\end{array}
	\]
	and, similarly, 
	\[
	D( q_2^\star, \tilde q_2^{\,t-1})  - D( q_2^\star, \hat q_2^{\,t}) \;\leq\; \theta L\log (|Y||B|)+ D( q_2^\star,\hat q_2^{\,t-1}) - D( q_2^\star,\hat q_2^{\,t}). 
	\]
	We now simplify~\eqref{eq.pushback_simplify} into
	\[
	\begin{array}{rcl}
	\Delta^{t} 
	&\leq&  V\, \big\langle{ q_1^\star\cdot \hat q_2^{\,t-1}-\hat q_1^{\,t-1}\cdot  q_2^\star },{r^{t-1}}\big\rangle 
	\,+\,
	\lambda^{t-1} \big(\langle{ q_1^\star},{g^{t-1}}\rangle 
	+ \langle{ q_2^\star},{h^{t-1}}\rangle -b\big)
	\\[0.2cm]
	&& \,+\, \eta^{-1} \big(D( q_1^\star,\hat q_1^{\,t-1}) + D( q_2^\star, \hat q_2^{\,t-1}) -D( q_1^\star,\hat q_1^{\,t}) - D( q_2^\star, \hat q_2^{\,t}) \big)
	\\[0.2cm]
	&& \,+\, \eta^{-1}\theta L \big(\log (|X||A|) +\log (|Y||B|)\big)\,+\,2L^2 \,+\,4\theta V L\,+\,\eta V^2 L
	\end{array}
	\]
	which leads to the desired result by summing it up from $t=1$ to $T$,
	\[
	\begin{array}{rcl}
	\displaystyle
	\sum_{t\,=\,1}^{T}\Delta^{t} 
	&\leq& \displaystyle V\sum_{t\,=\,1}^{T} \big\langle{ q_1^\star\cdot \hat q_2^{\,t-1}-\hat q_1^{\,t-1}\cdot  q_2^\star },{r^{t-1}}\big\rangle 
	\,+\,\sum_{t\,=\,1}^{T}
	\lambda^{t-1} \big(\langle{ q_1^\star},{g^{t-1}}\rangle 
	+ \langle{ q_2^\star},{h^{t-1}}\rangle -b\big)
	\\[0.2cm]
	&& \displaystyle\,+\, \eta^{-1} \sum_{t\,=\,1}^{T}\big(D( q_1^\star,\hat q_1^{\,t-1}) + D( q_2^\star, \hat q_2^{\,t-1}) -D( q_1^\star,\hat q_1^{\,t}) - D( q_2^\star, \hat q_2^{\,t}) \big)
	\\[0.2cm]
	&&\displaystyle \,+\, \eta^{-1}\theta L T\big(\log (|X||A|) +\log (|Y||B|)\big)\,+\,2L^2 T\,+\,4\theta V LT\,+\,\eta V^2 LT
	\\[0.2cm]
	&\leq& \displaystyle V\sum_{t\,=\,1}^{T} \big\langle{ q_1^\star\cdot \hat q_2^{\,t-1}-\hat q_1^{\,t-1}\cdot  q_2^\star },{r^{t-1}}\big\rangle 
	\,+\,\sum_{t\,=\,1}^{T}
	\lambda^{t-1} \big(\langle{ q_1^\star},{g^{t-1}}\rangle 
	+ \langle{ q_2^\star},{h^{t-1}}\rangle -b\big)
	\\[0.2cm]
	&& \displaystyle\,+\, \eta^{-1} \big(D( q_1^\star,\hat q_1^{\,0}) + D( q_2^\star, \hat q_2^{\,0}) \big)
	\\[0.2cm]
	&&\displaystyle \,+\, \eta^{-1}\theta L T\big(\log (|X||A|) +\log (|Y||B|)\big)\,+\,2L^2 T\,+\,4\theta V LT\,+\,\eta V^2 LT
	\end{array}
	\]
	which leads to the desired result by noting that 
	\[
	D( q_1^\star,\hat q_1^{\,0})\;\leq\; L\log(|X||A|), \; D( q_2^\star, \hat q_2^{\,0}) \;\leq\; L\log(|Y||B|), \; \text{ and } \; \sum_{t\,=\,1}^{T}\Delta^{t} \; \geq \; 0. 
	\]
\end{proof}

\section{Proofs of Lemma~\ref{lem.lambda} and Lemma~\ref{lem.violation}}
\label{ap.violation}

We first present the boundedness of the dual update $\lambda^t$ in Lemma~\ref{lem.lambda}. Our proof is based on a new drift analysis in Lemma~\ref{lem.drift} that has been established in~\cite{yu2017online} for providing a high probability bound for stochastic processes.

\begin{proof}[Proof of Lemma~\ref{lem.lambda}]
	Let $\calF^{t}$ be an $\sigma$-algebra that is generated by the state-action sequence, reward/utility functions for both players up to episode $t$. At the beginning, $\calF^0=\{\emptyset,\Omega \}$. We have a discrete-time random process $\{\lambda^{t},t\geq0 \}$ that adapts to $\calF^{t}$. 
	It suffices to check all assumptions in Lemma~\ref{lem.drift}. 
	
	By the dual update~\eqref{eq.dual},
	\[
	\begin{array}{rcl}
	\abr{  \lambda^{t+1}  - \lambda^{t} } 
	&=&
	\abr{ \max\! \Big( \lambda^{t} \,+\, \big( \langle{\hat q_1^{\,t+1}},{g^{t}}\rangle+\langle{\hat q_2^{\,t+1}},{h^{t}}\rangle-b \big),\; 0 \Big) - \lambda^{t}  }
	\\[0.2cm]
	&\leq&\abr{ \langle{\hat q_1^{\,t+1}},{g^{t}}\rangle+\langle{\hat q_2^{\,t+1}},{h^{t}}\rangle-b  }
	\\[0.2cm]
	&\leq& 2L
	\end{array}
	\]
	where the first inequality is clear from two cases for $\max(\cdot)$ and the second inequality is due to $\langle{\hat q_1^{\,t+1}},{g^{t}}\rangle$, $\langle{\hat q_2^{\,t+1}},{h^{t}}\rangle\in [0,L]$, $b\in [0,2L]$. Consequently,
	\begin{equation}\label{eq.lambda.diff}
	\lambda^{t+t_0}  - \lambda^{t}  \;=\;\sum_{s\,=\,t}^{t+t_0-1} \big( \lambda^{s+1}  - \lambda^{s} \big)
	\;\leq\;\sum_{s\,=\,t}^{t+t_0-1} \abr{ \lambda^{s+1}  - \lambda^{s} }
	\;\leq\; 2t_0 L
	\end{equation}
	which leads to $\mathbb{E} [ \, \lambda^{t+t_0}  - \lambda^{t} \,\vert\,\calF^t\, ]\leq 2t_0 L$. It is convenient to take $\delta_{\max} = 2L$ in Lemma~\ref{lem.drift}.
	
	We next determine the validity of other assumptions in Lemma~\ref{lem.drift}. Let us denote the event in Lemma~\ref{lem.empiricalP} by $\mathcal{E}_{\text{good}}$ and we have $P(\mathcal{E}_{\text{good}}) \geq 1-\delta$.
	We recall that the proof of Lemma~\ref{lem.gap_pd} remains to be valid if we replace $q_1^\star$ by $\bar q_1$ and $q_2^\star$ by $\bar q_2$ starting from~\eqref{eq.pushback_V}. By doing so, it is ready to obtain a similar result as~\eqref{eq.pushback_simplify}: under the good event $\mathcal{E}_{\text{good}}$ it holds for any $\tau$ that
	\[
	\begin{array}{rcl}
	&& \!\!\!\! \!\!\!\! \!\! 
	\Delta^{\tau} \,-\,4\theta V L\,-\,\eta V^2 L
	\\[0.2cm]
	& \leq & V\, \big\langle{ \bar q_1\cdot \hat q_2^{\,\tau-1}-\hat q_1^{\,\tau-1}\cdot \bar q_2 },{r^{\tau-1}}\big\rangle 
	\,+\,
	\lambda^{\tau-1} \big(\langle{ \bar q_1},{g^{\tau-1}}\rangle 
	+ \langle{ \bar q_2},{h^{\tau-1}}\rangle -b\big)
	\,+\,2L^2
	\\[0.2cm]
	&& \,+\, \eta^{-1} \big(D(\bar q_1,\tilde q_1^{\,\tau-1}) + D( \bar q_2, \tilde q_2^{\,\tau-1}) -D( \bar q_1,\hat q_1^{\,\tau}) - D( \bar q_2, \hat q_2^{\,\tau}) \big)
	\end{array}
	\]
	or, equivalently,
	\begin{equation}\label{eq.pushback_simplify_slater}
	\begin{array}{rcl}
	&& \!\!\!\! \!\!\!\! \!\! 
	(\lambda^\tau)^2 \,-\, (\lambda^{\tau-1})^2
	\\[0.2cm]
	&\leq& 2V\, \big\langle{ \bar q_1\cdot \hat q_2^{\,\tau-1}-\hat q_1^{\,\tau-1}\cdot \bar q_2 },{r^{\tau-1}}\big\rangle 
	\,+\,
	2\lambda^{\tau-1} \big(\langle{ \bar q_1},{g^{\tau-1}}\rangle 
	+ \langle{ \bar q_2},{h^{\tau-1}}\rangle -b\big)
	\,+\,4L^2
	\\[0.2cm]
	&& \,+\, 2\eta^{-1} \big(D(\bar q_1,\tilde q_1^{\,\tau-1}) + D( \bar q_2, \tilde q_2^{\,\tau-1}) -D( \bar q_1,\hat q_1^{\,\tau}) - D( \bar q_2, \hat q_2^{\,\tau}) \big)
	\,+\,8\theta V L\,+\,2\eta V^2 L.
	\end{array}
	\end{equation}
	We note that $ |\langle{ \bar q_1\cdot \hat q_2^{\,\tau}-\hat q_1^{\,\tau}\cdot \bar q_2 },{r^{\tau}}\rangle|\leq L$. By summing both sides of~\eqref{eq.pushback_simplify_slater} from $\tau=t+1$ to $\tau=t+t_0$,
	\[
	\begin{array}{rcl}
	(\lambda^{t+t_0})^2 \,-\, (\lambda^{t})^2
	& \leq & \displaystyle 2t_0VL
	\,+\,\sum_{\tau\,=\,t}^{t+t_0-1} 
	2\lambda^{\tau} \big(\langle{ \bar q_1},{g^{\tau}}\rangle 
	+ \langle{ \bar q_2},{h^{\tau}}\rangle -b\big)	\,+\,4t_0L^2
	\\[0.2cm]
	&& \,+\, 2\eta^{-1} \big(D(\bar q_1,\tilde q_1^{\,t}) + D( \bar q_2, \tilde q_2^{\,t}) \big)
	\,+\,8t_0\theta V L\,+\,2t_0\eta V^2 L
	\end{array}
	\]
	where we omit two non-positive terms.
	Taking the conditional expectation given $\calF^{t}$ and $\mathcal{E}_{\text{good}}$ yields,
	\begin{equation}\label{eq.pushback_slater}
	\begin{array}{rcl}
	&& \!\!\!\! \!\!\!\! \!\! 
	\mathbb{E} \sbr{ (\lambda^{t+t_0})^2 \,-\, (\lambda^{t})^2 \,\vert\, \calF^{t}, \mathcal{E}_{\text{good}} }
	\\[0.2cm]
	& \leq & \displaystyle 2t_0VL
	\,+\,\sum_{\tau\,=\,t}^{t+t_0-1} 
	2 \mathbb{E} \sbr{ \lambda^{\tau} \big(\langle{ \bar q_1},{g^{\tau}}\rangle 
		+ \langle{ \bar q_2},{h^{\tau}}\rangle -b\big) \,\vert\, \calF^{t}, \mathcal{E}_{\text{good}}  }	\,+\,4t_0L^2
	\\[0.2cm]
	&& \,+\, 2\eta^{-1} \mathbb{E} \sbr{ D(\bar q_1,\tilde q_1^{\,t}) + D( \bar q_2, \tilde q_2^{\,t}) \,\vert\, \calF^{t}, \mathcal{E}_{\text{good}}  }
	\,+\,8t_0\theta V L\,+\,2t_0\eta V^2 L
	\\[0.2cm]
	& \leq & \displaystyle 2t_0 VL
	\,-\,2\xi\sum_{\tau\,=\,t}^{t+t_0-1} 
	\mathbb{E}\sbr{ \lambda^{\tau}  \,\vert\, \calF^{t}, \mathcal{E}_{\text{good}} }
	\,+\,4t_0L^2
	\\[0.2cm]
	&& \,+\, 2\eta^{-1} L(\log(|X||A|/\theta)+\log(|Y||B|/\theta))
	\,+\,8t_0\theta V L\,+\,2t_0\eta V^2 L
	\\[0.2cm]
	& \leq & \displaystyle 2t_0 VL
	\,-\,2\xi t_0 \mathbb{E}\sbr{  \lambda^t \,\vert\, \calF^{t}, \mathcal{E}_{\text{good}}  } 
	\,+\,2 \xi t_0(t_0-1)L
	\,+\,4t_0L^2
	\\[0.2cm]
	&& \,+\, 2\eta^{-1} L(\log(|X||A|/\theta)+\log(|Y||B|/\theta))
	\,+\,8t_0\theta V L\,+\,2t_0\eta V^2 L
	\end{array}
	\end{equation}
	where the second inequality is due to Lemma~\ref{lem.D_difference} and the fact: by the law of total expectation, for any $\tau\geq t$, $\calF^{t}\subset\calF^\tau$ and
	\[
	\begin{array}{rcl}
	\mathbb{E} \sbr{ \lambda^{\tau} \big(\langle{ \bar q_1},{g^{\tau}}\rangle 
		+ \langle{ \bar q_2},{h^{\tau}}\rangle -b\big) \,\vert\, \calF^{t}, \mathcal{E}_{\text{good}}  }
	&=&
	\mathbb{E} \sbr{ \mathbb{E} \sbr{ \lambda^{\tau} \big(\langle{ \bar q_1},{g^{\tau}}\rangle 
			+ \langle{ \bar q_2},{h^{\tau}}\rangle -b\big) \,\vert\, \calF^{\tau} } \,\vert\, \calF^{t}, \mathcal{E}_{\text{good}} }
	\\[0.2cm]
	&=&
	\mathbb{E} \sbr{ \lambda^{\tau}  \mathbb{E} \sbr{ \langle{ \bar q_1},{g^{\tau}}\rangle 
			+ \langle{ \bar q_2},{h^{\tau}}\rangle -b} \,\vert\, \calF^{t}, \mathcal{E}_{\text{good}}  }
	\\[0.2cm]
	&=&
	\mathbb{E} \sbr{ \langle{ \bar q_1},{g^{\tau}}\rangle 
		+ \langle{ \bar q_2},{h^{\tau}}\rangle -b } 
	\mathbb{E} \sbr{ \lambda^{\tau}  \,\vert\, \calF^{t}, \mathcal{E}_{\text{good}} }
	\\[0.2cm]
	&\leq & \,-\,\xi \,	\mathbb{E}\sbr{ \lambda^{\tau}  \,\vert\, \calF^{t}, \mathcal{E}_{\text{good}} }
	\end{array}
	\]
	where the inequality is due to the strict feasibility assumption on $(\bar q_1,\bar q_2)$; the last inequality is due to that 
	\[
	\sum_{\tau\,=\,t}^{t+t_0-1} 
	\mathbb{E}\sbr{ \lambda^{\tau}  \,\vert\, \calF^{t}, \mathcal{E}_{\text{good}} } 
	\;\geq\;
	\sum_{\tau\,=\,t}^{t+t_0-1} 
	\mathbb{E}\sbr{  \lambda^t - 2(\tau-t)L \,\vert\, \calF^{t}, \mathcal{E}_{\text{good}} } 
	\;=\;
	\sum_{\tau\,=\,0}^{t_0-1} 
	\mathbb{E}\sbr{  \lambda^t - 2\tau L \,\vert\, \calF^{t}, \mathcal{E}_{\text{good}} } 
	\]
	which follows the fact $\lambda^\tau\geq \lambda^t - 2(\tau-t)L$ for any $\tau\geq t\geq 0$ if we note that $|\lambda^{t+1}  - \lambda^{t} | \leq 2L$.
	Hence, we can simplify~\eqref{eq.pushback_slater} as
	\[
	\begin{array}{rcl}
	&& \!\!\!\! \!\!\!\! \!\! \mathbb{E} \sbr{ (\lambda^{t+t_0})^2 \,\vert\, \calF^{t}, \mathcal{E}_{\text{good}} }
	\\[0.2cm]
	& \leq & \displaystyle \mathbb{E} \sbr{ (\lambda^{t})^2\,\vert\, \calF^{t}, \mathcal{E}_{\text{good}} }
	\,-\,2\xi t_0 \mathbb{E} \sbr{\lambda^t \,\vert\, \calF^{t}, \mathcal{E}_{\text{good}} }
	\,+\,2 \xi t_0^2L
	\,+\,4t_0L^2
	\,+\, 2t_0 VL
	\\[0.2cm]
	&& \,+\, 2\eta^{-1} L(\log(|X||A|/\theta)+\log(|Y||B|/\theta))
	\,+\,8t_0\theta V L\,+\,2t_0\eta V^2 L
	\\[0.2cm]
	& \leq & \displaystyle \mathbb{E} \sbr{(\lambda^{t})^2\,\vert\, \calF^{t}, \mathcal{E}_{\text{good}} }
	\,-\,\xi t_0 \mathbb{E} \sbr{\lambda^t \,\vert\, \calF^{t}, \mathcal{E}_{\text{good}} }
	\,-\,\xi t_0\Theta
	\,+\,2 \xi t_0^2L
	\,+\,4t_0L^2
	\,+\, 2t_0 VL
	\\[0.2cm]
	&& \,+\, 2\eta^{-1} L(\log(|X||A|/\theta)+\log(|Y||B|/\theta))
	\,+\,8t_0\theta V L\,+\,2t_0\eta V^2 L
	\\[0.2cm]
	& = & \displaystyle \mathbb{E} \sbr{ (\lambda^{t})^2\,\vert\, \calF^{t}, \mathcal{E}_{\text{good}} }
	\,-\,\xi t_0 \mathbb{E} \sbr{\lambda^t\,\vert\, \calF^{t}, \mathcal{E}_{\text{good}} } \,-\,\frac{1}{2} \xi^2 t_0^2
	\\[0.2cm]
	& \leq & \displaystyle
	\rbr{ \mathbb{E} \sbr{\lambda^{t}\,\vert\, \calF^{t}, \mathcal{E}_{\text{good}} }- \frac{1}{2}\xi t_0 }^2
	\end{array}
	\]
	where we apply $\lambda^t\geq \Theta$ for the second inequality and we take $\Theta$ in Lemma~\ref{lem.drift},
	\[
	\Theta \;=\;
	\frac{1}{2}\xi t_0 +2 t_0L
	\,+\,\frac{4L^2 +8\theta V L+2\eta V^2 L
		+2 VL}{\xi}\,+\, \frac{2L(\log(|X||A|/\theta)+\log(|Y||B|/\theta)) }{t_0\xi\eta}.
	\]
	Taking the square root and applying the Jensen's inequality yield
	\[
	\mathbb{E} \sbr{ \lambda^{t+t_0} \,\vert\, \calF^{t}, \mathcal{E}_{\text{good}} }
	\;\leq\;
	\sqrt{\mathbb{E} \sbr{ (\lambda^{t+t_0})^2 \,\vert\, \calF^{t}, \mathcal{E}_{\text{good}} } }
	\;\leq\; 
	\mathbb{E} \sbr{\lambda^{t} \,\vert\, \calF^{t}, \mathcal{E}_{\text{good}} }- \frac{1}{2}\xi t_0
	\]
	which shows that $\mathbb{E} \sbr{ \lambda^{t+t_0} -\lambda^{t}\,\vert\, \calF^{t}, \mathcal{E}_{\text{good}} } \leq - \frac{1}{2}\xi t_0$. 
	Application of law of total expectation to this inequality and~\eqref{eq.lambda.diff} with $\delta < \frac{1}{12}$ yields
	\[
	\begin{array}{rcl}
	\mathbb{E} \sbr{ \lambda^{t+t_0} -\lambda^{t}\,\vert\, \calF^{t} } 
	& = &
	P(\mathcal{E}_{\text{good}} )
	\mathbb{E} \sbr{ \lambda^{t+t_0} -\lambda^{t}\,\vert\, \calF^{t}, \mathcal{E}_{\text{good}} }
	\,+\,
	P(\bar{\mathcal{E}}_{\text{good}})
	\mathbb{E} \sbr{ \lambda^{t+t_0} -\lambda^{t}\,\vert\, \calF^{t}, \bar{\mathcal{E}}_{\text{good}} }
	\\[0.2cm]
	& \leq & - \dfrac{1}{2}\xi t_0 \times(1-\delta) \,+\,  2t_0L\times \delta
	\\[0.2cm]
	& \leq & - \dfrac{1}{4}\xi t_0
	\end{array}
	\]
	which verifies the assumption of Lemma~\ref{lem.drift} if we take $\zeta=\xi/4$.
	
	We now have verified all assumptions of Lemma~\ref{lem.drift} with appropriate parameters $\Theta$, $\delta_{\max}$, $\zeta$. For episode $t$, with probability $1-\delta$ it holds that 
	\[
	\begin{array}{rcl}
	\lambda^t & \leq & \displaystyle \Theta + t_0 \delta_{\max} + t_0 \frac{4\delta_{\max}^2}{\zeta} \log \rbr{\frac{8\delta_{\max}^2}{\zeta}} + t_0 \frac{4\delta_{\max}^2}{\zeta} \log\frac{1}{\delta}.
	\end{array}
	\]
	We complete the proof by taking a union bound over $t = 1,\cdots,T$.
\end{proof}

With Lemma~\ref{lem.lambda} in place, we are ready to prove Lemma~\ref{lem.violation}.

\begin{proof}[Proof of Lemma~\ref{lem.violation}]
	
	Let $Z^t \DefinedAs \sum_{\tau\,=\,0}^{t-1}
	\lambda^{\tau} \big(\langle{ q_1^\star},{g^{\tau}}\rangle 
	+ \langle{ q_2^\star},{h^{\tau}}\rangle -b\big)$. We note that 
	\[
	\begin{array}{rcl}
	&& \!\!\!\! \!\!\!\! \!\!
	\mathbb{E} \sbr{Z^t \,\vert\,\calF^{t-1}} 
	\\[0.2cm]
	&=& \displaystyle
	\mathbb{E} \sbr{ \sum_{\tau\,=\,0}^{t-1}
		\lambda^{\tau} \big(\langle{ q_1^\star},{g^{\tau}}\rangle 
		+ \langle{ q_2^\star},{h^{\tau}}\rangle -b\big)\,\bigg\vert\,\calF^{t-1}}
	\\[0.2cm]
	&=&  \displaystyle
	\mathbb{E} \sbr{ \sum_{\tau\,=\,0}^{t-2}
		\lambda^{\tau} \big(\langle{ q_1^\star},{g^{\tau}}\rangle 
		+ \langle{ q_2^\star},{h^{\tau}}\rangle -b\big)\,\bigg\vert\,\calF^{t-1}}
	\,+\, 
	\lambda^{t-1} \mathbb{E} \sbr{\big(\langle{ q_1^\star},{g^{t-1}}\rangle 
		+ \langle{ q_2^\star},{h^{t-1}}\rangle -b\big)\,\vert\,\calF^{t-1}}
	\\[0.2cm]
	&\leq&  \displaystyle
	\mathbb{E} \sbr{ \sum_{\tau\,=\,0}^{t-2}
		\lambda^{\tau} \big(\langle{ q_1^\star},{g^{\tau}}\rangle 
		+ \langle{ q_2^\star},{h^{\tau}}\rangle -b\big)\,\bigg\vert\,\calF^{t-1}}
	\\[0.4cm]
	&=&
	\mathbb{E} \sbr{Z^{t-1}} 
	\end{array}
	\]
	where the inequality is because of $\mathbb{E} \sbr{\big(\langle{ q_1^\star},{g^{t-1}}\rangle 
		+ \langle{ q_2^\star},{h^{t-1}}\rangle -b\big)\,\vert\,\calF^{t-1}} = \langle{ q_1^\star},{g}\rangle 
	+ \langle{ q_2^\star},{h}\rangle -b\leq 0$. Hence,  $\{Z^t,t\geq0 \}$ a supermartingale. 
	
	We also note that
	$\abr{Z^{t+1}-Z^{t}} 
	= \lambda^{t} \abr{\langle{ q_1^\star},{g^{t}}\rangle 
		+ \langle{ q_2^\star},{h^{t}}\rangle -b}
	\leq 2\lambda^{t} L$.
	Thus, if $|Z^{t+1}-Z^t|> c$ for some $c\in\mathbb{R}^+$, then $\lambda^t> c/(2L)$. Let $Y^t \DefinedAs \lambda^t - c/(2L)$. Therefore,
	\[
	\{ |Z^{t+1}-Z^t|> c \}  \;\subset\;  \{ Y^t > 0 \}.
	\]
	By Lemma~\ref{lem.azuma_general},
	\begin{equation}\label{eq.azuma_dual}
	P \rbr{ \sum_{t\,=\,0}^{T-1}
		\lambda^{t} \big(\langle{ q_1^\star},{g^{t}}\rangle 
		+ \langle{ q_2^\star},{h^{t}}\rangle -b\big)\geq z } \;\leq\; {\rm e}^{-z^2/ (2c^2T)} \,+\, \sum_{\tau\,=\,0}^{T-1} P \rbr{\lambda^t>\frac{c}{2L}}.
	\end{equation}
	By Lemma~\ref{lem.lambda}, with probability $1-\delta$ it holds for any $t$ that
	\[
	\lambda^t \; \leq\; \displaystyle \Theta + 2 t_0L+ t_0 \frac{64 L^2}{\xi} \log \rbr{\frac{128 L^2}{\xi}} + t_0 \frac{64 L^2}{\xi} \log\frac{1}{\delta}
	\]
	or, equivalently,
	\[
	P \rbr{
		\lambda^t \;\geq\; \displaystyle \Theta + 2 t_0L+ t_0 \frac{64 L^2}{\xi} \log \rbr{\frac{128 L^2}{\xi}} + t_0 \frac{64 L^2}{\xi} \log\frac{1}{\delta}
	} \;\leq\; \delta.
	\]
	If we take 
	\[
	c\;=\;  2\Theta L + 4 t_0 L^2+ t_0 \frac{128 L^3}{\xi} \log \rbr{\frac{128 L^2}{\xi}} + t_0 \frac{128 L^3}{\xi} \log\frac{1}{\delta}
	\;\text{ and }\;
	z\;=\;\sqrt{2T c^2 \log (1/(\delta T))}
	\]
	then~\eqref{eq.azuma_dual} becomes
	\[
	P \rbr{ \sum_{t\,=\,0}^{T-1}
		\lambda^{t} \big(\langle{ q_1^\star},{g^{t}}\rangle 
		+ \langle{ q_2^\star},{h^{t}}\rangle -b\big)\geq z } \;\leq\; 2\delta  T
	\]
	which proves the desired result.
\end{proof}

\section{Proof of Theorem~\ref{thm.violation_hat}}
\label{ap.violation_hat}

By the dual update~\eqref{eq.dual},
\begin{equation}\label{eq.lambda_lb}
\begin{array}{rcl}
\lambda^{t} &=& \max\! \Big( \lambda^{t-1} \,+\, \big(\inner{\hat q_1^{\,t}}{g^{t-1}} +\inner{\hat q_2^{\,t}}{h^{t-1}}-b\, \big),\; 0 \Big)
\\[0.2cm]
&\geq &  \lambda^{t-1} \,+\, \big(\inner{\hat q_1^{\,t}}{g^{t-1}} +\inner{\hat q_2^{\,t}}{h^{t-1}}-b\, \big)
\\[0.2cm]
&= &  \lambda^{t-1} \,+\, \big(\inner{\hat q_1^{\,t-1}}{g^{t-1}} +\inner{\hat q_2^{\,t-1}}{h^{t-1}}-b\, \big)\,+\, \inner{\hat q_1^{\,t} - \hat q_1^{\,t-1}}{g^{t-1}} \,+\,\inner{\hat q_2^{\,t}-\hat q_2^{\,t-1}}{h^{t-1}}
\\[0.2cm]
&\geq &  \lambda^{t-1} \,+\, \big(\inner{\hat q_1^{\,t-1}}{g^{t-1}} +\inner{\hat q_2^{\,t-1}}{h^{t-1}}-b\, \big)\,-\, \norm{\hat q_1^{\,t} - \hat q_1^{\,t-1}}_1 \,-\,\norm{\hat q_2^{\,t}-\hat q_2^{\,t-1}}_1
\end{array}
\end{equation}
where the last inequality is due to: $\langle{\hat q_1^{\,t} - \hat q_1^{\,t-1}},{g^{t-1}}\rangle\leq \Vert \hat q_1^{\,t} - \hat q_1^{\,t-1}\Vert_1 \Vert g^{t-1}\Vert_\infty$, $\langle{\hat q_2^{\,t} - \hat q_2^{\,t-1}},{h^{t-1}}\rangle\leq \Vert \hat q_2^{\,t} - \hat q_2^{\,t-1}\Vert_1 \Vert h^{t-1}\Vert_\infty$, and $\Vert g^{t-1}\Vert_\infty$, $\Vert h^{t-1}\Vert_\infty\in [0,1]$. We note that $\lambda^0=0$ from the initialization. Summing up both sides of~\eqref{eq.lambda_lb} from $t=1$ to $t=T$ leads to
\begin{equation}\label{eq.violation_dqq}
\sum_{t\,=\,0}^{T-1} \big(\inner{\hat q_1^{\,t}}{g^{t}} +\inner{\hat q_2^{\,t}}{h^{t}}-b\, \big)
\;\leq\; 
\lambda^T
\,+\, 
\sum_{t\,=\,1}^{T} \rbr{
	\norm{\hat q_1^{\,t} - \hat q_1^{\,t-1}}_1 
	+\norm{\hat q_2^{\,t}-\hat q_2^{\,t-1}}_1}.
\end{equation}
We recall $ \hat q_1^{\,t}\in \Delta(k_1^t)$, $ \hat q_2^{\,t}\in \Delta(k_2^t)$ in the primal update~\eqref{eq.primal} and $\Delta(k_1^t)$ and $\Delta(k_2^t)$ in the confidence sets~\eqref{eq.confidence}.
To bound $\norm{\hat q_1^{\,t} - \hat q_1^{\,t-1}}_1 
+\norm{\hat q_2^{\,t}-\hat q_2^{\,t-1}}_1$, we consider two cases: (i) $k_1^t = k_1^{t-1}$ and $k_2^t = k_2^{t-1}$; (ii) either $k_1^t \neq k_1^{t-1}$ or $k_2^t \neq k_2^{t-1}$.

\noindent\textbf{Case~(i)}.	In this case, we have: $ \hat q_1^{\,t}$, $\hat q_1^{\,t-1}\in \Delta(k_1^t)$, $ \hat q_2^{\,t}$, $\hat q_2^{\,t-1}\in \Delta(k_2^t)$. 
We begin with the primal update~\eqref{eq.primal} and apply Lemma~\ref{lem.pushback} with, 
\[
f(x,y) \vert_{x\,=\,q_1,\, y\,=\,q_2} 
\;=\;
V\, \big\langle{q_1\cdot \hat q_2^{\,t-1}+\hat q_1^{\,t-1}\cdot q_2},{r^{t-1}}\big\rangle 
\,+\,
\lambda^{t-1} \langle{q_1},{g^{t-1}}\rangle 
\,-\,\lambda^{t-1} \langle{q_2},{h^{t-1}}\rangle 
\]
and $x^\star = \hat q_1^{\,t}$, $y^\star= \hat q_2^{\,t}$, $x' = \tilde q_1^{\,t-1}$, $y'=\tilde q_2^{\,t-1}$, $x = \tilde q_1^{\,t-1}$, and $y=\tilde q_2^{\,t-1}$. Thus, 
\[
\begin{array}{rcl}
&& \!\!\!\! \!\!\!\! \!\!
V\, \big\langle{\hat q_1^{\,t}\cdot \hat q_2^{\,t-1}+\hat q_1^{\,t-1}\cdot  \tilde q_2^{\,t-1}},{r^{t-1}}\big\rangle 
\,+\,
\lambda^{t-1} \langle{\hat q_1^{\,t}},{g^{t-1}}\rangle 
\,-\,\lambda^{t-1} \langle{ \tilde q_2^{\,t-1}},{h^{t-1}}\rangle 
\\[0.2cm]
&& 
\,+\, \eta^{-1} \big(D(\hat q_1^{\,t},\tilde q_1^{\,t-1}) + D(\hat q_2^{\,t}, \tilde q_2^{\,t-1})\big)
\\[0.2cm]
&\leq& V\, \big\langle{ \tilde q_1^{\,t-1}\cdot \hat q_2^{\,t-1}+\hat q_1^{\,t-1}\cdot \hat q_2^{\,t}},{r^{t-1}}\big\rangle 
\,+\,
\lambda^{t-1} \langle{ \tilde q_1^{\,t-1}},{g^{t-1}}\rangle 
\,-\,\lambda^{t-1} \langle{ \hat q_2^{\,t}},{h^{t-1}}\rangle 
\\[0.2cm]
&&
\,-\, \eta^{-1} \big(D( \tilde q_1^{\,t-1},\hat q_1^{\,t}) + D( \tilde q_2^{\,t-1}, \hat q_2^{\,t}) \big).
\end{array}
\]
or, equivalently,
\begin{equation}\label{eq.pushback_I}
\begin{array}{rcl}
&& \!\!\!\! \!\!\!\!  \!\! 
\eta^{-1} \big(D(\hat q_1^{\,t},\tilde q_1^{\,t-1}) + D(\hat q_2^{\,t}, \tilde q_2^{\,t-1})\big)
\,+\,
\eta^{-1} \big(D( \tilde q_1^{\,t-1},\hat q_1^{\,t}) + D( \tilde q_2^{\,t-1}, \hat q_2^{\,t}) \big)
\\[0.2cm]
&\leq&  V\, \big\langle{ (\tilde q_1^{\,t-1} -\hat q_1^{\,t} )\cdot \hat q_2^{\,t-1}+\hat q_1^{\,t-1}\cdot (\hat q_2^{\,t} - \tilde q_2^{\,t-1} )},{r^{t-1}}\big\rangle 
\\[0.2cm]
&& 
\,+\,
\lambda^{t-1} \langle{ \tilde q_1^{\,t-1} - \hat q_1^{\,t} },{g^{t-1}}\rangle 
\,+\,\lambda^{t-1} \langle{ \tilde q_2^{\,t-1} - \hat q_2^{\,t}},{h^{t-1}}\rangle.
\end{array}
\end{equation}
We note that $\langle{ (\tilde q_1^{\,t-1} -\hat q_1^{\,t} )\cdot \hat q_2^{\,t-1}},{r^{t-1}}\rangle \leq \Vert (\tilde q_1^{\,t-1} -\hat q_1^{\,t} )\cdot \hat q_2^{\,t-1}\Vert_1\Vert{r^{t-1}}\Vert_\infty\leq \Vert\tilde q_1^{\,t-1} -\hat q_1^{\,t} \Vert_1$, and, similarly, $\langle\hat q_1^{\,t-1}\cdot (\hat q_2^{\,t} - \tilde q_2^{\,t-1} ),{r^{t-1}}\rangle\leq\Vert\hat q_2^{\,t} - \tilde q_2^{\,t-1}\Vert_1$. Thus, we can reduce~\eqref{eq.pushback_I} into
\[
\begin{array}{rcl}
&& \!\!\!\! \!\!\!\! \!\! 
\eta^{-1} \big(D(\hat q_1^{\,t},\tilde q_1^{\,t-1}) + D(\hat q_2^{\,t}, \tilde q_2^{\,t-1})\big)
\,+\,
\eta^{-1} \big(D( \tilde q_1^{\,t-1},\hat q_1^{\,t}) + D( \tilde q_2^{\,t-1}, \hat q_2^{\,t}) \big)
\\[0.2cm]
&\leq&
(V+\lambda^{t-1}) 
\rbr{
	\norm{ \tilde q_1^{\,t-1} - \hat q_1^{\,t} }_1+
	\norm{ \tilde q_2^{\,t-1} - \hat q_2^{\,t}}_1
}
\end{array}
\]
where the left-hand side can be lower bounded by Lemma~\ref{lem.D_lb}, 
\[
D(\hat q_1^{\,t},\tilde q_1^{\,t-1})+D( \tilde q_1^{\,t-1},\hat q_1^{\,t}) \;\geq\; L^{-1}\norm{ \tilde q_1^{\,t-1}-\hat q_1^{\,t} }_1^2
\]
\[
D(\hat q_2^{\,t},\tilde q_2^{\,t-1})+D( \tilde q_2^{\,t-1},\hat q_2^{\,t}) \;\geq\; L^{-1}\norm{ \tilde q_2^{\,t-1}-\hat q_2^{\,t} }_1^2.
\]
Then, we apply the inequality $(x+y)^2\leq2(x^2+y^2) $ and cancel a non-negative term to obtain
\begin{equation}\label{eq.qq_dual}
\norm{ \tilde q_1^{\,t-1}-\hat q_1^{\,t} }_1 + \norm{ \tilde q_2^{\,t-1}-\hat q_2^{\,t} }_1
\;\leq\;
2\eta  L (V+\lambda^{t-1}).
\end{equation}
By the definition of $\tilde q_1^{\,t-1}$ and $\tilde q_2^{\,t-1}$,
\[
\begin{array}{rcl}
\norm{ \tilde q_1^{\,t-1}-\hat q_1^{\,t} }_1 
&=&\displaystyle
\sum_{\ell\,=\,0}^{L-1}\sum_{x\,\in\,X_\ell} \sum_{a\,\in\,A}
\abr{
	(1-\theta)\hat {q}_1^{\,t-1}(x,a)+\theta\frac{1}{|X_\ell||A|} - \hat q_1^{\,t}(x,a)
}
\\[0.2cm]
&\geq&\displaystyle
\sum_{\ell\,=\,0}^{L-1}\sum_{x\,\in\,X_\ell} \sum_{a\,\in\,A}
\rbr{
	(1-\theta)  \abr{
		\hat {q}_1^{\,t-1}(x,a)- \hat q_1^{\,t}(x,a)
	} -\theta \rbr{\frac{1}{|X_\ell||A|} + \hat q_1^{\,t}(x,a)}
}
\\[0.2cm]
&=& (1-\theta) \norm{\hat {q}_1^{\,t-1}-\hat q_1^{\,t}}_1 -2\theta L.
\end{array}
\]
Similarly, we have $\Vert{ \tilde q_2^{\,t-1}-\hat q_2^{\,t} }\Vert_1 \leq (1-\theta) \Vert{\hat {q}_2^{\,t-1}-\hat q_2^{\,t}}\Vert_1 -2\theta L$. Thus, we can further reduce~\eqref{eq.qq_dual} into
\begin{equation}\label{eq.case(i)}
\norm{\hat {q}_1^{\,t-1}-\hat q_1^{\,t}}_1 +\Vert{\hat {q}_2^{\,t-1}-\hat q_2^{\,t}}\Vert_1 
\;\leq\;
2\eta (1-\theta)^{-1} L (V+\lambda^{t-1}) + 4 \theta (1-\theta)^{-1}  L.
\end{equation}

\noindent\textbf{Case~(ii)}. In this case, either $\hat q_1^{\,t}$, $\hat q_1^{\,t-1}$ or $ \hat q_2^{\,t}$, $\hat q_2^{\,t-1}$ might not have the same domain. For instance, when $k_1^t>k_1^{t-1}$, it is possible that $\Delta(k_1^{t})$ becomes different from $\Delta(k_1^{t-1})$. We note that $k_1^t>k_1^{t-1}$ only happens when episode $t$ is the first one that belongs to epoch $k_1^t$. By Lemma~\ref{lem.epoch}, $k_1^T \leq \sqrt{ T|X||A|}\log (8T/(|X||A|))$ and  $k_2^T \leq \sqrt{ T|Y||B|}\log (8T/(|Y||B|))$ if we are given $T\geq \max(|X||A|, |Y||B|)$.

We now combine two cases above for~\eqref{eq.violation_dqq},
\[
\begin{array}{rcl}
&& \!\!\!\! \!\!\!\! \!\!\!\!
\displaystyle
\sum_{t\,=\,1}^{T} \rbr{
	\norm{\hat q_1^{\,t} - \hat q_1^{\,t-1}}_1 
	+\norm{\hat q_2^{\,t}-\hat q_2^{\,t-1}}_1}
\\[0.2cm]
&=& \displaystyle
\sum_{\substack{1\,\leq\,t\,\leq\,T\\k_1^t \,=\, k_1^{k-1}\,\wedge\, k_2^t \,=\, k_2^{k-1}}}\rbr{
	\norm{\hat q_1^{\,t} - \hat q_1^{\,t-1}}_1 
	+\norm{\hat q_2^{\,t}-\hat q_2^{\,t-1}}_1}
\\[0.2cm]
&&\displaystyle
\,+\,
\sum_{\substack{1\,\leq\,t\,\leq\,T\\k_1^t \,=\, k_1^{k-1}\,\vee\, k_2^t \,=\, k_2^{k-1}}}\rbr{
	\norm{\hat q_1^{\,t} - \hat q_1^{\,t-1}}_1 
	+\norm{\hat q_2^{\,t}-\hat q_2^{\,t-1}}_1}
\\[0.2cm]
&\leq& \displaystyle
\sum_{\substack{1\,\leq\,t\,\leq\,T\\k_1^t \,=\, k_1^{k-1}\,\wedge\, k_2^t \,=\, k_2^{k-1}}}\rbr{
	\norm{\hat q_1^{\,t} - \hat q_1^{\,t-1}}_1 
	+\norm{\hat q_2^{\,t}-\hat q_2^{\,t-1}}_1}
\,+\,2L (k_1^T+k_2^T)
\\[0.2cm]
&\leq&\displaystyle
2\eta (1-\theta)^{-1} L\sum_{t\,=\,1}^{T} (V+\lambda^{t-1}) + 4 \theta (1-\theta)^{-1}  LT +2L (k_1^T+k_2^T)
\end{array}
\]
where the first inequality is due to: $\norm{\hat q_1^{\,t} - \hat q_1^{\,t-1}}_1 \leq 2L$ and $\norm{\hat q_2^{\,t}-\hat q_2^{\,t-1}}_1\leq 2L$, and we apply~\eqref{eq.case(i)} from the case (i) for the last inequality. Using the bounds on $k_1^T$, $k_2^T$ in the case (ii), we conclude the desired bound for~\eqref{eq.violation_dqq},
\[
\begin{array}{rcl}
&& \!\!\!\! \!\!\!\! \!\!
\displaystyle\sum_{t\,=\,0}^{T-1} \big(\inner{\hat q_1^{\,t}}{g^{t}} +\inner{\hat q_2^{\,t}}{h^{t}}-b\, \big)
\\[0.2cm]
&\leq&\displaystyle
\lambda^T
\,+\, 
\frac{2\eta L}{1-\theta}
\sum_{t\,=\,1}^{T} \lambda^{t-1}+ \frac{2\eta V+4 \theta}{1-\theta}  LT 
\\[0.2cm]
&&\displaystyle+2L 
\rbr{\sqrt{ T|X||A|}\log (8T/(|X||A|))+ \sqrt{ T|Y||B|}\log (8T/(|Y||B|))}.
\end{array}
\]
We complete the proof by noting $\lambda^0=0$, $V=L\sqrt{T}$, $\eta = 1/(TL)$, and $\theta = 1/T$. 

\section{Constrained MGs with Side Constraints}
\label{ap.CMG_sc}

In this section, we present a special case of Problem~\eqref{eq.opt_wc} that is described as a zero-sum MG with side constraint~\citep{singh2014characterization}. 
Having defined episodic MDPs and occupancy measures in Section~\ref{prelim}, we can formulate a constrained minimax problem in which the objective function is a sum of the expected total rewards over $T$ episodes and the constraint is on two agent' expected total utilities, 
\begin{equation}\label{eq.opt_wc_s}
\begin{array}{rcl}
\minimize\limits_{q_1\,\in\,\Delta(P_1)} \,\,\maximize\limits_{q_2\,\in\,\Delta(P_2)} &&\!\!\!\! \displaystyle \sum_{t\,=\,0}^{T-1} \inner{q_1\cdot q_2}{r^t} 
\\[0.4cm]
\subject &&\!\!\!\!
\inner{q_1}{g} \;\leq\; b_1
\; \text{ and } \;
\inner{q_2}{h}\;\leq\; b_2
\end{array}
\end{equation}
where we take $b_1$, $b_2\in (0, L]$ to avoid trivial cases since we note that $\langle{q_1},{g}\rangle$, $\langle{q_2},{h}\rangle\in [0,L]$. The side constraint corresponds to the limited use of budget/resource for each player. It is straightforward to generalize it to account for multiple constraints. When the transitions $P_1$ and $P_2$ are known, the occupancy measure sets $\Delta(P_1)$ and $\Delta(P_2)$ define convex polytopes on $q_1$ and $q_2$. 

Let $(q_1^\star,q_2^\star)$ be a solution to Problem~\eqref{eq.opt_wc_s} in hindsight. The existence of $(q_1^\star,q_2^\star)$ is well-known under compactness of the constraint sets~\citep{neumann1928theorie,rosen1965existence}. Since two constraints are decoupled, it is natural to define the usual Nash equilibrium via two conditions~\citep{altman2000constrained,daskalakis2020complexity}: (i) $ \sum_{t\,=\,0}^{T-1}\langle q_1^\star\cdot q_2^\star, r^t \rangle \leq \sum_{t\,=\,0}^{T-1} \langle q_1\cdot q_2^\star, r^t \rangle$ for any $q_1 \in \Delta(P_1)$ satisfying $\langle{q_1},{g}\rangle \leq b_1$; (ii) $\sum_{t\,=\,0}^{T-1}\langle q_1^\star\cdot q_2, r^t \rangle \leq \sum_{t\,=\,0}^{T-1}\langle q_1^\star\cdot q_2^\star, r^t \rangle$ for any $q_2\in\Delta(P_2)$ satisfying $\inner{q_2}{h}\leq b_2$. 
With this solution concept, we define the regret for any algorithm that plays the game for $T$ episodes by 
\begin{equation}\label{eq.regret_s}
\text{Regret}(T) \;=\; \sum_{t\,=\,0}^{T-1} \rbr{ \inner{q_1^t \cdot q_2^\star}{r^t} - \inner{q_1^\star \cdot q_2^t}{r^t} }
\end{equation}
where two players take policies $\pi^t$ and $\mu^t$ in episode $t$ and  they define occupancy measures $q_1^t$ and $q_2^t$ under the true transitions $P_1$ and $P_2$. 

To measure the constraint satisfaction, we introduce the violation as a non-negative part of accumulated constraint violations $\langle{q_1^t},{g}\rangle-b_1$ and $\langle{q_2^t},{h}\rangle - b_2$ over $T$ episodes,
\begin{equation}\label{eq.violation_s}
\!\!\!\!
\text{Violation}_1(T) \,=\, \sbr{\sum_{t\,=\,0}^{T-1} \rbr{\inner{q_1^t}{g^t}-b_1}}_+
\, \text{ and } \,
\text{Violation}_2(T) \,=\, \sbr{\sum_{t\,=\,0}^{T-1} \rbr{\inner{q_2^t}{h^t}-b_2}}_+.
\end{equation}

We next make an assumption that guarantees the existence of constrained Nash equilibrium~\citep{altman2000constrained}.

\begin{assumption}[Feasibility]\label{as.feasibility_s}
	There exists a joint policy $(\bar \pi,\bar\mu)$ associated to the occupancy measure $(\bar q_1,\bar q_2)$ and $\xi>0$ such that $\inner{\bar q_1}{g}+\xi\leq b_1$ and $\inner{\bar q_2}{h}+\xi\leq b_1$.
\end{assumption}

\subsection{Algorithm and Performance Guarantees}
We now are ready to specialize Algorithm~\ref{UCB-MPD} to Problem~\eqref{eq.opt_wc_s}. The only change is to replace the primal-dual update~\eqref{eq.primal} and~\eqref{eq.dual} by the following optimistic primal-dual mirror descent step. 

Let us recall that the occupancy measures $q_1^t$ for the min-player and $q_2^t$ for the max-player are defined over the true transitions $P_1$ and $P_2$ in episode $t$. The primal update of our algorithm maintains two occupancy measures $\hat q_1^{\,t}$, $\hat q_2^{\,t}$ to estimate $q_1^t$, $q_2^t$, separately. Although $\hat q_1^{\,t}$, $\hat q_2^{\,t}$ do not necessarily come from the true transitions $P_1$, $P_2$, they propose a min-policy $\pi^t$ for the min-player and a max-policy $\mu^t$ for the max-player given by~\eqref{eq.policy}.

We can revise our Lagrangian-based design to update estimates $\hat q_1^{\,t}$ and $\hat q_2^{\,t}$ as follows. 
Assume that the transitions $P_1$ and $P_2$ are known. We consider a one-episode constrained minimax problem based on reward/utility functions: $r^{t-1}$, $g^{t-1}$, $h^{t-1}$, revealed at the end of episode $t-1$,
\[
\begin{array}{rcl}
\minimize\limits_{q_1\,\in\,\Delta(P_1)} \,\,\maximize\limits_{q_2\,\in\,\Delta(P_2)}
&& \!\!\!\! \inner{q_1\cdot q_2}{r^{t-1}}
\\[0.3cm]
\subject && \!\!\!\!
\inner{q_1}{g^{t-1}} \;\leq\; b_1
\; \text{ and }\;
\inner{q_2}{h^{t-1}} \;\leq\; b_2
\end{array}
\]
where $\Delta(P_1)$ and $\Delta(P_2)$ are sets of valid occupancy measures under $P_1$ and $P_2$, respectively. We apply the method of Lagrange multipliers~\citep{bertsekas2014constrained} to deal with constraints by formulating a generalized Lagrangian-based function,
\[
\begin{array}{rcl}
L^t(q_1,q_2;\lambda_1,\lambda_2) 
&\DefinedAs&
\langle{q_1\cdot q_2},{r^{t-1}}\rangle
\,+\,
\lambda_1 \big( \langle{q_1},{g^{t-1}}\rangle - b_1 \big)
\,-\,
\lambda_2 \big( \langle{ q_2},{h^{t-1}}\rangle- b_2 \big)
\end{array}
\]
where $q_1$ is the first primal variable for the min-player, $q_2$ is the second primal variable for the max-player, and $\lambda_1$, $\lambda_2\geq 0$ work as the Lagrange multiplier or the dual variable in penalizing the min-player/max-player via the first/second $\lambda$-term.  Once we update $\lambda_1=\lambda_1^{t-1}$ and $\lambda_2=\lambda_2^{t-1}$ from the last episode, we reach a constrained saddle-point problem, 
$$\minimize_{q_1\,\in\,\Delta(P_1)}\,\maximize_{q_2\,\in\,\Delta(P_2)}\;\; L^t(q_1,q_2;\lambda_1^{t-1},\lambda_2^{t-1}).$$

However, it is not feasible to take the domains $\Delta(P_1)$ and $\Delta(P_2)$ since the true transitions $P_1$ and $P_2$ are unknown. Instead, we use their optimistic estimates $\Delta(k_1^t)$ and $\Delta(k_2^t)$ in sense that $q_1^t\in \Delta(k_1^t)$ and $q_2^t\in \Delta(k_2^t)$ hold with high probability; see Lemma~\ref{lem.empiricalP}. Denote $\hat q^{\,t} \DefinedAs (\hat q_1^{\,t},\hat q_2^{\,t})$. By the linear approximation of $L^t(q_1,q_2;\lambda^{t-1})$ at the previous iterate $(q_1^{t-1},q_2^{t-1})$, we update the primal variable via an online mirror descent step over the optimistic domains of $q_1$ and $q_2$,
\begin{equation}\label{eq.primal_s}
\begin{array}{rcl}
&&  \!\!\!\! \!\!\!\! \!\!\!\! \!\!
\hat q^t 
\,\leftarrow\,
\displaystyle \argmin_{q_1\,\in\,\Delta(k_1^t)} \argmax_{q_2\,\in\,\Delta(k_2^t)} 
\Big(
V\, 
\big\langle{q_1\cdot \hat q_2^{\,t-1}+\hat q_1^{\,t-1}\cdot q_2},{r^{t-1}}\big\rangle 
\\[0.2cm]
&& \;\;\;\;  \;\;\;\;  \;\;\;\;  \;\;\;\;  \;\;\;\;  \;\;\;\;  \;\;\;\;  \;\;\;\;
\,+\,
\lambda_1^{t-1}\langle{q_1},{g^{t-1}}\rangle 
\, -\, 
\lambda_2^{t-1}\langle{q_2},{h^{t-1}}\rangle
\, +\, 
\displaystyle \eta^{-1} D\big(q\,\vert\,  \tilde{q}^{\,t-1}\big)\Big)
\end{array}
\end{equation}
where $V$, $\eta>0$ are some regularization parameters, $D(\cdot\,\vert\,\cdot)$ is the unnormalized Kullback-Leibler divergence with a slightly abuse in a way that $D(q\,\vert\,q') \DefinedAs D(q_1\,\vert\,q_1')-D(q_2\,\vert\,q_2')$, $\tilde{q}_1^{\,t-1}$ and $\tilde{q}_2^{\,t-1}$ are mixing policies given by~\eqref{eq.mixing}.
The unnormalized Kullback-Leibler (KL) divergence between two distributions $p$, $q$ is defined by $D(p\,\vert\,q) \DefinedAs \sum_{i} p_i \ln\frac{p_i}{q_i} - \sum_{i} (p_i-q_i)$. Moreover,~\eqref{eq.primal_s} has an efficient update that is similar as the one in Appendix~\ref{ap.implementation}. 

Once we obtain $\hat q^{\,t}$, we next perform the dual update. We treat two $\lambda$-related regularization terms in $L^t(\hat q_1^{\,t}, \hat q_2^{\,t};\lambda_1,\lambda_2)$, separately. The dual update works for each player in the usual way by adding up all past constraint violations,
\begin{equation}\label{eq.dual_s}
\lambda_1^{t} \;=\; \max \big( \lambda_1^{t-1} \,+\, (\langle{\hat q_1^{\,t}},{g^{t-1}}\rangle -b_1\, ),\, 0 \big)
\;\, \text{ and } \;\,
\lambda_2^{t} \;=\; \max \big( \lambda_2^{t-1} \,+\, (\langle{\hat q_2^{\,t}},{h^{t-1}}\rangle-b_2\, ),\, 0 \big).
\end{equation}
The dual update~\eqref{eq.dual_s} increases $\lambda_1^{t-1}$ when $\hat q_1^{\,t}$ violates the approximate constraint $\langle{q_1},{g^{t-1}}\rangle \leq b_1$; it is similar for $\lambda_2^{t-1}$. Once we replace the primal-dual update~\eqref{eq.primal} and~\eqref{eq.dual} in line~4 of Algorithm~\ref{UCB-MPD} by~\eqref{eq.primal_s} and~\eqref{eq.dual_s}, we obtain a new version of Algorithm~\ref{UCB-MPD} for Problem~\eqref{eq.opt_wc_s}. 

Similar to Theorem~\ref{thm.main}, we have the following bounds on the regret and the constraint violation.

\begin{theorem}[Regret Bound and Constraint Violation]\label{thm.main_s}
	Let Assumption~\ref{as.feasibility_s} hold. 
	Fix $p\in\rbr{0,1}$ and $T\geq \max(|X||A|,|B||Y|)$. 
	In Algorithm~\ref{UCB-MPD} with the primal-dual update~\eqref{eq.primal_s} and~\eqref{eq.dual_s}, we set
	$V=L\sqrt{T}$, $\eta = 1/(TL)$, and $\theta = 1/T$.
	Then, the regret~\eqref{eq.regret} and the constraint violation~\eqref{eq.violation} satisfy 
	\[
	\begin{array}{rcl}
	\text{\normalfont Regret}(T) & \leq & 
	\tilde O \big( (|X|+|Y|) L \sqrt{T(|A|+|B|)} \big)
	\\[0.2cm]
	\text{\normalfont Violation}_1(T), \text{\normalfont Violation}_2(T) & \leq & 
	\tilde O \big( (|X|+|Y|) L \sqrt{T(|A|+|B|)} \big)
	\end{array}
	\]
	with probability $1-p$, where $\tilde O(\cdot)$ hides factor $\log\tfrac{1}{p}$.
\end{theorem}

We analyze Algorithm~\ref{UCB-MPD} with the primal-dual update~\eqref{eq.primal_s} and~\eqref{eq.dual_s} by following the proof idea in Appendix~\ref{proof}. For completeness, we provide proof details in next two sections.

\subsection{Regret Analysis}

We recall that our algorithm maintains the occupancy measures $(\hat q_1^{\,t}, \hat q_2^{\,t})$ for estimating policies $(\pi^t,\mu^t)$ and Problem~\eqref{eq.opt_wc_s} defines the comparison solution $(q_1^\star,q_2^\star)$ in hindsight. We decompose the regret~\eqref{eq.regret_s} as follows,
\[
\begin{array}{rcl}
\text{Regret}(T) 
&=&\displaystyle 
\underbrace{
	\sum_{t\,=\,0}^{T-1} \big\langle{ \hat q_1^{\,t}\cdot  q_2^\star- q_1^\star\cdot \hat q_2^{\,t} },{r^{t}}\big\rangle 
}_{\hat{\text{\normalfont Regret}}(T)}
\,+\,
\underbrace{
	\sum_{t\,=\,0}^{T-1} \inner{( q_1^t-\hat q_1^{\,t}) \cdot q_2^\star}{r^t} 
}_{\text{\normalfont Error}_1}
\,+\,
\underbrace{
	\sum_{t\,=\,0}^{T-1} \inner{q_1^\star \cdot (\hat q_2^{\,t}-q_2^t)}{r^t} 
}_{\text{\normalfont Error}_2}
\end{array}
\]
where ${\text{\normalfont Error}_1}$ is the error of using $\hat q_1^{\,t}$ for the min-player and ${\text{\normalfont Error}_2}$ is the error of using $\hat q_2^{\,t}$ for the max-player. By the occupancy measures in Algorithm~\ref{UCB-MPD}, ${\text{\normalfont Error}_1}$ and ${\text{\normalfont Error}_2}$ take the bounds in Lemma~\ref{lem.error12}. However, we need to develop a new upper bound for ${\hat{\text{\normalfont Regret}}(T)}$ as follows. 
\begin{lemma}\label{lem.gap_pd_s}
	Fix $\delta\in (0,1)$. Then, with probability $1-\delta$,
	\[
	\begin{array}{rcl}
	\displaystyle
	{\hat{\text{\normalfont Regret}}(T)}
	&\leq& \displaystyle V^{-1}
	\sum_{t\,=\,0}^{T-1} 
	\left(
	\lambda_1^{t} \big(\langle{ q_1^\star},{g^{t}}\rangle 
	-b_1\big)
	\,+\,
	\lambda_2^{t} \big(\langle{ q_2^\star},{h^{t}}\rangle -b_2\big)
	\right)
	\\[0.2cm]
	&&\displaystyle 
	\,+\, (\eta V)^{-1} L (1+\theta T)\big(\log (|X||A|) +\log (|Y||B|)\big)
	\,+\,(2 V^{-1}L+4\theta +\eta V) LT.
	\end{array}
	\]
\end{lemma}
\begin{proof}
	By Lemma~\ref{lem.empiricalP}, with probability $1-\delta$ it holds that 
	\[
	\Delta(P_1) \;\subset\; \cap_{t \,=\,0}^{T-1} \Delta(k_1^t) 
	\;\text{ and }\; 
	\Delta(P_2) \;\subset\; \cap_{t \,=\,0}^{T-1} \Delta(k_2^t).
	\]
	We note that the solution $(q_1^\star,q_2^\star)$ in hindsight to Problem~\eqref{eq.opt_wc_s} satisfies $q_1^\star\in\Delta(P_1)$ and $q_2^\star\in \Delta(P_2)$. Hence, $q_1^\star\in  \cap_{t \,=\,0}^{T-1} \Delta(k_1^t)  $ and $q_2^\star\in \Delta(P_2) \cap_{t \,=\,0}^{T-1} \Delta(k_2^t)$ with probability $1-\delta$. For episode $t$, we apply Lemma~\ref{lem.pushback} to the primal update~\eqref{eq.primal_s} with 
	\[
	f(x,y) \vert_{x\,=\,q_1,\, y\,=\,q_2} 
	\;=\;
	V\, \big\langle{q_1\cdot \hat q_2^{\,t-1}+\hat q_1^{\,t-1}\cdot q_2},{r^{t-1}}\big\rangle 
	\,+\,
	\lambda_1^{t-1} \langle{q_1},{g^{t-1}}\rangle 
	\,-\,\lambda_2^{t-1} \langle{q_2},{h^{t-1}}\rangle 
	\]
	and $x^\star = \hat q_1^{\,t}$, $y^\star= \hat q_2^{\,t}$, $x' = \tilde q_1^{\,t-1}$, $y'=\tilde q_2^{\,t-1}$, $x = q_1^\star$, and $y=q_2^\star$. Thus, with  probability $1-\delta$ it holds for any $t$ that
	\[
	\begin{array}{rcl}
	&& \!\!\!\! \!\!\!\! \!\! V\, \big\langle{\hat q_1^{\,t}\cdot \hat q_2^{\,t-1}+\hat q_1^{\,t-1}\cdot  q_2^\star},{r^{t-1}}\big\rangle 
	\,+\,
	\lambda_1^{t-1} \langle{\hat q_1^{\,t}},{g^{t-1}}\rangle 
	\,-\,\lambda_2^{t-1} \langle{ q_2^\star},{h^{t-1}}\rangle 
	\\[0.2cm]
	&& \!\!\!\! \!\!\!\! \!\! \,+\, \eta^{-1} \big(D(\hat q_1^{\,t},\tilde q_1^{\,t-1}) + D(\hat q_2^{\,t}, \tilde q_2^{\,t-1})\big)
	\\[0.2cm]
	&\leq  & V\, \big\langle{ q_1^\star\cdot \hat q_2^{\,t-1}+\hat q_1^{\,t-1}\cdot \hat q_2^{\,t}},{r^{t-1}}\big\rangle 
	\,+\,
	\lambda_1^{t-1} \langle{ q_1^\star},{g^{t-1}}\rangle 
	\,-\,\lambda_2^{t-1} \langle{ \hat q_2^{\,t}},{h^{t-1}}\rangle 
	\\[0.2cm]
	&& \,+\, \eta^{-1} \big(D( q_1^\star,\tilde q_1^{\,t-1}) \,+\, D( q_2^\star, \tilde q_2^{\,t-1}) \,-\, D( q_1^\star,\hat q_1^{\,t}) \,-\, D( q_2^\star, \hat q_2^{\,t}) \big)
	\end{array}
	\]
	or, equivalently,
	\begin{equation}\label{eq.pushback_V_s}
	\begin{array}{rcl}
	&& \!\!\!\! \!\!\!\! \!\! V\, \big\langle{\hat q_1^{\,t}\cdot \hat q_2^{\,t-1}- \hat q_1^{\,t-1}\cdot \hat q_2^{\,t}},{r^{t-1}}\big\rangle 
	\,+\,
	\lambda_1^{t-1} \langle{\hat q_1^{\,t}},{g^{t-1}}\rangle 
	\,+\,\lambda_2^{t-1} \langle{ \hat q_2^{\,t}},{h^{t-1}}\rangle 
	\\[0.2cm]
	&& \!\!\!\! \!\!\!\! \!\! \,+\, \eta^{-1} \big(D(\hat q_1^{\,t}, \tilde q_1^{\,t-1}) + D(\hat q_2^{\,t}, \tilde q_2^{\,t-1})\big)
	\\[0.2cm]
	&\leq&  V\, \big\langle{ q_1^\star\cdot \hat q_2^{\,t-1}-\hat q_1^{\,t-1}\cdot  q_2^\star },{r^{t-1}}\big\rangle 
	\,+\,
	\lambda_1^{t-1} \langle{ q_1^\star},{g^{t-1}}\rangle 
	\,+\,\lambda_2^{t-1} \langle{ q_2^\star},{h^{t-1}}\rangle 
	\\[0.2cm]
	&& \,+\, \eta^{-1} \big(D( q_1^\star,\tilde q_1^{\,t-1}) \,+\, D( q_2^\star, \tilde q_2^{\,t-1}) \,-\, D( q_1^\star,\hat q_1^{\,t}) \,-\, D( q_2^\star, \hat q_2^{\,t}) \big).
	\end{array}
	\end{equation}
	
	Let $\Delta_1^t\DefinedAs \frac{1}{2} \rbr{(\lambda_1^{t})^2-(\lambda_1^{t-1})^2}$ be the drift of the first consecutive dual updates. Then,
	\begin{equation}\label{eq.drift_Q_s1}
	\begin{array}{rcl}
	\Delta_1^t &=& \displaystyle 
	\frac{1}{2}\rbr{ ( \lambda_1^{t} )^2 - ( \lambda_1^{t-1} )^2 }
	\\[0.2cm]
	&=& \displaystyle 
	\frac{1}{2}\rbr{ \max\!^2 \Big( \lambda_1^{t-1} \,+\, \big( \langle{\hat q_1^{\,t}},{g^{t-1}}\rangle-b_1 \big),\; 0 \Big) - ( \lambda_1^{t-1} )^2 }
	\\[0.2cm]
	&\leq& \displaystyle 
	\lambda_1^{t-1}\big( \langle{\hat q_1^{\,t}},{g^{t-1}}\rangle-b_1 \big) \,+\, \frac{1}{2}\big( \langle{\hat q_1^{\,t}},{g^{t-1}}\rangle-b_1 \big)^2
	\\[0.2cm]
	&\leq& \displaystyle 
	\lambda_1^{t-1}\big( \langle{\hat q_1^{\,t}},{g^{t-1}}\rangle-b_1 \big) \,+\, L^2
	\end{array}
	\end{equation}
	where the first inequality is due to $\max^2(x,0)\leq x^2$ and we apply $\langle{\hat q_1^{\,t}},{g^{t-1}}\rangle \in [0,L]$, $b_1\in [0,L]$ in the last inequality. Similarly, if we let $\Delta_2^t\DefinedAs \frac{1}{2} \rbr{(\lambda_2^{t})^2-(\lambda_2^{t-1})^2}$, then
	\begin{equation}\label{eq.drift_Q_s2}
	\Delta_2^t  \;\leq\; \lambda_2^{t-1}\big( \langle{\hat q_2^{\,t}},{h^{t-1}}\rangle-b_2 \big) \,+\, L^2.
	\end{equation}
	Adding~\eqref{eq.drift_Q_s1} and~\eqref{eq.drift_Q_s2} to~\eqref{eq.pushback_V_s} from both sides of the inequalities without changing the inequality direction yields
	\begin{equation}\label{eq.pushback_VQ_s}
	\begin{array}{rcl}
	&& \!\!\!\! \!\!\!\! \!\! 
	V\, \big\langle{\hat q_1^{\,t}\cdot \hat q_2^{\,t-1}- \hat q_1^{\,t-1}\cdot \hat q_2^{\,t}},{r^{t-1}}\big\rangle 
	\,+\,
	\Delta_1^{t} 
	\,+\,
	\Delta_2^{t} 
	\,+\, \eta^{-1} \big(D(\hat q_1^{\,t},\tilde q_1^{\,t-1}) + D(\hat q_2^{\,t}, \tilde q_2^{\,t-1})\big)
	\\[0.2cm]
	&\leq& V\, \big\langle{ q_1^\star\cdot \hat q_2^{\,t-1}-\hat q_1^{\,t-1}\cdot  q_2^\star },{r^{t-1}}\big\rangle 
	\,+\,
	\lambda_1^{t-1} \big(\langle{ q_1^\star},{g^{t-1}}\rangle 
	-b_1\big)
	\,+\,
	\lambda_2^{t-1} \big( \langle{ q_2^\star},{h^{t-1}}\rangle -b_2\big)
	\,+\,2L^2
	\\[0.2cm]
	&& \,+\, \eta^{-1} \big(D( q_1^\star,\tilde q_1^{\,t-1}) + D( q_2^\star, \tilde q_2^{\,t-1}) \,-\, D( q_1^\star,\hat q_1^{\,t}) \,-\, D( q_2^\star, \hat q_2^{\,t}) \big).
	\end{array}
	\end{equation}
	However,
	\[
	\begin{array}{rcl}
	&& \!\!\!\! \!\!\!\! \!\! V\, \big\langle{\hat q_1^{\,t}\cdot \hat q_2^{\,t-1}- \hat q_1^{\,t-1}\cdot \hat q_2^{\,t}},{r^{t-1}}\big\rangle 
	\,+\, \eta^{-1} \big(D(\hat q_1^{\,t},\tilde q_1^{\,t-1}) + D(\hat q_2^{\,t}, \tilde q_2^{\,t-1})\big)
	\\[0.2cm]
	&=& V\, \big\langle{\hat q_1^{\,t}\cdot \hat q_2^{\,t-1}- \tilde q_1^{\,t-1}\cdot \hat q_2^{\,t-1}},{r^{t-1}}\big\rangle 
	\,+\,V\, \big\langle{ \tilde q_1^{\,t-1}\cdot \hat q_2^{\,t-1} - \hat q_1^{\,t-1}\cdot \hat q_2^{\,t-1}},{r^{t-1}}\big\rangle 
	\\[0.2cm]
	&&  
	\,+\,V\, \big\langle{ \hat q_1^{\,t-1}\cdot \hat q_2^{\,t-1} - \hat q_1^{\,t-1}\cdot \tilde q_2^{\,t-1}},{r^{t-1}}\big\rangle 
	\,+\, V\, \big\langle{\hat q_1^{\,t-1}\cdot \tilde q_2^{\,t-1}- \hat q_1^{\,t-1}\cdot \hat q_2^{\,t}},{r^{t-1}}\big\rangle 
	\\[0.2cm]
	&&  \,+\, \eta^{-1} D(\hat q_1^{\,t},\tilde q_1^{\,t-1}) \,+\, \eta^{-1} D(\hat q_2^{\,t}, \tilde q_2^{\,t-1})
	\\[0.2cm]
	&\geq&  -\,V\, \norm{\hat q_2^{\,t-1}\cdot r^{t-1} }_\infty \norm{\hat q_1^{\,t} - \tilde q_1^{\,t-1}}_1
	\,-\,V\norm{ \hat q_2^{\,t-1} \cdot r^{t-1} }_\infty\norm{ \tilde q_1^{\,t-1}- \hat q_1^{\,t-1}}_1
	\\[0.2cm]
	&& 
	\,-\,V\,\norm{ \hat q_1^{\,t-1}\cdot  r^{t-1} }_\infty\norm{ \hat q_2^{\,t-1} - \tilde q_2^{\,t-1}}_1
	\,-\, V\,\norm{ \hat q_1^{\,t-1}\cdot r^{t-1} }_\infty\norm{ \tilde q_2^{\,t-1}-\hat q_2^{\,t} }_1
	\\[0.2cm]
	&& \,+\, (2\eta L)^{-1}\norm{\hat q_1^{\,t}-\tilde q_1^{\,t-1}}_1^2 \,+\, (2\eta L)^{-1} \norm{\hat q_2^{\,t} - \tilde q_2^{\,t-1}}_1
	\\[0.2cm]
	&\geq&  -\,V\, \norm{\hat q_1^{\,t} - \tilde q_1^{\,t-1}}_1
	\,-\, 2\theta V L
	\,+\, (2\eta L)^{-1}\norm{\hat q_1^{\,t}-\tilde q_1^{\,t-1}}_1^2
	\\[0.2cm]
	&& 
	\,-\,2\theta VL
	\,-\, V\,\norm{ \tilde q_2^{\,t-1}-\hat q_2^{\,t} }_1
	\,+\, (2\eta L)^{-1} \norm{\hat q_2^{\,t} - \tilde q_2^{\,t-1}}_1
	\\[0.2cm]
	&\geq& \,-\,4\theta V L\,-\,\eta V^2 L
	\end{array}
	\]
	where we apply the H\"older's inequality and Lemma~\ref{lem.D_lb} in the first inequality, the second inequality is due to that
	\[
	\begin{array}{rcl}
	\norm{ \tilde q_1^{\,t-1}- \hat q_1^{\,t-1}}_1 
	&=& \displaystyle\sum_{\ell\,=\,0}^{L-1}\sum_{x\,\in\,X_\ell}\sum_{a\,\in\,A}\abr{ (1-\theta)\hat {q}_1^{\,t-1}(x,a)+\theta\frac{1}{|X_\ell||A|}-\hat q_1^{\,t-1}(x,a) }
	\\[0.2cm]
	&\leq& \displaystyle \theta\sum_{\ell\,=\,0}^{L-1}\sum_{x\,\in\,X_\ell}\sum_{a\,\in\,A} \hat {q}_1^{\,t-1}(x,a)+\theta\sum_{\ell\,=\,0}^{L-1}\sum_{x\,\in\,X_\ell}\sum_{a\,\in\,A}\frac{1}{|X_\ell||A|}
	\\[0.2cm]
	&=& 2\theta L
	\end{array}
	\]
	and $\Vert{ \tilde q_2^{\,t-1}- \hat q_2^{\,t-1}}\Vert_1 \leq 2\theta L$ that can be proved similarly, and the last inequality is due to 
	$-bx+ax^2\geq - b^2/(4a)$ for any $a$, $b > 0$. Therefore, we take the lower bound above for the left-hand side of~\eqref{eq.pushback_VQ_s},
	\begin{equation}\label{eq.pushback_simplify_s}
	\begin{array}{rcl}
	&& \!\!\!\! \!\!\!\! \!\!
	\Delta_1^{t} \,+\,\Delta_2^{t} \,-\,4\theta V L\,-\,\eta V^2 L
	\\[0.2cm]
	&\leq& V\, \big\langle{ q_1^\star\cdot \hat q_2^{\,t-1}-\hat q_1^{\,t-1}\cdot  q_2^\star },{r^{t-1}}\big\rangle 
	\,+\,
	\lambda_1^{t-1} \big(\langle{ q_1^\star},{g^{t-1}}\rangle 
	-b_1\big)
	\,+\,
	\lambda_2^{t-1} \big(\langle{ q_2^\star},{h^{t-1}}\rangle -b_2\big)
	\,+\,2L^2
	\\[0.2cm]
	&& \,+\, \eta^{-1} \big(D( q_1^\star,\tilde q_1^{t-1}) + D( q_2^\star, \tilde q_2^{\,t-1}) -D( q_1^\star,\hat q_1^{\,t}) - D( q_2^\star, \hat q_2^{\,t}) \big).
	\end{array}
	\end{equation}
	By Lemma~\ref{lem.D_difference},
	\[
	\begin{array}{rcl}
	D( q_1^\star,\tilde q_1^{\,t-1}) - D( q_1^\star,\hat q_1^{\,t}) &=& D( q_1^\star,\tilde q_1^{\,t-1}) - D( q_1^\star,\hat q_1^{\,t-1})  + D( q_1^\star,\hat q_1^{\,t-1}) - D( q_1^\star,\hat q_1^{\,t}) 
	\\[0.2cm]
	&\leq & \theta L\log (|X||A|)+ D( q_1^\star,\hat q_1^{\,t-1}) - D( q_1^\star,\hat q_1^{\,t}) 
	\end{array}
	\]
	and, similarly, 
	\[
	D( q_2^\star, \tilde q_2^{\,t-1})  - D( q_2^\star, \hat q_2^{\,t}) \;\leq\; \theta L\log (|Y||B|)+ D( q_2^\star,\hat q_2^{\,t-1}) - D( q_2^\star,\hat q_2^{\,t}). 
	\]
	We now simplify~\eqref{eq.pushback_simplify_s} into
	\[
	\begin{array}{rcl}
	\Delta_1^{t}  \,+\, \Delta_2^{t}
	&\leq&  V\, \big\langle{ q_1^\star\cdot \hat q_2^{\,t-1}-\hat q_1^{\,t-1}\cdot  q_2^\star },{r^{t-1}}\big\rangle 
	\,+\,
	\lambda_1^{t-1} \big(\langle{ q_1^\star},{g^{t-1}}\rangle 
	-b_1\big)
	\,+\,
	\lambda_2^{t-1} \big( \langle{ q_2^\star},{h^{t-1}}\rangle -b_2\big)
	\\[0.2cm]
	&& \,+\, \eta^{-1} \big(D( q_1^\star,\hat q_1^{\,t-1}) + D( q_2^\star, \hat q_2^{\,t-1}) -D( q_1^\star,\hat q_1^{\,t}) - D( q_2^\star, \hat q_2^{\,t}) \big)
	\\[0.2cm]
	&& \,+\, \eta^{-1}\theta L \big(\log (|X||A|) +\log (|Y||B|)\big)\,+\,2L^2 \,+\,4\theta V L\,+\,\eta V^2 L
	\end{array}
	\]
	which leads to the desired result by summing it up from $t=1$ to $T$,
	\[
	\begin{array}{rcl}
	\displaystyle
	\sum_{t\,=\,1}^{T} \big(\Delta_1^{t} + \Delta_2^{t}  \big)
	&\leq& \displaystyle V\sum_{t\,=\,1}^{T} \big\langle{ q_1^\star\cdot \hat q_2^{\,t-1}-\hat q_1^{\,t-1}\cdot  q_2^\star },{r^{t-1}}\big\rangle 
	\\[0.2cm]
	&& \displaystyle
	\,+\,\sum_{t\,=\,1}^{T}
	\left( \lambda_1^{t-1} \big(\langle{ q_1^\star},{g^{t-1}}\rangle 
	-b_1\big) + \lambda_2^{t-1} \big( \langle{ q_2^\star},{h^{t-1}}\rangle-b_2\big) \right)
	\\[0.2cm]
	&& \displaystyle\,+\, \eta^{-1} \sum_{t\,=\,1}^{T}\big(D( q_1^\star,\hat q_1^{\,t-1}) + D( q_2^\star, \hat q_2^{\,t-1}) -D( q_1^\star,\hat q_1^{\,t}) - D( q_2^\star, \hat q_2^{\,t}) \big)
	\\[0.2cm]
	&&\displaystyle \,+\, \eta^{-1}\theta L T\big(\log (|X||A|) +\log (|Y||B|)\big)\,+\,2L^2 T\,+\,4\theta V LT\,+\,\eta V^2 LT
	\\[0.2cm]
	&\leq& \displaystyle V\sum_{t\,=\,1}^{T} \big\langle{ q_1^\star\cdot \hat q_2^{\,t-1}-\hat q_1^{\,t-1}\cdot  q_2^\star },{r^{t-1}}\big\rangle 
	\\[0.2cm]
	&&\displaystyle
	\,+\,\sum_{t\,=\,1}^{T}
	\left(
	\lambda_1^{t-1} \big(\langle{ q_1^\star},{g^{t-1}}\rangle 
	-b_1\big)
	+	
	\lambda_2^{t-1}\big( \langle{ q_2^\star},{h^{t-1}}\rangle -b_2\big)
	\right)
	\\[0.2cm]
	&& \displaystyle\,+\, \eta^{-1} \big(D( q_1^\star,\hat q_1^{\,0}) + D( q_2^\star, \hat q_2^{\,0}) \big)
	\\[0.2cm]
	&&\displaystyle \,+\, \eta^{-1}\theta L T\big(\log (|X||A|) +\log (|Y||B|)\big)\,+\,2L^2 T\,+\,4\theta V LT\,+\,\eta V^2 LT
	\end{array}
	\]
	which leads to the desired result by noting that 
	\[
	D( q_1^\star,\hat q_1^{\,0})\;\leq\; L\log(|X||A|), \; D( q_2^\star, \hat q_2^{\,0}) \;\leq\; L\log(|Y||B|), \; \text{ and } \; \sum_{t\,=\,1}^{T}\big( \Delta_1^{t}+\Delta_2^{t} \big)  \; \geq \; 0. 
	\]
\end{proof}

To analyze the bound in Lemma~\ref{lem.gap_pd_s},  in Lemma~\ref{lem.lambda_s}, we next utilize a new drift bound from Lemma~\ref{lem.drift} to establish the boundedness of $\lambda^t \DefinedAs (\lambda_1^{t},\lambda_2^{t})$ first. Then, we apply a general Azuma-Hoeffding inequality for supermartingales in Lemma~\ref{lem.violation_s}. 

\begin{lemma}\label{lem.lambda_s}
	Let Assumption~\ref{as.feasibility_s} hold. 
	Fix $\delta\in (0,1)$. For any integer $t_0>0$, with probability $1-T\delta$,
	\[
	\norm{\lambda^t} \; \leq\; \displaystyle \Theta + 2 t_0 L+ t_0 \frac{64 L^2}{\xi} \log \rbr{\frac{128 L^2}{\xi}} + t_0 \frac{64 L^2}{\xi} \log\frac{1}{\delta}
	\]
	for all $t = 1,\ldots,T$, where $\xi >0$ and
	\[
	\begin{array}{rcl}
	\Theta &\DefinedAs& t_0
	\rbr{\tfrac{1}{2}\xi  +2 L} \,+\,\tfrac{4L^2 + (8\theta +2\eta V
		+2 )VL}{\xi}
	\,+\,\tfrac{2L (\log(|X||A|/\theta)+\log(|Y||B|/\theta)) }{t_0\xi\eta}.
	\end{array}
	\]
\end{lemma}
\begin{proof}
	Let $\calF^{t}$ be an $\sigma$-algebra that is generated by the state-action sequence, reward/utility functions for both players up to episode $t$. At the beginning, $\calF^0=\{\emptyset,\Omega \}$. We have a discrete-time random process $\{\norm{\lambda^t},t\geq0 \}$ that adapts to $\calF^{t}$. 
	It suffices to check all assumptions in Lemma~\ref{lem.drift}. 
	
	By the dual update~\eqref{eq.dual_s},
	\[
	\begin{array}{rcl}
	\abr{  \lambda_1^{t+1}  - \lambda_1^{t} } 
	&=&
	\abr{ \max\! \Big( \lambda_1^{t} \,+\, \big( \langle{\hat q_1^{\,t+1}},{g^{t}}\rangle - b_1 \big),\; 0 \Big) - \lambda_1^{t}  }
	\\[0.2cm]
	&\leq&\abr{ \langle{\hat q_1^{\,t+1}},{g^{t}}\rangle - b_1  }
	\\[0.2cm]
	&\leq& L
	\end{array}
	\]
	where the first inequality is clear from two cases for $\max(\cdot)$ and the second inequality is due to $\langle{\hat q_1^{\,t+1}},{g^{t}}\rangle\in [0,L]$, $b_1\in [0,L]$. Similarly, $\abr{  \lambda_2^{t+1}  - \lambda_2^{t} } \leq L$. Hence,
	\[
	\abr{ \norm{\lambda^{t+1}} - \norm{\lambda^t}}
	\;\leq\; 
	\norm{\lambda^{t+1}-\lambda^{t} } 
	\;=\; \sqrt{(\lambda_1^{t+1}-\lambda_1^{t})^2 + (\lambda_2^{t+1}-\lambda_2^{t})^2} \;\leq\;{2} L.
	\]
	Consequently,
	\begin{equation}\label{eq.lambda.normdiff}
	\norm{\lambda}^{t+t_0}  - \norm{\lambda}^{t}  \;=\;\sum_{s\,=\,t}^{t+t_0-1} \big( \norm{\lambda}^{s+1}  - \norm{\lambda}^{s} \big)
	\;\leq\;\sum_{s\,=\,t}^{t+t_0-1} \abr{ \norm{\lambda}^{s+1}  - \norm{\lambda}^{s} }
	\;\leq\; 2t_0 L
	\end{equation}
	which leads to $\mathbb{E} [ \, \norm{\lambda}^{t+t_0}  - \norm{\lambda}^{t} \,\vert\,\calF^t\, ]\leq 2t_0 L$. It is convenient to take $\delta_{\max} = 2L$ in Lemma~\ref{lem.drift}.
	
	We next determine the validity of other assumptions in Lemma~\ref{lem.drift}. Let us denote the event in Lemma~\ref{lem.empiricalP} by $\mathcal{E}_{\text{good}}$ and we have $P(\mathcal{E}_{\text{good}})\geq 1-\delta$. Let $\Delta^t \DefinedAs \frac{1}{2} \big( \norm{\lambda^t}^2-\norm{\lambda^{t-1}}^2 \big)$. Clearly, $\Delta^t = \Delta_1^t+\Delta_2^t$. 
	We recall that the proof of Lemma~\ref{lem.gap_pd} remains to be valid if we replace $q_1^\star$ by $\bar q_1$ and $q_2^\star$ by $\bar q_2$ starting from~\eqref{eq.pushback_V_s}. By doing so, it is ready to obtain a similar result as~\eqref{eq.pushback_simplify_s}: under the good event $\mathcal{E}_{\text{good}}$ it holds for any $\tau$ that
	\[
	\begin{array}{rcl}
	&& \!\!\!\! \!\!\!\! \!\! 
	\Delta^{\tau} \,-\,4\theta V L\,-\,\eta V^2 L
	\\[0.2cm]
	& \leq & V\, \big\langle{ \bar q_1\cdot \hat q_2^{\,\tau-1}-\hat q_1^{\,\tau-1}\cdot \bar q_2 },{r^{\tau-1}}\big\rangle 
	\,+\,
	\lambda_1^{\tau-1} \big(\langle{ \bar q_1},{g^{\tau-1}}\rangle 
	-b_1\big)
	\,+\,
	\lambda_2^{\tau-1} \big( \langle{ \bar q_2},{h^{\tau-1}}\rangle -b_2\big)
	\,+\,2L^2
	\\[0.2cm]
	&& \,+\, \eta^{-1} \big(D(\bar q_1,\tilde q_1^{\,\tau-1}) + D( \bar q_2, \tilde q_2^{\,\tau-1}) -D( \bar q_1,\hat q_1^{\,\tau}) - D( \bar q_2, \hat q_2^{\,\tau}) \big).
	\end{array}
	\]
	or, equivalently,
	\begin{equation}\label{eq.pushback_simplify_slater_s}
	\begin{array}{rcl}
	&& \!\!\!\! \!\!\!\! \!\! 
	\norm{\lambda^\tau}^2 \,-\, \norm{\lambda^{\tau-1}}^2
	\\[0.2cm]
	&\leq& 2V\, \big\langle{ \bar q_1\cdot \hat q_2^{\,\tau-1}-\hat q_1^{\,\tau-1}\cdot \bar q_2 },{r^{\tau-1}}\big\rangle 
	\,+\,
	2\lambda_1^{\tau-1} \big(\langle{ \bar q_1},{g^{\tau-1}}\rangle 
	-b_1\big)
	\,+\,
	2\lambda_2^{\tau-1} \big(
	\langle{ \bar q_2},{h^{\tau-1}}\rangle -b_2\big)
	\,+\,4L^2
	\\[0.2cm]
	&& \,+\, 2\eta^{-1} \big(D(\bar q_1,\tilde q_1^{\,\tau-1}) + D( \bar q_2, \tilde q_2^{\,\tau-1}) -D( \bar q_1,\hat q_1^{\,\tau}) - D( \bar q_2, \hat q_2^{\,\tau}) \big)
	\,+\,8\theta V L\,+\,2\eta V^2 L.
	\end{array}
	\end{equation}
	We note that $ |\langle{ \bar q_1\cdot \hat q_2^{\,\tau}-\hat q_1^{\,\tau}\cdot \bar q_2 },{r^{\tau}}\rangle|\leq L$. By summing both sides of~\eqref{eq.pushback_simplify_slater_s} from $\tau=t+1$ to $\tau=t+t_0$,
	\[
	\begin{array}{rcl}
	\norm{\lambda^{t+t_0}}^2 \,-\, \norm{\lambda^{t}}^2
	& \leq & \displaystyle 2t_0VL
	\,+\,2\sum_{\tau\,=\,t}^{t+t_0-1} 
	\left(
	\lambda_1^{\tau} \big(\langle{ \bar q_1},{g^{\tau}}\rangle - b_1\big)
	+
	\lambda_2^{\tau} \big( \langle{ \bar q_2},{h^{\tau}}\rangle - b_2\big)
	\right)
	\,+\,4t_0L^2
	\\[0.2cm]
	&& \,+\, 2\eta^{-1} \big(D(\bar q_1,\tilde q_1^{\,t}) + D( \bar q_2, \tilde q_2^{\,t}) \big)
	\,+\,8t_0\theta V L\,+\,2t_0\eta V^2 L
	\end{array}
	\]
	where we omit two non-positive terms.
	Taking the conditional expectation given $\calF^{t}$ and $\mathcal{E}_{\text{good}}$ yields,
	\begin{equation}\label{eq.pushback_slater_s}
	\begin{array}{rcl}
	&& \!\!\!\! \!\!\!\! \!\! 
	\mathbb{E} \sbr{ \norm{\lambda^{t+t_0}}^2 \,-\, \norm{\lambda^{t}}^2 \,\vert\, \calF^{t}, \mathcal{E}_{\text{good}}}
	\\[0.2cm]
	& \leq & \displaystyle 2t_0VL
	\,+\, 2 \sum_{\tau\,=\,t}^{t+t_0-1} 
	\mathbb{E} \sbr{ \lambda_1^{\tau} \big(\langle{ \bar q_1},{g^{\tau}}\rangle 
		-b_1\big) +  \lambda_2^{\tau} \big( \langle{ \bar q_2},{h^{\tau}}\rangle - b_2 \big) \,\vert\, \calF^{t}, \mathcal{E}_{\text{good}}}	\,+\,4t_0L^2
	\\[0.2cm]
	&& \,+\, 2\eta^{-1} \mathbb{E} \sbr{ D(\bar q_1,\tilde q_1^{\,t}) + D( \bar q_2, \tilde q_2^{\,t}) \,\vert\, \calF^{t}, \mathcal{E}_{\text{good}}}
	\,+\,8t_0\theta V L\,+\,2t_0\eta V^2 L
	\\[0.2cm]
	& \leq & \displaystyle 2t_0 VL
	\,-\,2\xi\sum_{\tau\,=\,t}^{t+t_0-1} 
	\mathbb{E}\sbr{ \norm{\lambda^{\tau} } \,\vert\, \calF^{t}, \mathcal{E}_{\text{good}}}
	\,+\,4t_0L^2
	\\[0.2cm]
	&& \,+\, 2\eta^{-1} L(\log(|X||A|/\theta)+\log(|Y||B|/\theta))
	\,+\,8t_0\theta V L\,+\,2t_0\eta V^2 L
	\\[0.2cm]
	& \leq & \displaystyle 2t_0 VL
	\,-\,2\xi t_0 \mathbb{E}\sbr{  \norm{\lambda^t} \,\vert\, \calF^{t}, \mathcal{E}_{\text{good}}} 
	\,+\,2 \xi t_0(t_0-1)L
	\,+\,4t_0L^2
	\\[0.2cm]
	&& \,+\, 2\eta^{-1} L(\log(|X||A|/\theta)+\log(|Y||B|/\theta))
	\,+\,8t_0\theta V L\,+\,2t_0\eta V^2 L
	\end{array}
	\end{equation}
	where the second inequality is due to Lemma~\ref{lem.D_difference} and the fact: by the law of total expectation, for any $\tau\geq t$, $\calF^{t}\subset\calF^\tau$ and
	\[
	\begin{array}{rcl}
	&& \!\!\!\! \!\!\!\! \!\!\mathbb{E} \sbr{ \lambda_1^{\tau} \big(\langle{ \bar q_1},{g^{\tau}}\rangle 
		-b_1\big) + \lambda_2^{\tau} \big( \langle{ \bar q_2},{h^{\tau}}\rangle -b_2\big) \,\vert\, \calF^{t}, \mathcal{E}_{\text{good}}}
	\\[0.2cm]
	&=&
	\mathbb{E} \sbr{ \mathbb{E} \sbr{ \lambda_1^{\tau} \big(\langle{ \bar q_1},{g^{\tau}}\rangle 
			-b_1\big) + \lambda_2^{\tau} \big(
			\langle{ \bar q_2},{h^{\tau}}\rangle -b_2\big) \,\vert\, \calF^{\tau}, \mathcal{E}_{\text{good}}} \,\vert\, \calF^{t}, \mathcal{E}_{\text{good}}}
	\\[0.2cm]
	&=&
	\mathbb{E} \sbr{ \lambda_1^{\tau}  \mathbb{E} \sbr{ \langle{ \bar q_1},{g^{\tau}}\rangle 
			-b_1} \,\vert\, \calF^{t}, \mathcal{E}_{\text{good}}}
	\,+\,
	\mathbb{E} \sbr{ \lambda_2^{\tau}  \mathbb{E} \sbr{
			\langle{ \bar q_2},{h^{\tau}}\rangle -b_2} \,\vert\, \calF^{t}, \mathcal{E}_{\text{good}}}
	\\[0.2cm]
	&=&
	\mathbb{E} \sbr{ \langle{ \bar q_1},{g^{\tau}}\rangle - b_1} 
	\mathbb{E} \sbr{ \lambda_1^{\tau}  \,\vert\, \calF^{t}, \mathcal{E}_{\text{good}}}
	\,+\,
	\mathbb{E} \sbr{  \langle{ \bar q_2},{h^{\tau}}\rangle -b_2} 
	\mathbb{E} \sbr{ \lambda_2^{\tau}  \,\vert\, \calF^{t}, \mathcal{E}_{\text{good}}}
	\\[0.2cm]
	&\leq & \,-\, \xi \,	\mathbb{E}\sbr{ \lambda_1^{\tau} + \lambda_2^{\tau} \,\vert\, \calF^{t}, \mathcal{E}_{\text{good}}}
	\\[0.2cm]
	&\leq & \,-\, \xi \,	\mathbb{E}\sbr{ \norm{\lambda^{\tau} } \,\vert\, \calF^{t}, \mathcal{E}_{\text{good}}}
	\end{array}
	\]
	where the inequality is due to the strict feasibility assumption on $(\bar q_1,\bar q_2)$; the last inequality is due to that 
	\[
	\begin{array}{rcl}
	\displaystyle 
	\sum_{\tau\,=\,t}^{t+t_0-1} 
	\mathbb{E}\sbr{ \norm{\lambda^{\tau}}  \,\vert\, \calF^{t}, \mathcal{E}_{\text{good}}} 
	& \geq &
	\displaystyle
	\sum_{\tau\,=\,t}^{t+t_0-1} 
	\mathbb{E}\sbr{  \norm{\lambda^t} - 2(\tau-t)L \,\vert\, \calF^{t}, \mathcal{E}_{\text{good}}} 
	\\[0.2cm]
	& = &
	\displaystyle
	\sum_{\tau\,=\,0}^{t_0-1} 
	\mathbb{E}\sbr{ \norm{ \lambda^t} - 2\tau L \,\vert\, \calF^{t}, \mathcal{E}_{\text{good}}} 
	\end{array}
	\]
	which follows the fact $ \norm{\lambda^\tau} \geq \norm{\lambda^t} - 2(\tau-t)L$ for any $\tau\geq t\geq 0$ if we note that $ \abr{ \norm{\lambda^{t+1}}  - \norm{\lambda^{t}} } \leq 2L$.
	Hence, we can simplify~\eqref{eq.pushback_slater_s} as
	\[
	\begin{array}{rcl}
	&& \!\!\!\! \!\!\!\! \!\! 
	\mathbb{E} \sbr{ \norm{\lambda^{t+t_0}}^2 \,\vert\, \calF^{t}, \mathcal{E}_{\text{good}}}
	\\[0.2cm]
	& \leq & \displaystyle \mathbb{E} \sbr{ \norm{\lambda^{t}}^2\,\vert\, \calF^{t}, \mathcal{E}_{\text{good}}}
	\,-\,2\xi t_0 \mathbb{E} \sbr{ \norm{\lambda^t }\,\vert\, \calF^{t}, \mathcal{E}_{\text{good}}}
	\,+\,2 \xi t_0^2L
	\,+\,4t_0L^2
	\,+\, 2t_0 VL
	\\[0.2cm]
	&& \,+\, 2\eta^{-1} L(\log(|X||A|/\theta)+\log(|Y||B|/\theta))
	\,+\,8t_0\theta V L\,+\,2t_0\eta V^2 L
	\\[0.2cm]
	& \leq & \displaystyle \mathbb{E} \sbr{\norm{\lambda^{t}}^2\,\vert\, \calF^{t}, \mathcal{E}_{\text{good}}}
	\,-\,\xi t_0 \mathbb{E} \sbr{\norm{\lambda^t }\,\vert\, \calF^{t}, \mathcal{E}_{\text{good}}}
	\,-\,\xi t_0\Theta
	\,+\,2 \xi t_0^2L
	\,+\,4t_0L^2
	\,+\, 2t_0 VL
	\\[0.2cm]
	&& \,+\, 2\eta^{-1} L(\log(|X||A|/\theta)+\log(|Y||B|/\theta))
	\,+\,8t_0\theta V L\,+\,2t_0\eta V^2 L
	\\[0.2cm]
	& = & \displaystyle \mathbb{E} \sbr{\norm{\lambda^{t}}^2\,\vert\, \calF^{t}, \mathcal{E}_{\text{good}}}
	\,-\,\xi t_0 \mathbb{E} \sbr{\norm{\lambda^t}\,\vert\, \calF^{t}, \mathcal{E}_{\text{good}}} \,-\,\frac{1}{2} \xi^2 t_0^2
	\\[0.2cm]
	& \leq & \displaystyle
	\rbr{ \mathbb{E} \sbr{\norm{\lambda^{t}}\,\vert\, \calF^{t}, \mathcal{E}_{\text{good}}} - \frac{1}{2}\xi t_0 }^2
	\end{array}
	\]
	where we apply $\lambda^t\geq \Theta$ for the second inequality and we take $\Theta$ in Lemma~\ref{lem.drift},
	\[
	\Theta \;=\;
	\frac{1}{2}\xi t_0 +2 t_0L
	\,+\,\frac{4L^2 +8\theta V L+2\eta V^2 L
		+2 VL}{\xi}\,+\, \frac{2L(\log(|X||A|/\theta)+\log(|Y||B|/\theta)) }{t_0\xi\eta}.
	\]
	Taking the square root and applying the Jensen's inequality yield
	\[
	\mathbb{E} \sbr{ \norm{ \lambda^{t+t_0}} \,\vert\, \calF^{t}, \mathcal{E}_{\text{good}}}
	\;\leq\;
	\sqrt{\mathbb{E} \sbr{ \norm{\lambda^{t+t_0}}^2 \,\vert\, \calF^{t}, \mathcal{E}_{\text{good}}} }
	\;\leq\; 
	\norm{\lambda^{t}} - \frac{1}{2}\xi t_0
	\]
	which shows that $\mathbb{E} \sbr{ \norm{\lambda^{t+t_0}} - \norm{\lambda^{t}}\,\vert\, \calF^{t}, \mathcal{E}_{\text{good}} } \leq - \frac{1}{2}\xi t_0$. Application of law of total expectation to this inequality and~\eqref{eq.lambda.normdiff} with $\delta<\frac{1}{12}$ yields
	\[
	\begin{array}{rcl}
	\mathbb{E} \sbr{ \norm{\lambda^{t+t_0}} - \norm{\lambda^{t}}\,\vert\, \calF^{t}} 
	& = & P(\mathcal{E}_{\text{good}})\mathbb{E} \sbr{ \norm{\lambda^{t+t_0}} - \norm{\lambda^{t}}\,\vert\, \calF^{t}, \mathcal{E}_{\text{good}} } 
	\\[0.2cm]
	&  & + \,
	P(\bar{\mathcal{E}}_{\text{good}})\mathbb{E} \sbr{ \norm{\lambda^{t+t_0}} - \norm{\lambda^{t}}\,\vert\, \calF^{t}, \bar{\mathcal{E}}_{\text{good}} } 
	\\[0.2cm]
	& \leq &  - \dfrac{1}{2}\xi t_0 \times (1-\delta) \,+\, 2t_0L\times\delta
	\\[0.2cm]
	& \leq & - \dfrac{1}{4}\xi t_0 
	\end{array}
	\]
	which verifies the assumption of Lemma~\ref{lem.drift} if we take $\zeta=\xi/4$.
	
	We now have verified all assumptions of Lemma~\ref{lem.drift} with appropriate parameters $\Theta$, $\delta_{\max}$, $\zeta$. For episode $t$, with probability $1-\delta$ it holds that 
	\[
	\begin{array}{rcl}
	\norm{\lambda^t} & \leq & \displaystyle \Theta + t_0 \delta_{\max} + t_0 \frac{4\delta_{\max}^2}{\zeta} \log \rbr{\frac{8\delta_{\max}^2}{\zeta}} + t_0 \frac{4\delta_{\max}^2}{\zeta} \log\frac{1}{\delta}.
	\end{array}
	\]
	We complete the proof by taking a union bound over $t = 1,\cdots,T$.
\end{proof}

\begin{lemma}\label{lem.violation_s}
	Let Assumption~\ref{as.feasibility_s} hold.
	Fix $\delta\in (0,1)$. 
	For any integer $t_0>0$, with probability $1-2T\delta$,
	\[
	\sum_{t\,=\,0}^{T-1}
	\left(
	\lambda_1^{t} \big(\langle{ q_1^\star},{g^{t}}\rangle - b_1\big)
	+
	\lambda_2^{t} \big(\langle{ q_2^\star},{h^{t}}\rangle -b_2\big)
	\right)
	\;\leq\; \sqrt{2T c^2 \log (1/(\delta T))}
	\]
	where
	$c \DefinedAs  2\Theta L + 4 t_0 L^2 + \tfrac{128 t_0 L^3}{\xi} \rbr{\log \rbr{\tfrac{128 L^2}{\xi}} +\log\tfrac{1}{\delta}}$ 
	and
	$\xi >0$. 
\end{lemma}
\begin{proof}
	Let $Z^t \DefinedAs \sum_{\tau\,=\,0}^{t-1} \left(
	\lambda_1^{\tau} \big(\langle{ q_1^\star},{g^{\tau}}\rangle -b_1\big)+\lambda_2^{\tau} \big(\langle{ q_2^\star},{h^{\tau}}\rangle -b_2\big) \right)$. We note that 
	\[
	\begin{array}{rcl}
	&& \!\!\!\! \!\!\!\! \!\!
	\mathbb{E} \sbr{Z^t \,\vert\,\calF^{t-1}} 
	\\[0.2cm]
	&=& \displaystyle
	\mathbb{E} \sbr{ \sum_{\tau\,=\,0}^{t-1}
		\left(
		\lambda_1^{\tau} \big(\langle{ q_1^\star},{g^{\tau}}\rangle -b_1\big)
		+
		\lambda_2^{\tau} \big(\langle{ q_2^\star},{h^{\tau}}\rangle -b_2\big)
		\right)
		\,\bigg\vert\,\calF^{t-1}}
	\\[0.2cm]
	&=&  \displaystyle
	\mathbb{E} \sbr{ \sum_{\tau\,=\,0}^{t-2}
		\left(
		\lambda_1^{\tau} \big(\langle{ q_1^\star},{g^{\tau}}\rangle -b_1\big)
		+
		\lambda_2^{\tau} \big( \langle{ q_2^\star},{h^{\tau}}\rangle -b_2\big)
		\right)
		\,\bigg\vert\,\calF^{t-1}}
	\\[0.2cm]
	&&
	\,+\, 
	\lambda_1^{t-1} \mathbb{E} \sbr{\big(\langle{ q_1^\star},{g^{t-1}}\rangle -b_1\big)\,\vert\,\calF^{t-1}}
	\,+\,
	\lambda_2^{t-1} \mathbb{E} \sbr{\big(\langle{ q_2^\star},{h^{t-1}}\rangle -b_2\big)\,\vert\,\calF^{t-1}}
	\\[0.2cm]
	&\leq&  \displaystyle
	\mathbb{E} \sbr{ \sum_{\tau\,=\,0}^{t-2}
		\left(
		\lambda_1^{\tau} \big(\langle{ q_1^\star},{g^{\tau}}\rangle -b_1\big)
		+
		\lambda_2^{\tau} \big( \langle{ q_2^\star},{h^{\tau}}\rangle -b_2\big)
		\right)
		\,\bigg\vert\,\calF^{t-1}}
	\\[0.4cm]
	&=&
	\mathbb{E} \sbr{Z^{t-1}} 
	\end{array}
	\]
	where the inequality is because of $\mathbb{E} \sbr{\big(\langle{ q_1^\star},{g^{t-1}}\rangle 
		-b_1\big)\,\vert\,\calF^{t-1}} = \langle{ q_1^\star},{g}\rangle 
	-b_1\leq 0$ and \\
	$ \mathbb{E} \sbr{\big( \langle{ q_2^\star},{h^{t-1}}\rangle -b_1\big)\,\vert\,\calF^{t-1}} \leq \langle{ q_2^\star},{h}\rangle -b_2 \leq 0$. Hence,  $\{Z^t,t\geq0 \}$ a supermartingale. 
	
	We also note that
	\[
	\abr{Z^{t+1}-Z^{t}} 
	\; = \; 
	\lambda_1^{t} \abr{\langle{ q_1^\star},{g^{t}}\rangle -b_1}
	+
	\lambda_2^{t} \abr{ \langle{ q_2^\star},{h^{t}}\rangle -b_2}
	\; \leq \;
	2 \norm{\lambda^{t}} L
	\]
	Thus, if $|Z^{t+1}-Z^t|> c$ for some $c\in\mathbb{R}^+$, then $ \norm{\lambda^t}> c/(2L)$. Let $Y^t \DefinedAs \norm{\lambda^t} - c/(2L)$. Therefore,
	\[
	\{ |Z^{t+1}-Z^t|> c \}  \;\subset\;  \{ Y^t > 0 \}.
	\]
	By Lemma~\ref{lem.azuma_general},
	\begin{equation}\label{eq.azuma_dual_s}
	P \rbr{ \sum_{t\,=\,0}^{T-1}
		\left(
		\lambda_1^{t} \big(\langle{ q_1^\star},{g^{t}}\rangle - b_1\big)
		+
		\lambda_2^{t} \big( \langle{ q_2^\star},{h^{t}}\rangle - b_2\big)
		\right)
		\geq z } \;\leq\; {\rm e}^{-z^2/ (2c^2T)} \,+\, \sum_{\tau\,=\,0}^{T-1} P \rbr{ \norm{\lambda^t}>\frac{c}{2L}}.
	\end{equation}
	By Lemma~\ref{lem.lambda_s}, with probability $1-\delta$ it holds for any $t$ that
	\[
	\norm{\lambda^t} \; \leq\; \displaystyle \Theta + 2 t_0L + t_0 \frac{64 L^2}{\xi} \log \rbr{\frac{128 L^2}{\xi}} + t_0 \frac{64 L^2}{\xi} \log\frac{1}{\delta}
	\]
	or, equivalently,
	\[
	P \rbr{
		\norm{\lambda^t} \;\geq\; \displaystyle \Theta + 2 t_0L + t_0 \frac{64 L^2}{\xi} \log \rbr{\frac{128 L^2}{\xi}} + t_0 \frac{64 L^2}{\xi} \log\frac{1}{\delta}
	} \;\leq\; \delta.
	\]
	If we take 
	\[
	c\;=\;  2\Theta L + 4 t_0 L^2 + t_0 \frac{128 L^3}{\xi} \log \rbr{\frac{128 L^2}{\xi}} + t_0 \frac{128 L^3}{\xi} \log\frac{1}{\delta}
	\;\text{ and }\;
	z\;=\;\sqrt{2T c^2 \log (1/(\delta T))}
	\]
	then~\eqref{eq.azuma_dual_s} becomes
	\[
	P \rbr{ \sum_{t\,=\,0}^{T-1}
		\left(
		\lambda_1^{t} \big(\langle{ q_1^\star},{g^{t}}\rangle -b_1\big)
		+
		\lambda_2^{t} \big( \langle{ q_2^\star},{h^{t}}\rangle -b_2\big)
		\right)
		\geq z } 
	\;\leq\; 
	2\delta  T
	\]
	which proves the desired result.
\end{proof}

We now ready to conclude a bound on ${\hat{\text{\normalfont Regret}}(T)}$ by combining Lemma~\ref{lem.violation_s} and Lemma~\ref{lem.gap_pd_s}.
\begin{theorem}\label{thm.regret_hat_s}
	Let Assumption~\ref{as.feasibility_s} hold. Fix $T\geq \max(|X||A|,|B||Y|)$. Let $V=L\sqrt{T}$, $\eta = 1/(TL)$, $t_0=\sqrt{T}$, and $\theta = 1/T$. Then,
	with probability $1-2T\delta$ it holds that
	\[
	{\hat{\text{\normalfont Regret}}(T)}\;\leq\;
	\tilde O \big({(|X|+|Y|) L \sqrt T}\big).
	\]
\end{theorem}
\begin{proof} 
	Using the given parameters $V$, $\eta$, $t_0$, and $\theta$ for Lemma~\ref{lem.gap_pd_s}, ${\hat{\text{\normalfont Regret}}(T)}$ is upper bounded by 
	$\frac{1}{L\sqrt{T}}
	\sum_{t\,=\,0}^{T-1}
	\left(
	\lambda_1^{t} \big(\langle{ q_1^\star},{g^{t}}\rangle  -b_1\big)
	+
	\lambda_2^{t} \big( \langle{ q_2^\star},{h^{t}}\rangle -b_2\big)
	\right)
	+ \tilde O(L\sqrt{T })$
	with probability $1-\delta$.
	We note that $\Theta\leq	\tilde O(L^2\sqrt{T})$ and $T\geq \max(|X||A|,|B||Y|)$.	
	Using parameters in Lemma~\ref{lem.violation_s}, with probability $1-2T\delta$,
	\[
	\sum_{t\,=\,0}^{T-1}
	\left(
	\lambda_1^{t} \big(\langle{ q_1^\star},{g^{t}}\rangle - b_1\big)
	+
	\lambda_2^{t} \big( \langle{ q_2^\star},{h^{t}}\rangle - b_2\big)
	\right)
	\;\leq\; \tilde O(L^3 T).
	\] 
	We complete the proof by noting $L\leq |X|+|Y|$. 
\end{proof}

We conclude the regret bound in Theorem~\ref{thm.main_s} by combining Lemma~\ref{lem.error12} and Theorem~\ref{thm.regret_hat_s}, and $\delta=p/(2T)$.

\subsection{Constraint Violation Analysis}

We begin with a decomposition using the auxiliary occupancy measures $(q_1^t, q_2^t)$. By inserting $\langle \hat q_1^{\,t},g^{t}\rangle$ and $\langle \hat q_2^{\,t},h^{t}\rangle$ into $\text{Violation}_1(T) $ and $\text{Violation}_2(T) $, we have
\[
\begin{array}{rcl}
\text{Violation}_1(T) 
&=&
\underbrace{\sbr{\sum_{t\,=\,0}^{T-1} \rbr{\inner{\hat q_1^{\,t}}{g^t}-b_1}}_+}_{\hat{\text{\normalfont Violation}_1}(T)}
\,+\,
\underbrace{
	\sum_{t\,=\,0}^{T-1} \inner{q_1^t-\hat q_1^{\,t}}{g^t}
}_{\text{\normalfont Error}_3}
\end{array}
\]
\[
\begin{array}{rcl}
\text{Violation}_2(T) 
&=&
\underbrace{\sbr{\sum_{t\,=\,0}^{T-1} \rbr{\inner{\hat q_2^{\,t}}{h^t}-b_2}}_+}_{\hat{\text{\normalfont Violation}_2}(T)}
\, + \,
\underbrace{
	\sum_{t\,=\,0}^{T-1}\inner{q_2^t-\hat q_2^{\,t}}{h^t}
}_{\text{\normalfont Error}_4}.
\end{array}
\]

For ${\text{\normalfont Error}_3}$ and ${\text{\normalfont Error}_4}$, we have the same bounds in Lemma~\ref{lem.error34}.
We next bound ${\hat{\text{\normalfont Violation}_1}(T)}$ and ${\hat{\text{\normalfont Violation}_2}(T)}$ by applying the epoch property~\citep{jaksch2010near}; see a proof in Appendix~\ref{ap.violation_hat}.

\begin{theorem}\label{thm.violation_hat_s}
	Let $V=L\sqrt{T}$, $\eta = 1/(TL)$, $t_0=\sqrt{T}$, and $\theta = 1/T$. Then,
	\[
	\begin{array}{rcl}
	{\hat{\text{\normalfont Violation}_1}(T)}, {\hat{\text{\normalfont Violation}_2}(T)}
	&\leq&\displaystyle
	\norm{\lambda^T}
	\,+\, 
	\frac{2 }{T-1}
	\sum_{t\,=\,1}^{T} \norm{\lambda^{t-1}}
	\,+\,
	\tilde O\big(L\sqrt{ T(|X||A|+|Y||B|)}\big).
	\end{array}
	\]
\end{theorem}
\begin{proof}
	By the dual update~\eqref{eq.dual_s},
	\begin{subequations}
		\begin{equation}\label{eq.lambda_lb_s1}
		\begin{array}{rcl}
		\lambda_1^{t} &=& \max\! \Big( \lambda_1^{t-1} \,+\, \big(\inner{\hat q_1^{\,t}}{g^{t-1}} - b_1\, \big),\; 0 \Big)
		\\[0.2cm]
		&\geq &  \lambda_1^{t-1} \,+\, \big(\inner{\hat q_1^{\,t}}{g^{t-1}} - b_1\, \big)
		\\[0.2cm]
		&= &  \lambda_1^{t-1} \,+\, \big(\inner{\hat q_1^{\,t-1}}{g^{t-1}} - b_1\, \big)\,+\, \inner{\hat q_1^{\,t} - \hat q_1^{\,t-1}}{g^{t-1}}
		\\[0.2cm]
		&\geq &  \lambda_1^{t-1} \,+\, \big(\inner{\hat q_1^{\,t-1}}{g^{t-1}} - b_1\, \big)\,-\, \norm{\hat q_1^{\,t} - \hat q_1^{\,t-1}}_1 
		\end{array}
		\end{equation}
		where the last inequality is due to: $\langle{\hat q_1^{\,t} - \hat q_1^{\,t-1}},{g^{t-1}}\rangle\leq \Vert \hat q_1^{\,t} - \hat q_1^{\,t-1}\Vert_1 \Vert g^{t-1}\Vert_\infty$,
		and $\Vert g^{t-1}\Vert_\infty\in [0,1]$. Similarly, 
		\begin{equation}\label{eq.lambda_lb_s2}
		\lambda_2^{t} \;\geq\; \lambda_2^{t-1} \,+\, \big(\inner{\hat q_2^{\,t-1}}{h^{t-1}} - b_2\, \big)\,-\, \norm{\hat q_2^{\,t} - \hat q_2^{\,t-1}}_1.
		\end{equation}
	\end{subequations}
	We note that $\lambda_1^0=\lambda_2^0=0$ from the initialization. Summing up both sides of~\eqref{eq.lambda_lb_s1} from $t=1$ to $t=T$ leads to
	\begin{subequations}
		\begin{equation}\label{eq.violation_dqq_s1}
		\sum_{t\,=\,0}^{T-1} \big(\inner{\hat q_1^{\,t}}{g^{t}} -b_1\, \big)
		\;\leq\; 
		\lambda_1^T
		\,+\, 
		\sum_{t\,=\,1}^{T}
		\norm{\hat q_1^{\,t} - \hat q_1^{\,t-1}}_1.
		\end{equation}
		Similarly, 
		\begin{equation}\label{eq.violation_dqq_s2}
		\sum_{t\,=\,0}^{T-1} \big(\inner{\hat q_2^{\,t}}{h^{t}} -b_2\, \big)
		\;\leq\; 
		\lambda_2^T
		\,+\, 
		\sum_{t\,=\,1}^{T}
		\norm{\hat q_2^{\,t} - \hat q_2^{\,t-1}}_1.
		\end{equation}
	\end{subequations}
	Hence,
	\begin{equation}\label{eq.violation_dqq_s}
	{\hat{\text{\normalfont Violation}_1}(T)}, {\hat{\text{\normalfont Violation}_2}(T)}
	\;\leq\;
	\norm{\lambda^T}
	\,+\,
	\sum_{t\,=\,1}^{T}
	\left( \norm{\hat q_1^{\,t} - \hat q_1^{\,t-1}}_1 + \norm{\hat q_1^{\,t} - \hat q_1^{\,t-1}}_1 \right).
	\end{equation}
	
	We recall $ \hat q_1^{\,t}\in \Delta(k_1^t)$, $ \hat q_2^{\,t}\in \Delta(k_2^t)$ in the primal update~\eqref{eq.primal_s} and $\Delta(k_1^t)$ and $\Delta(k_2^t)$ in the confidence sets~\eqref{eq.confidence}.
	To bound $\norm{\hat q_1^{\,t} - \hat q_1^{\,t-1}}_1 
	+\norm{\hat q_2^{\,t}-\hat q_2^{\,t-1}}_1$, we consider two cases: (i) $k_1^t = k_1^{t-1}$ and $k_2^t = k_2^{t-1}$; (ii) either $k_1^t \neq k_1^{t-1}$ or $k_2^t \neq k_2^{t-1}$.
	
	\noindent\textbf{Case~(i)}.	In this case, we have: $ \hat q_1^{\,t}$, $\hat q_1^{\,t-1}\in \Delta(k_1^t)$, $ \hat q_2^{\,t}$, $\hat q_2^{\,t-1}\in \Delta(k_2^t)$. 
	We begin with the primal update~\eqref{eq.primal} and apply Lemma~\ref{lem.pushback} with, 
	\[
	f(x,y) \vert_{x\,=\,q_1,\, y\,=\,q_2} 
	\;=\;
	V\, \big\langle{q_1\cdot \hat q_2^{\,t-1}+\hat q_1^{\,t-1}\cdot q_2},{r^{t-1}}\big\rangle 
	\,+\,
	\lambda_1^{t-1} \langle{q_1},{g^{t-1}}\rangle 
	\,-\,\lambda_2^{t-1} \langle{q_2},{h^{t-1}}\rangle 
	\]
	and $x^\star = \hat q_1^{\,t}$, $y^\star= \hat q_2^{\,t}$, $x' = \tilde q_1^{\,t-1}$, $y'=\tilde q_2^{\,t-1}$, $x = \tilde q_1^{\,t-1}$, and $y=\tilde q_2^{\,t-1}$. Thus, 
	\[
	\begin{array}{rcl}
	&& \!\!\!\! \!\!\!\! \!\!
	V\, \big\langle{\hat q_1^{\,t}\cdot \hat q_2^{\,t-1}+\hat q_1^{\,t-1}\cdot  \tilde q_2^{\,t-1}},{r^{t-1}}\big\rangle 
	\,+\,
	\lambda_1^{t-1} \langle{\hat q_1^{\,t}},{g^{t-1}}\rangle 
	\,-\,
	\lambda_2^{t-1} \langle{ \tilde q_2^{\,t-1}},{h^{t-1}}\rangle 
	\\[0.2cm]
	&& 
	\,+\, \eta^{-1} \big(D(\hat q_1^{\,t},\tilde q_1^{\,t-1}) + D(\hat q_2^{\,t}, \tilde q_2^{\,t-1})\big)
	\\[0.2cm]
	&\leq& V\, \big\langle{ \tilde q_1^{\,t-1}\cdot \hat q_2^{\,t-1}+\hat q_1^{\,t-1}\cdot \hat q_2^{\,t}},{r^{t-1}}\big\rangle 
	\,+\,
	\lambda_1^{t-1} \langle{ \tilde q_1^{\,t-1}},{g^{t-1}}\rangle 
	\,-\,
	\lambda_2^{t-1} \langle{ \hat q_2^{\,t}},{h^{t-1}}\rangle 
	\\[0.2cm]
	&&
	\,-\, \eta^{-1} \big(D( \tilde q_1^{\,t-1},\hat q_1^{\,t}) + D( \tilde q_2^{\,t-1}, \hat q_2^{\,t}) \big).
	\end{array}
	\]
	or, equivalently,
	\begin{equation}\label{eq.pushback_I_s}
	\begin{array}{rcl}
	&& \!\!\!\! \!\!\!\!  \!\! 
	\eta^{-1} \big(D(\hat q_1^{\,t},\tilde q_1^{\,t-1}) + D(\hat q_2^{\,t}, \tilde q_2^{\,t-1})\big)
	\,+\,
	\eta^{-1} \big(D( \tilde q_1^{\,t-1},\hat q_1^{\,t}) + D( \tilde q_2^{\,t-1}, \hat q_2^{\,t}) \big)
	\\[0.2cm]
	&\leq&  V\, \big\langle{ (\tilde q_1^{\,t-1} -\hat q_1^{\,t} )\cdot \hat q_2^{\,t-1}+\hat q_1^{\,t-1}\cdot (\hat q_2^{\,t} - \tilde q_2^{\,t-1} )},{r^{t-1}}\big\rangle 
	\\[0.2cm]
	&& 
	\,+\,
	\lambda_1^{t-1} \langle{ \tilde q_1^{\,t-1} - \hat q_1^{\,t} },{g^{t-1}}\rangle 
	\,+\,
	\lambda_2^{t-1} \langle{ \tilde q_2^{\,t-1} - \hat q_2^{\,t}},{h^{t-1}}\rangle.
	\end{array}
	\end{equation}
	We note that $\langle{ (\tilde q_1^{\,t-1} -\hat q_1^{\,t} )\cdot \hat q_2^{\,t-1}},{r^{t-1}}\rangle \leq \Vert (\tilde q_1^{\,t-1} -\hat q_1^{\,t} )\cdot \hat q_2^{\,t-1}\Vert_1\Vert{r^{t-1}}\Vert_\infty\leq \Vert\tilde q_1^{\,t-1} -\hat q_1^{\,t} \Vert_1$, and, similarly, $\langle\hat q_1^{\,t-1}\cdot (\hat q_2^{\,t} - \tilde q_2^{\,t-1} ),{r^{t-1}}\rangle\leq\Vert\hat q_2^{\,t} - \tilde q_2^{\,t-1}\Vert_1$. Thus, we can reduce~\eqref{eq.pushback_I_s} into
	\[
	\begin{array}{rcl}
	&& \!\!\!\! \!\!\!\! \!\! 
	\eta^{-1} \big(D(\hat q_1^{\,t},\tilde q_1^{\,t-1}) + D(\hat q_2^{\,t}, \tilde q_2^{\,t-1})\big)
	\,+\,
	\eta^{-1} \big(D( \tilde q_1^{\,t-1},\hat q_1^{\,t}) + D( \tilde q_2^{\,t-1}, \hat q_2^{\,t}) \big)
	\\[0.2cm]
	&\leq&
	(V+\lambda_1^{t-1}) 
	\norm{ \tilde q_1^{\,t-1} - \hat q_1^{\,t} }_1
	\,+\,
	(V+\lambda_2^{t-1}) 
	\norm{ \tilde q_2^{\,t-1} - \hat q_2^{\,t}}_1
	\\[0.2cm]
	&\leq&
	(V+ \norm{\lambda^{t-1}}) 
	\left(
	\norm{ \tilde q_1^{\,t-1} - \hat q_1^{\,t} }_1
	+
	\norm{ \tilde q_2^{\,t-1} - \hat q_2^{\,t}}_1
	\right)
	\end{array}
	\]
	where the left-hand side can be lower bounded by Lemma~\ref{lem.D_lb}, 
	\[
	D(\hat q_1^{\,t},\tilde q_1^{\,t-1})+D( \tilde q_1^{\,t-1},\hat q_1^{\,t}) \;\geq\; L^{-1}\norm{ \tilde q_1^{\,t-1}-\hat q_1^{\,t} }_1^2
	\]
	\[
	D(\hat q_2^{\,t},\tilde q_2^{\,t-1})+D( \tilde q_2^{\,t-1},\hat q_2^{\,t}) \;\geq\; L^{-1}\norm{ \tilde q_2^{\,t-1}-\hat q_2^{\,t} }_1^2.
	\]
	Then, we apply the inequality $(x+y)^2\leq2(x^2+y^2) $ and cancel a non-negative term to obtain
	\begin{equation}\label{eq.qq_dual_s}
	\norm{ \tilde q_1^{\,t-1}-\hat q_1^{\,t} }_1 + \norm{ \tilde q_2^{\,t-1}-\hat q_2^{\,t} }_1
	\;\leq\;
	2\eta  L (V + \norm{\lambda^{t-1}}).
	\end{equation}
	By the definition of $\tilde q_1^{\,t-1}$ and $\tilde q_2^{\,t-1}$,
	\[
	\begin{array}{rcl}
	\norm{ \tilde q_1^{\,t-1}-\hat q_1^{\,t} }_1 
	&=&\displaystyle
	\sum_{\ell\,=\,0}^{L-1}\sum_{x\,\in\,X_\ell} \sum_{a\,\in\,A}
	\abr{
		(1-\theta)\hat {q}_1^{\,t-1}(x,a)+\theta\frac{1}{|X_\ell||A|} - \hat q_1^{\,t}(x,a)
	}
	\\[0.2cm]
	&\geq&\displaystyle
	\sum_{\ell\,=\,0}^{L-1}\sum_{x\,\in\,X_\ell} \sum_{a\,\in\,A}
	\rbr{
		(1-\theta)  \abr{
			\hat {q}_1^{\,t-1}(x,a)- \hat q_1^{\,t}(x,a)
		} -\theta \rbr{\frac{1}{|X_\ell||A|} + \hat q_1^{\,t}(x,a)}
	}
	\\[0.2cm]
	&=& (1-\theta) \norm{\hat {q}_1^{\,t-1}-\hat q_1^{\,t}}_1 -2\theta L.
	\end{array}
	\]
	Similarly, we have $\Vert{ \tilde q_2^{\,t-1}-\hat q_2^{\,t} }\Vert_1 \leq (1-\theta) \Vert{\hat {q}_2^{\,t-1}-\hat q_2^{\,t}}\Vert_1 -2\theta L$. Thus, we can further reduce~\eqref{eq.qq_dual_s} into
	\begin{equation}\label{eq.case(i)_s}
	\norm{\hat {q}_1^{\,t-1}-\hat q_1^{\,t}}_1 +\Vert{\hat {q}_2^{\,t-1}-\hat q_2^{\,t}}\Vert_1 
	\;\leq\;
	2\eta (1-\theta)^{-1} L (V + \norm{\lambda^{t-1}}) + 4 \theta (1-\theta)^{-1}  L.
	\end{equation}
	
	\noindent\textbf{Case~(ii)}. In this case, either $\hat q_1^{\,t}$, $\hat q_1^{\,t-1}$ or $ \hat q_2^{\,t}$, $\hat q_2^{\,t-1}$ might not have the same domain. For instance, when $k_1^t>k_1^{t-1}$, it is possible that $\Delta(k_1^{t})$ becomes different from $\Delta(k_1^{t-1})$. We note that $k_1^t>k_1^{t-1}$ only happens when episode $t$ is the first one that belongs to epoch $k_1^t$. By Lemma~\ref{lem.epoch}, $k_1^T \leq \sqrt{ T|X||A|}\log (8T/(|X||A|))$ and  $k_2^T \leq \sqrt{ T|Y||B|}\log (8T/(|Y||B|))$ if we are given $T\geq \max(|X||A|, |Y||B|)$.
	
	We now combine two cases above for~\eqref{eq.violation_dqq_s},
	\[
	\begin{array}{rcl}
	&& \!\!\!\! \!\!\!\! \!\!\!\!
	\displaystyle
	\sum_{t\,=\,1}^{T} \rbr{
		\norm{\hat q_1^{\,t} - \hat q_1^{\,t-1}}_1 
		+\norm{\hat q_2^{\,t}-\hat q_2^{\,t-1}}_1}
	\\[0.2cm]
	&=& \displaystyle
	\sum_{\substack{1\,\leq\,t\,\leq\,T\\k_1^t \,=\, k_1^{k-1}\,\wedge\, k_2^t \,=\, k_2^{k-1}}}\rbr{
		\norm{\hat q_1^{\,t} - \hat q_1^{\,t-1}}_1 
		+\norm{\hat q_2^{\,t}-\hat q_2^{\,t-1}}_1}
	\\[0.2cm]
	&&\displaystyle
	\,+\,
	\sum_{\substack{1\,\leq\,t\,\leq\,T\\k_1^t \,=\, k_1^{k-1}\,\vee\, k_2^t \,=\, k_2^{k-1}}}\rbr{
		\norm{\hat q_1^{\,t} - \hat q_1^{\,t-1}}_1 
		+\norm{\hat q_2^{\,t}-\hat q_2^{\,t-1}}_1}
	\\[0.2cm]
	&\leq& \displaystyle
	\sum_{\substack{1\,\leq\,t\,\leq\,T\\k_1^t \,=\, k_1^{k-1}\,\wedge\, k_2^t \,=\, k_2^{k-1}}}\rbr{
		\norm{\hat q_1^{\,t} - \hat q_1^{\,t-1}}_1 
		+\norm{\hat q_2^{\,t}-\hat q_2^{\,t-1}}_1}
	\,+\,2L (k_1^T+k_2^T)
	\\[0.2cm]
	&\leq&\displaystyle
	2\eta (1-\theta)^{-1} L\sum_{t\,=\,1}^{T} (V + \norm{\lambda^{t-1}}) + 4 \theta (1-\theta)^{-1}  LT +2L (k_1^T+k_2^T)
	\end{array}
	\]
	where the first inequality is due to: $\norm{\hat q_1^{\,t} - \hat q_1^{\,t-1}}_1 \leq 2L$ and $\norm{\hat q_2^{\,t}-\hat q_2^{\,t-1}}_1\leq 2L$, and we apply~\eqref{eq.case(i)_s} from the case (i) for the last inequality. Using the bounds on $k_1^T$, $k_2^T$ in the case (ii), we conclude the desired bound for~\eqref{eq.violation_dqq_s},
	\[
	\begin{array}{rcl}
	&& \!\!\!\! \!\!\!\! \!\!
	\displaystyle {\hat{\text{\normalfont Violation}_1}(T)}, {\hat{\text{\normalfont Violation}_2}(T)}
	\\[0.2cm]
	&\leq&\displaystyle
	\norm{\lambda^T}
	\,+\, 
	\frac{2\eta L}{1-\theta}
	\sum_{t\,=\,1}^{T} \norm{\lambda^{t-1}} + \frac{2\eta V+4 \theta}{1-\theta}  LT 
	\\[0.2cm]
	&&\displaystyle+2L 
	\rbr{\sqrt{ T|X||A|}\log (8T/(|X||A|))+ \sqrt{ T|Y||B|}\log (8T/(|Y||B|))}.
	\end{array}
	\]
	We complete the proof by noting $\lambda_1^0=\lambda_2^0=0$, $V=L\sqrt{T}$, $\eta = 1/(TL)$, and $\theta = 1/T$.
\end{proof}

To get the violation bound, we apply Lemma~\ref{lem.lambda_s} to Theorem~\ref{thm.violation_hat_s}, use Lemma~\ref{lem.error34}, and take $\delta=p/(2T)$.


\section{Supporting Lemmas}

We collect some useful lemmas in literature for the convenience of reading our paper. 

The following drift analysis of stochastic processes is useful in the constraint violation analysis.
\begin{lemma}\citep{yu2017online}\label{lem.drift} 
	Let $\{ Z^t, t\geq0 \}$ be a discrete-time stochastic process that is adapted to a filtration $\{\calF^t, t\geq 0 \}$ with $Z^0 =0$ and $\calF^0 = \{\emptyset,\Omega\}$. Assume that there exists $t_0 \in \mathbb{Z}^+$, $\Theta\in \mathbb{R}^+$, $ \delta_{\max}\in\mathbb{R}^+$, and $\zeta\in (0, \delta_{\max}]$ such that for all $t\geq 1$,
	\[
	\abr{ Z^{t+1}-Z^t } \;\leq\; \delta_{\max}
	\;\text{ and }\;
	\mathbb{E} \sbr{ Z^{t+t_0}-Z^t  \,\vert\, \calF^t}
	\;\leq\;
	\begin{cases}
	t_0\, \delta_{\max} \;\; \text{ when } Z^t \leq \Theta
	\\[0.2cm]
	-\,t_0\,\zeta \;\;\;\; \text{ otherwise } Z^t \geq \Theta.
	\end{cases}
	\]
	Then, with probability $1-\delta$ it holds for any $t$ that 
	\[
	Z^t \;\leq \; \Theta + t_0 \delta_{\max} + t_0 \frac{4\delta_{\max}^2}{\zeta} \log \rbr{\frac{8\delta_{\max}^2}{\zeta}} + t_0 \frac{4\delta_{\max}^2}{\zeta} \log\frac{1}{\delta}.
	\]
\end{lemma}

A general Azuma-Hoeffding inequality for supermartingales with unbounded differences is given as follows.

\begin{lemma}\citep{yu2017online}\label{lem.azuma_general} 
	Let $\{ Z^t, t\geq0 \}$ be a supermartingale that is adapted to a filtration $\{\calF^t, t\geq 0 \}$ with $Z^0 =0$ and $\calF^0 = \{\emptyset,\Omega\}$. 
	Let $\{ Y^t, t\geq0 \}$ be a discrete-time stochastic process that is adapted to a filtration $\{\calF^t, t\geq 0 \}$.
	Assume that there exists a constant $c\in\mathbb{R}^+$ such that $\{|Z^{t+1}-Z^t| > c\} \subset \{ Y^t>0 \}$ for any $t\geq 0$.
	Then, for any $z \in\mathbb{R}^+$ and $t\geq 1$,
	\[
	P( Z^t \geq z ) \;\leq\; {\rm e}^{-z^2/ (2c^2t)} \,+\, \sum_{\tau\,=\,0}^{t-1} P(Y^t>0).
	\]
\end{lemma}

The following two lemmas are useful in the epoch analysis. 
\begin{lemma}\citep{jaksch2010near}\label{lem.series} 
	Let a sequence of positive numbers be $x_1,\ldots,x_n$. Assume that $0\leq x_k\leq X_{k-1}\DefinedAs \max(1, \sum_{i\,=\,1}^{k-1}x_i)$ for $1\leq k\leq n$. Then,
	\[
	\sum_{k\,=\,1}^{n}\frac{x_k}{\sqrt{X_{k-1}}}\;\leq\;(\sqrt{2}+1)\sqrt{X_n}.
	\]
\end{lemma}

\begin{lemma}\citep{jaksch2010near}\label{lem.epoch} 
	Assume that $T\geq \max(|X||A|, |Y||B|)$. Then, the epochs $k_1^T$ and $k_2^T$ for episode $T$ 
	\[
	k_1^T\;\leq\; |X||A|\log \rbr{\frac{8T}{|X||A|}}\;\leq\;\sqrt{T  |X||A|}\log \rbr{\frac{8T}{|X||A|}}
	\]
	\[
	k_2^T\;\leq\; |Y||B|\log \rbr{\frac{8T}{|Y||B|}}\;\leq\;\sqrt{T  |Y||B|}\log \rbr{\frac{8T}{|Y||B|}}.
	\]
\end{lemma}

\end{document}